\documentclass{article} %

\usepackage{arxiv}

\usepackage{amsmath,amsfonts,bm}

\def\eqref#1{equation~\ref{#1}}

\def\1{\bm{1}}

\def\eps{{\epsilon}}

\def\rx{{\textnormal{x}}}
\def\ry{{\textnormal{y}}}

\def\vtheta{{\bm{\theta}}}

\def\ve{{\bm{e}}}

\def\vp{{\bm{p}}}

\def\vw{{\bm{w}}}
\def\vx{{\bm{x}}}
\def\vy{{\bm{y}}}
\def\vz{{\bm{z}}}

\def\mI{{\bm{I}}}

\def\mW{{\bm{W}}}

\DeclareMathAlphabet{\mathsfit}{\encodingdefault}{\sfdefault}{m}{sl}
\SetMathAlphabet{\mathsfit}{bold}{\encodingdefault}{\sfdefault}{bx}{n}

\DeclareMathOperator*{\argmax}{arg\,max}
\DeclareMathOperator*{\argmin}{arg\,min}

\usepackage[utf8]{inputenc} %
\usepackage[T1]{fontenc}    %
\usepackage{hyperref}       %
\usepackage{url}            %
\usepackage{booktabs}       %
\usepackage{amsfonts}       %
\usepackage{nicefrac}       %
\usepackage{microtype}      %
\usepackage{lipsum}		%
\usepackage{graphicx}
\usepackage{natbib}
\usepackage{doi}

\usepackage{amsmath, amsthm, amssymb}

\usepackage{xcolor}         %

\usepackage{graphicx}

\usepackage{capt-of}%
\usepackage{varwidth}

\newcommand\op[1]{\operatorname{#1}}

\usepackage{colortbl}
\usepackage{wrapfig}
\usepackage{mathtools}
\usepackage{algorithm, algorithmic}
\usepackage{multirow}

\usepackage{subfig}
\usepackage{xfrac}
\usepackage[font=small]{caption}

\usepackage{ulem,xpatch}
\xpatchcmd{\sout}
{\bgroup}
{\bgroup}
{}{}

\newcommand\grayl[1]{\multicolumn{1}{>{\columncolor{mygray}}l}{#1}}
\newcommand\grayc[1]{\multicolumn{1}{>{\columncolor{mygray}}c}{#1}}

\renewcommand{\epsilon}{\varepsilon}
\newcommand{\gl}{\textsl{GenLabel}}

\newcommand{\dd}{\mathrm{d}}
\newcommand{\ie}{{\it i.e.}, }
\newcommand{\eg}{{\it e.g.}, }
\newcommand{\sm}{Appendix} %
\newcommand{\openml}{\citep{OpenML2013} }

\newcommand\bmt[1]{\textcolor{black}{#1}}%

\newtheorem{theorem}{Theorem}
\newtheorem{proposition}{Proposition}
\newtheorem{example}{Example}

\newtheorem{lemma}{Lemma}
\newtheorem{remark}{Remark}

\definecolor{mygray}{gray}{0.85}

\makeatletter
\newcommand\footnoteref[1]{\protected@xdef\@thefnmark{\ref{#1}}\@footnotemark}
\makeatother

\title{\gl{}: Mixup Relabeling using Generative Models}
\date{}

\author{ 
Jy-yong Sohn\footnotemark[2]
,\ \ Liang Shang\footnotemark[2] ,\ \  Hongxu Chen\footnotemark[2] ,\ \ Jaekyun Moon\footnotemark[4] , \ \ 
Dimitris Papailiopoulos\footnotemark[2] ,\ \ Kangwook Lee\footnotemark[2] \\ \\
\normalsize \footnotemark[2] \ \ University of Wisconsin-Madison\\
\footnotemark[4] \ \ Korea Advanced Institute of Science and Technology
}

\begin{document}

\maketitle

\begin{abstract}
Mixup is a data augmentation method that generates new data points by mixing a pair of input data. While mixup generally improves the prediction performance, it sometimes degrades the performance. In this paper, we first identify the main causes of this phenomenon by theoretically and empirically analyzing the mixup algorithm.
To resolve this, we propose \gl{}, a simple yet effective relabeling algorithm designed for mixup.
In particular, \gl{} helps the mixup algorithm correctly label mixup samples by learning the class-conditional data distribution using generative models.
Via extensive theoretical and empirical analysis, we show that mixup, when used together with \gl{}, can effectively resolve the aforementioned phenomenon, improving the generalization performance and the adversarial robustness.
\end{abstract}

\section{Introduction}
Mixup~\citep{zhang2017mixup} is a widely adopted data augmentation algorithm used when training a classifier, which generates synthetic samples by linearly interpolating two randomly chosen samples. 
Each mixed sample is \emph{soft}-labeled, i.e., it is labeled as a mixture of two (possibly same) classes of the chosen samples. 
The rationale behind the mixup algorithm is that such mixed samples can fill up the void space in between different class manifolds, effectively regularizing the model behavior. 
Mixup has been shown to improve generalization on multiple benchmark image datasets,
and several variants of mixup have been proposed in the past few years.
For example, manifold-mixup~\citep{pmlr-v97-verma19a} generalizes the mixup algorithm by applying the same algorithm in the latent feature space, and some other variants tailor the original mixup algorithm to perform better on computer vision tasks~\citep{yun2019cutmix,kim2020puzzle,uddin2020saliencymix,kim2021co}.
Though the initial studies lack theoretical supports for the mixup algorithm, recent studies provide some theoretical explanations on why and how mixup can improve generalization~\citep{zhang2021does,carratino2020mixup}.

Mixup, however, does not always improve generalization, and sometimes it even hurts. 
For instance,~\citet{guo2019mixup} showed that the generalization performance of mixup is up to 1.8\% worse than vanilla training on some image classification tasks. 
Similarly, ~\citet{greenewald2021k} showed that for the classification tasks on some UCI datasets~\citep{Dua2019}, the original mixup (which is $k$-mixup with $k=1$ in their paper) degrades the generalization performance of vanilla training up to 2.5\%.
Unfortunately, these empirical observations on the failure of mixup has not been supported by a clear theoretical understanding.

\subsection{Main contributions}

In this work, we present a rigorous understanding of when and why the current mixup algorithm fails.
To obtain this understanding, we take a closer look at the failure scenarios of mixup, particularly focusing on the low-dimensional input setting.
We identify two main reasons behind mixup's failure scenarios.
The first reason we identify is \emph{manifold intrusion}, which was firstly defined in
\citep{guo2019mixup}.
Mixup samples generated by mixing two classes may intrude the manifold of a third class, so such intruding mixup samples will cause label conflicts with the true samples from the intruded class.
We perform theoretical and empirical analysis of the effect of manifold intrusion on mixup's performance.
The second reason we identify is about how the current mixup algorithm labels the mixed samples. 
The current algorithm assigns a mixup sample with a two-hot encoded label, which is a linear combination of the two one-hot encoded labels of the original samples. 
We prove that, focusing on a specific softmax regression setting, such linear interpolation of one-hot encoded labels results in a strictly suboptimal margin.%

After we identify the key reasons behind mixup's failure cases, we propose a simple yet effective fix for the current mixup algorithm.
Our proposed algorithm \gl{} is a relabeling algorithm designed for mixup. 
The idea is strikingly simple -- \gl{}  relabels a mixed sample using the likelihoods that are estimated with learned generative models.

See Fig.~\ref{Fig:GenLabel} for visual illustration.
Consider a three-way classification problem.
\gl{} first learns generative models, from which it can estimate the likelihood of a sample drawn from each class.
Let the likelihood of sample $\vx$ drawn from class $c$ be $p_c(\vx)$, and the estimated likelihood be $\widehat{p_c}(\vx)$.
Given a new mixed sample $\vx^\text{mix}$, \gl{} first estimates all three likelihoods $\widehat{p_c}(\vx^\text{mix})$ for class $c \in \{1,2,3\}$, and then assigns the mixed sample the following label:
$\vy^\text{gen} = \text{softmax}( \log \widehat{p_1}(\vx^\text{mix}), \log \widehat{p_2}(\vx^\text{mix}), \log \widehat{p_3}(\vx^\text{mix}))$.
Note that the $c^{\op{th}}$ element of $\vy^{\op{gen}}$ is $\frac{\widehat{p_c}(\vx^\text{mix})}{\sum_{c'=1}^3 \widehat{p_{c'}}(\vx^\text{mix})}$, which is identical to the posterior probability $\mathbb{P}(y=c | \vx^\text{mix}) = \frac{p_c(\vx^\text{mix}) \mathbb{P}(y=c)}{\sum_{c'=1}^3 p_{c'}(\vx^\text{mix}) \mathbb{P}(y=c') }$ of $\vx^\text{mix}$ belonging to class $c$,
when we have a balanced dataset and a perfect likelihood estimation, i.e., $\mathbb{P}(y=c) = \mathbb{P}(y=c')$ for all class pair $c, c'$, and $\widehat{p_c}(\vx^\text{mix}) = {p_c}(\vx^\text{mix})$ for all class $c$. Thus, \gl{} is a labeling method that assigns the posterior probability of the label $y$ given a mixed sample $\vx^\text{mix}$.
This property of \gl{} allows us to fix the issue of  the conventional labeling method in mixup.
For the example given in Fig.~\ref{Fig:GenLabel}, the mixed sample $\vx^{\op{mix}}$ lies on the manifold of class 2, when we mix $\vx$ in class 1 and $\vx^{\prime}$ in class 3. While the original mixup labels $\vx^{\op{mix}}$ as a mixture of classes 1 and 3, the label assigned by \gl{} is nearly identical to the ground-truth label (class 2), since we have $\widehat{p_2}(\vx^{\op{mix}}) \gg \widehat{p_1}(\vx^{\op{mix}})$ and $ \widehat{p_2}(\vx^{\op{mix}}) \gg \widehat{p_3}(\vx^{\op{mix}})
$.

The suggested \gl{} has been analyzed in diverse perspectives, showing that \gl{} helps fixing the issue of mixup and improving the performances. 
First, we empirically show that \gl{} fixes the manifold intrusion issue on toy datasets. Second, we mathematically prove that on several toy datasets, \gl{} combined with mixup maximizes the margin of a classifier, while mixup alone leads to a much smaller margin, even worse than that of the vanilla training. Third, our mathematical results show that \gl{} improves the adversarial robustness of mixup in logistic regression models and fully-connected (FC) networks with ReLU activations.

Finally, we tested \gl{} 
on 109 low-dimensional real datasets in OpenML~\citep{OpenML2013}.
Our experimental results show that the suggested \gl{} helps mixup improve not only the generalization performance, but also the adversarial robustness of a classifier in various low-dimensional datasets, in both logistic regression models and FC ReLU networks. This corroborates the advantages of \gl{} we showed in our theoretical analysis.

\begin{figure}[t]
    \vspace{-2mm}
	\centering
	\includegraphics[width=0.8 \linewidth]{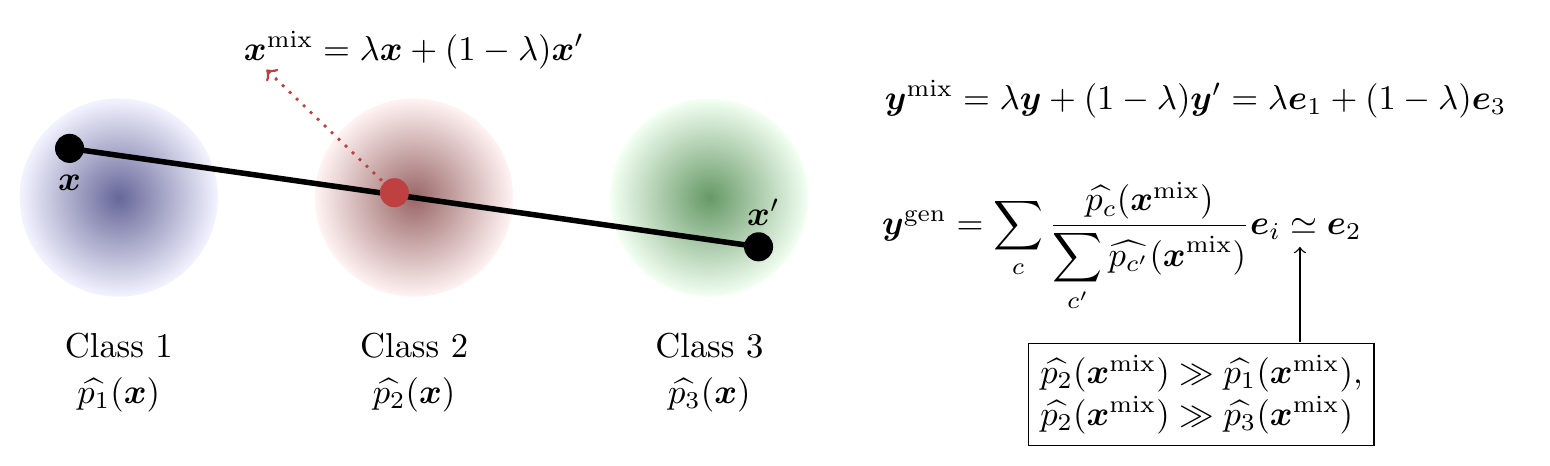}
	\caption
	{
    Conceptual visualization of \gl{} applied to mixup. Consider a mixed sample $\vx^{\op{mix}} = \lambda \vx + (1-\lambda) \vx'$ for a given $\lambda \in [0,1]$ and randomly chosen data samples $\vx, \vx'$.
    Conventionally, this mixed sample is labeled as $\vy^{\op{mix}} = \lambda \vy + (1-\lambda) \vy'$, which is incorrect if 
    the mixed sample lies on the manifold of another class $\vy^{\star} \notin \{\vy, \vy^{\prime}\}$.
    For example, mixing classes 1 and 3 generates a mixed sample lying on the manifold of class 2.
    To fix this issue, this paper suggests \gl{} which relabels the mixed sample. \gl{} first learns the underlying data distribution for each class $c$, denoted by $\widehat{p_c}(\vx)$. Then, the mixed sample is relabeled as $\vy^{\op{gen}}$ based on the likelihood $\widehat{p_c}(\vx^{\op{mix}})$ of the mixed sample drawn from each class $c$, where the detailed expression for $\vy^{\op{gen}}$ is given above. In the above example, we have $\vy^{\op{mix}} = \lambda \ve_1 + (1-\lambda) \ve_3$ and $\vy^{\op{gen}} \simeq \ve_2$, where $\ve_c$ is the standard basis vector with a 1 in the $c$-th coordinate and 0's elsewhere.
    \gl{} coincides with the ground-truth label $\ve_2$, while the original labeling method does not.
    }
	\label{Fig:GenLabel}
\end{figure}

\subsection{Preliminaries}

Below we summarize the basic notations and concepts used in our paper.

\paragraph{Notations}
In this work, we focus on $k$-way classification tasks.
A dataset with $n$ data points is denoted by $S = \{\vz_i \}_{i=1}^n$, where the $i$-th data point is represented by a tuple $\vz_i = (\vx_i, \vy_i)$ composed of the input feature $\vx_i \in \mathbb{R}^d$ and the label $\vy_i \in [0,1]^k$.
We use one-hot encoding for the label, \ie the label of class-$c$ data points is represented as $\vy = \ve_c$, where $\ve_c$ is the standard basis vector with a 1 in the $c$-th coordinate and 0's elsewhere.
For a mixed sample, we allow \emph{soft label}, \eg $0.5\ve_1 + 0.5\ve_2$ denotes that the mixed sample is equally likely to be from class $1$ and $2$. For a given dataset $S = \{\vz_i\}_{i=1}^n$, the set of input features is denoted by $X = \{\vx_i\}_{i=1}^n %
$, and the empirical distribution of the input feature is denoted as $D_X$.
We assume that each data point $\vx$ in class $c$ is generated from (unknown) probability distribution $p_{c}(\vx) \coloneqq p_{\rx|\ry}(\vx|\ve_c)$. 

For a positive integer $k$, we use the notation $[k] \coloneqq \{1, 2, \cdots, k \}$. The identity matrix of size $k \times k$ is denoted by $\mI_k$.
For a given statement $A$, we define $\mathbf{1}_{A}=1$ if $A$ is true, and $\mathbf{1}_{A}=0$ otherwise.
When a distance metric is specified, $d(x_{1},x_{2})$ denotes the distance between $x_{1}$ and $x_{2}$, and $d(x,A) = \min_{a \in A} d(x,a)$ denotes the minimum of the distances between $x$ and the points in a closed set $A$.

\vspace{-2mm}
\paragraph{Mixup} 
Mixup~\citep{zhang2017mixup} generates synthetic data points by applying a linear combination of two samples. 
Given samples $\vx_i$ and $\vx_j$, it generates an augmented point $\vx^{\op{mix}} = \lambda \vx_i+(1-\lambda) \vx_j$ having a mixed label $\vy^{\op{mix}} = \lambda \vy_i+(1-\lambda) \vy_j$,
for randomly sampled 
$\lambda \sim \op{Beta}(\alpha, \alpha)$ for a given $\alpha > 0$. 

\vspace{-2mm}
\paragraph{Gaussian mixture model}
Gaussian mixture (GM) model is a generative model, which assumes that samples $\vx$ in each class (say class $c$) follow a multivariate Gaussian distribution $\mathcal{N} (\bm{\mu}_c, \bm{\Sigma}_c)$ for some mean $\bm{\mu}_c$ and covariance matrix $\bm{\Sigma}_c$. 
One can estimate the model parameters by computing the within-class sample mean $\widehat{\bm{\mu}_c}$ and the within-class sample covariance matrix $\widehat{\bm{\Sigma}_c}$ for each class. 
A GM model of two classes, say class $1$ and $2$, can be modeled as $\pi_1 \mathcal{N} (\widehat{\bm{\mu}_1}, \widehat{\bm{\Sigma}_1}) + (1-\pi_1) \mathcal{N} (\widehat{\bm{\mu}_2}, \widehat{\bm{\Sigma}_2})$, where $\pi_1 = \mathbb{P}(\vy=\ve_1)$.

\textbf{Kernel density estimator} Kernel density estimator (KDE) is a non-parametric density estimator that makes use of a kernel function.
For example, KDE with Gaussian kernel estimates the distribution of class $c$ as $\frac{1}{n_c} \sum_{i=1}^{n_c} \mathcal{N}(\vx_i, h^2 \widehat{\bm{\Sigma}_c})$ for a given bandwidth $h$, where
$\{\vx_i\}_{i=1}^{n_c}$ is the set of samples in class $c$ and $\widehat{\bm{\Sigma}_c}$ is the sample covariance matrix of class $c$.
One can use KDE as a generative model, creating new samples from the estimated density.

\section{Related works}

\paragraph{Mixup and variants}

Mixup and its variants have been considered as promising data augmentation schemes improving the generalization and robustness performance in various image classification tasks~\citep{zhang2017mixup,pmlr-v97-verma19a,tokozume2017learning,inoue2018data,shimada2019data,hendrycks2019augmix,yun2019cutmix,kim2020puzzle,uddin2020saliencymix,kim2021co,zhang2021does}. 
However, the performance of mixup for low-dimesional datasets have been rarely observed in previous works. This paper focuses on the failures of mixup in low-dimensional datasets, and provide a simple label correction method to solve this issue, which improves both generalization performance and adversarial robustness in various real datasets.

\paragraph{Manifold intrusion}
\citet{guo2019mixup} observed that mixup samples of two classes may intrude the manifold of a third class.
The authors dubbed this phenomenon as \emph{manifold intrusion}.
The manifold intrusion problem can explain why mixup sometimes hurts generalization -- such intruding mixup samples will cause label conflicts with the true samples from the intruded class.
\citet{hwang2021mixrl} found that a similar label conflict problem becomes even more salient in the regression setting. 
To resolve the label conflict issue of the manifold-intruding mixup points, previous works have suggested various mixing strategies which avoid generating mixup samples that causes the label conflict.~\citet{guo2019mixup} suggested regularizing the mixup samples lie in the out-of-manifold region, by learning the mixing policy that prohibits generating the in-manifold mixup samples.~\citet{greenewald2021k} suggested  %
using the concept of optimal transport to mix data samples that are adjacent to each other. This scheme helps both the mixed sample and the corresponding data sample pair lie on the same manifold, which avoids facing the label-conflicting scenarios.
Focusing on the regression setting,~\citet{hwang2021mixrl} suggested learning a mixing policy by measuring how helpful mixing each pair is.
Although all these regularization techniques prohibit generating mixup points that incur label conflicts, they also inherently give up the potential benefits of using such label-conflicting mixup samples by properly re-labeling them.
In this paper, for the first time, we solve the label conflict issue of manifold-intruding mixup samples by \emph{re-labeling} those mixup samples based on the class-conditional distribution estimated by generative models.

\vspace{-2mm}
\paragraph{Generative models for improving generalization and robustness} 
Generative models have been widely used for classification tasks for several decades. A generative classifier~\citep{ng2002discriminative} predicts label $\vy$ based on the class-conditional density $p(\vx|\vy)$ estimated by generative models, and there are various recent works developing generative classifiers~\citep{schott2018towards,ju2020abs}.
The present paper also makes use of generative models for classification task, but we use them for re-labeling augmented data, while existing works use them for the prediction itself.
Some previous works proposed generative model-based data augmentation schemes~\citep{antoniou2017data,perez2017effectiveness,tanaka2019data}.
While these schemes use the learned distribution to create on-manifold synthetic data, we use the learned distribution to re-label both on-manifold and out-of-manifold mixup samples.
There have been extensive works on using generative models to improve the robustness against adversarial attacks and out-of-distribution samples ~\citep{ilyas2017robust,xiao2018generating,samangouei2018defense,song2017pixeldefend,schott2018towards,li2018generative,ghosh2019resisting,serra2019input,choi2018waic,lee2018simple}.
Though looking similar, these algorithms are not data augmentation algorithms and hence their study is only tangentially related to this work. 
Our method can be used together with any of these algorithms, possibly further improving the model robustness.

\vspace{-2mm}
\paragraph{Adversarial robustness on low-dimensional datasets}
Although the area of adversarial machine learning has been started and developed in the image classification task~\citep{szegedy2013intriguing,kurakin2016adversarial,yuan2019adversarial,biggio2018wild,chakraborty2018adversarial,madry2017towards,carlini2017towards,carlini2019evaluating}, there have been discussions on the adversarial attacks on low-dimensional datasets, e.g., tabular datasets~\citep{ballet2019imperceptible,cartella2021adversarial,gupta2021quantitative}.
\bmt{
This paper focuses on low-dimensional datasets, and discusses methods for improving both the generalization performance and the adversarial robustness of mixup.
We first consider \emph{margin} as a proxy for the generalization/robustness performances, and examine the issues of mixup reducing the margin of a classifier. Then, we provide a re-labeling scheme that fixes these issues, and improves not only the margin, but also the generalization/robustness performances of mixup. %
}

\begin{table}[t]
\vspace{-2mm}
\centering
	\caption{
	Clean accuracy (\%) on synthetic datasets (circle, moon in scikit-learn~\citep{scikit-learn} and two-circle, 2D/3D cube datasets designed by us) and real datasets in OpenML~\citep{OpenML2013}.
	OpenML-$x$ represents the dataset with ID number $x$ in OpenML.
	Note that mixup has a worse performance than vanilla training on these datasets.
	}
    \tiny	
    \setlength{\tabcolsep}{3pt} %
       \begin{tabular}{c|c|c|c|c|c|c|c|c|c|c}%
    \midrule  
    \textbf{Dataset} & 
    Circle & Moon  
    & Two-circle & 2D cube & 3D cube
    & OpenML-48 & OpenML-61 & OpenML-307 & OpenML-818 & OpenML-927
    \\
    \midrule
    \textbf{Vanilla training}       
        & 99.70$\pm$0.16
        & 98.74$\pm$0.34
        & 91.25$\pm$8.14
        & 98.08$\pm$1.01
        & 93.01$\pm$1.65
        & 39.13$\pm$0.00
        & 95.56$\pm$0.00
        & 66.80$\pm$0.16
        & 100.00$\pm$0.00
        & 76.92$\pm$0.00
         \\
        \textbf{Mixup training}         
        & 85.78$\pm$14.44
        & 96.96$\pm$0.67
        & 57.70$\pm$2.94
        & 96.60$\pm$1.51
        & 89.41$\pm$1.83  
        & 30.87$\pm$4.84
        & 88.00$\pm$1.09
        & 54.41$\pm$1.29
        & 92.26$\pm$0.43
        & 69.23$\pm$4.87
        \\
      
        \textbf{Difference (vanilla $-$ mixup)}         
        & 13.92$\pm$14.36
        & 1.78$\pm$0.61
        & 33.55$\pm$9.77
        & 1.47$\pm$1.34
        & 3.60$\pm$2.11
        & 8.26$\pm$4.84
        & 7.56$\pm$1.09
        & 12.39$\pm$1.45
        & 7.74$\pm$0.43
        & 7.69$\pm$4.87
        \\
        \bottomrule
    \end{tabular}     
	\label{Table:mixup_fail}
\end{table}

\section{Failure of mixup on low-dimensional data}\label{sec:failure}
\vspace{-2mm}

In this section, we observe the failure scenarios of mixup, \textit{i.e.}, when mixup performs even worse than vanilla training, especially focusing on the low-dimensional data setting.
Table~\ref{Table:mixup_fail} shows the scenarios when mixup has a lower accuracy than vanilla training, for synthetic datasets\footnote{\label{synthetic_details}The details of synthetic datasets designed by us are provided in Section~\ref{sec:exp_setup} in \sm.} and OpenML datasets~\citep{OpenML2013}.
For example, in the Two-circle dataset, the performance gap between mixup and vanilla training is larger than 30\%.
A natural question is, why mixup has such failure scenarios? 
Here we identify and analyze two main reasons for the failure of mixup.
First, as pointed out by~\citep{guo2019mixup}, mixup has the manifold intrusion issue, i.e., a mixup sample generated by mixing two classes may overlap with a data sample drawn from the third class~\citep{guo2019mixup}. 
We provide theoretical/empirical analysis of the effect of manifold intrusion on the performance of mixup.
Second, we theoretically/empirically show that even when there is no manifold intrusion issue, the labeling method used in mixup may harm the margin/accuracy of a classifier.

\subsection{Manifold intrusion of mixup reduces the margin and accuracy}\label{sec:manifold_intrusion_issue}
\vspace{-1mm}

In~\citep{guo2019mixup}, the manifold intrusion (MI) is defined as the scenario when the mixup sample $\vx^{\op{mix}}$, generated by mixing data in class $c_1$ and $c_2$, collides with a real data sample having the ground-truth label $c_3 \notin \{c_1, c_2\}$. 
Below we theoretically show that the manifold intrusion can reduce the margin of a classifier trained by mixup.

\begin{example}\label{ex:intrusion}
Consider binary classification on the dataset $S = \{ (x_i, \vy_i)\}_{i=1}^3 =  \{(-1, \ve_1), (0, \ve_2), (+1, \ve_1) \}$ in Fig.~\ref{Fig:toy_intrusion_margin_dataset}, where each data point in class 1 is represented as brown circle, and each data point in class 2 is shown as blue triangle.
We consider the classifier $f_{\theta}$ parameterized by $\theta > 0$, shown in Fig.~\ref{Fig:toy_intrusion_margin_classifier}. This classifier estimates the label of a given feature $x$ as $\hat{\vy} = [f_{\theta}(x), 1-f_{\theta}(x)]$, and the margin of this classifier is represented as
$\op{margin} (f_\theta) = \min \{\theta, 1-\theta \}$.

As in Fig.~\ref{Fig:toy_intrusion_example}, applying mixup on this dataset suffers from manifold intrusion (MI); mixing $x_1 = -1$ and $x_3 = +1$ with coefficient $\lambda=0.5$ generates $x^{\op{mix}} = \lambda x_1 + (1-\lambda) x_3 = 0$ with label  $\vy^{\op{mix}} = \lambda \vy_1 + (1-\lambda) \vy_3 = \ve_1$, while we have another data at the same location $x_2 = 0$ having different label $\vy_2 = \ve_2$. Here we observe how this \emph{label conflict} affects the margin of the classifier. To be specific, we compare two schemes: (1) mixup and (2) mixup-without-MI. To avoid MI, we set the scheme (2) to mix only samples with different classes. Here, the mixing coefficient is uniform-randomly sampled as
$\lambda \sim \op{unif}[0,1]$, which is a special case of having $\alpha=1$ in $\lambda \sim \op{Beta}(\alpha, \alpha)$.
For a given scheme $s \in \{ \op{mixup}, \op{mixup-without-MI} \}$, let $\theta_{s}$ be the parameter $\theta$ that minimizes the MSE loss $\ell(\hat{\vy}, \vy) = \lVert \hat{\vy} - \vy \rVert_2^2$.
It turns out that $\op{margin}(\theta_{\op{mixup}}) = \frac{7}{16}$ and 
$\op{margin}(\theta_{\op{mixup-without-MI}}) = \frac{1}{2}$ as shown in Section~\ref{sec:proof_prop:toy_intrusion_margin} of \sm.
\end{example}

\begin{figure}[t]
    \vspace{-5mm}
	\centering
	\subfloat[][\centering{Dataset $S = \{(x_i, \vy_i)\}_{i=1}^3$}]{\includegraphics[width=50mm ]{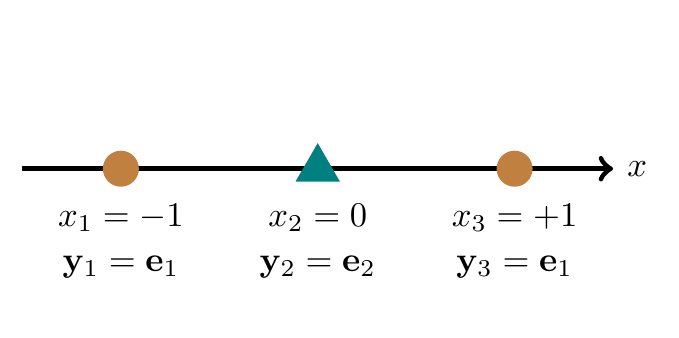}
	\label{Fig:toy_intrusion_margin_dataset}} 
 	\subfloat[][\centering{Classifier $f_{\theta}(x)$ for $\theta > 0$} ]{\includegraphics[height=30mm ]{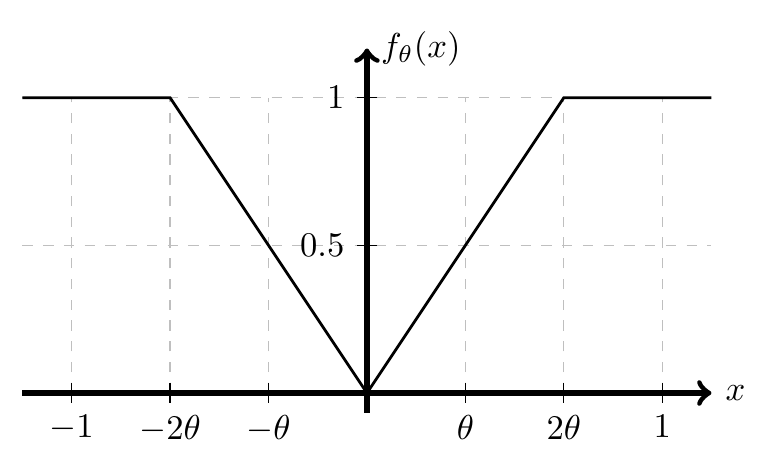}\label{Fig:toy_intrusion_margin_classifier}}
	\subfloat[][\centering{Manifold intrusion (label conflict)}]{\includegraphics[width=50mm ]{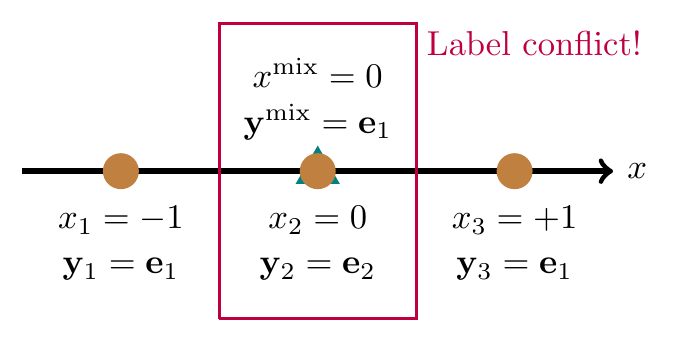}\label{Fig:toy_intrusion_example}}  	
	\caption
	{Binary classification problem defined in Example~\ref{ex:intrusion}: (a) the dataset $S$, (b) the classifier $f_{\theta}$ used for getting the estimated label $\hat{\vy} = [f_{\theta}(x), 1 - f_{\theta}(x)]$.
	As in (c), mixing $x_1=-1$ and $x_3=+1$ with coefficient $\lambda=0.5$ generates a mixed point $x^{\op{mix}} = \lambda x_1 + (1-\lambda) x_3 = 0$ having label $\vy^{\op{mix}}= \lambda \vy_1 + (1-\lambda) \vy_3 = \ve_1$. This mixed point incurs the manifold intrusion, since we have another data at the same location $x_2 = 0$ having different label $\vy_2 = \ve_2$.
    As described in Example~\ref{ex:intrusion}, the manifold intrusion reduces the margin of mixup, assuming we use the classifier $f_{\theta}(x)$ defined in (b).
    }
	\label{Fig:toy_intrusion_margin}
\end{figure}

The example above shows that we can achieve the maximum margin $\frac{1}{2}$ if we remove mixup points having manifold intrusion, while the   naive way of mixing every pair of points degrades the margin to $\frac{7}{16}$.

\begin{wrapfigure}{r}{0.35\textwidth}
\vspace{-6mm}
  \begin{center}
    \includegraphics[width=0.33\textwidth]{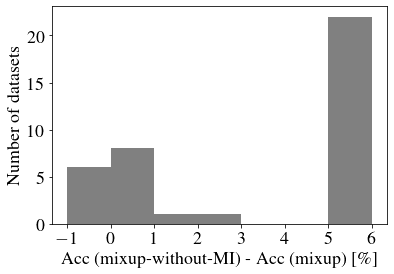}
  \end{center}
  \vspace{-3mm}
  \caption{
  The effect of manifold intrusion (MI) on the classification accuracy of mixup, for 38 datasets in OpenML. 
  Here, we compare mixup with a variant called ``mixup-without-MI'', which follows the basic mixup procedure, but excludes the augmented samples incurring manifold intrusion.
  This shows that manifold intrusion is causing accuracy drop in various real datasets.
 }
  \vspace{-7mm}
  \label{Fig:intrusion_openML}
\end{wrapfigure}

Now we empirically show how the manifold-intruding mixup points affect the classification accuracy of mixup, in various real datasets in OpenML. 
Similar to the previous example, we consider two schemes: (1) mixup and (2) mixup-without-MI, where the second scheme is defined as a usual mixup with the exclusion of mixup points that incur the manifold intrusion. 
Here we decide whether a mixup point is suffering from manifold intrusion, using a relaxed version of the definition suggested in~\citep{guo2019mixup}: we call a mixed point $\vx^{\op{mix}} = \lambda \vx_1 + (1-\lambda) \vx_2$ has manifold intrusion if the label of the nearest neighbor $\vx^{\op{nn}} = \argmin_{\vx \in X} d(\vx, \vx^{\op{mix}})$ is different from $\vy_1$ and $\vy_2$.

Fig.\ref{Fig:intrusion_openML} compares the classification accuracy of (1) mixup and (2) mixup-without-MI, for 38 datasets in OpenML having more than two classes and having not more than 20 features.  
It turns out that for 24 out of 38 datasets, mixup-without-MI has accuracy gain larger than 1\% compared with mixup, and for the remaining 14 datasets, mixup and mixup-without-MI have similar accuracies; the difference of the accuracies is bounded above by 1\%. In other words, excluding the manifold-intruding mixup points is beneficial for improving the classification accuracy, for various real datasets.

\vspace{-1mm}
\subsection{Labeling method in mixup is sub-optimal in terms of the margin and accuracy
}\label{sec:softmax_regression}
\vspace{-1mm}

In this section, we show that for datasets which do not suffer from the manifold intrusion,
the labeling method used in mixup is a sub-optimal choice in terms of margin and accuracy.
Note that for a given mixed point $\vx^{\op{mix}} = \lambda \vx_i+(1-\lambda) \vx_j$, the conventional labeling method uses a linear interpolation of labels of original samples, represented as $\vy^{\op{lin}} = \lambda \vy_i+(1-\lambda) \vy_j$. We call this conventional method as \emph{linear labeling}. Here we compare this with an alternative labeling dubbed as \emph{logistic labeling}, represented as
$\vy^{\op{log}} = \rho \vy_i+(1-\rho) \vy_j$ where $\rho = \frac{1}{1 + \exp\{-2  (\lambda - 1/2) / \sigma^2\}}$ for some $\sigma > 0$.
We theoretically/empirically show that mixup with linear labeling performs worse than mixup with logistic labeling, in various synthetic/real datasets.
Below we start with analyzing mixup with linear/logistic labeling methods for a synthetic dataset. 

\begin{example}\label{ex:3_dots}
Consider a dataset with three data points illustrated in Fig.~\ref{Fig:toy_3_class}a, where the feature-label pairs are defined as $(\vx_1, \vy_1) = ([-d, +d],  \ve_1)$, $(\vx_2, \vy_2) = ([+d, +d],  \ve_2)$, and $(\vx_3, \vy_3) = ([-d, -d], \ve_3)$ with $d=5$. We train softmax regression model $\mW = [\vw_1^T; \vw_2^T; \vw_3^T] \in \mathbb{R}^{3 \times 2}$
for vanilla training and mixup with linear/logistic labeling. For a given feature $\vx = [x^{(1)}, x^{(2)}]$, the prediction score of each class is denoted as 
$\vp =[p_1, p_2, p_3] = \text{softmax}(\mW \vx) = \frac{1}{\sum_{i=1}^3 \exp(\vw_i^T \vx) }[e^{\vw_1^T \vx}, e^{\vw_2^T \vx}, e^{\vw_3^T \vx} ]$.
In Fig.~\ref{Fig:toy_3_class}b,  Fig.~\ref{Fig:toy_3_class}c and Fig.~\ref{Fig:toy_3_class}d, we compare the decision boundaries of vanilla training, mixup with linear labeling, and mixup with logistic labeling. 
Note that mixup with linear labeling has much smaller margin than vanilla training, while mixup with logistic labeling enjoys a larger margin than vanilla training.
\end{example}

The effect of linear/logistic labeling methods on the margin of a classifier can be explained as follows.
Recall that the softmax regression finds the model that minimizes the cross entropy loss between the prediction $\vp$ and the label $\vy$, and we achieve the minimum when $\vp = \vy$ holds.
In Fig.~\ref{Fig:toy_3_class}a, consider mixing $\vx_2 $ and $\vx_3$ with coefficient $\lambda \in [0,1]$ along the line $x^{(2)} = x^{(1)}$, generating $\vx^{\op{mix}} = \lambda \vx_2 + (1-\lambda)\vx_3 = [(2\lambda-1)d, (2\lambda-1)d]$. 

As in Fig.~\ref{Fig:toy_3_class}e, mixup with linear labeling assigns the label $\vy^{\op{lin}}=[y_1, y_2, y_3] = [0, \lambda, 1-\lambda]$ for the mixup points on the line $x^{(2)} = x^{(1)}$. The model $\mW$ is trained in a way that $\vp$ resembles $\vy^{\op{lin}}$, \ie set $p_1 = 0$ and set both $p_2$ and $p_3$ as a linear function of $\vx$ along the line $x^{(2)} = x^{(1)}$. 
This is true when $\vw_2$ and $\vw_3$ are close enough and symmetric about the line $x^{(2)} = - x^{(1)}$, as in Fig.~\ref{Fig:toy_3_class}c; in such case, we have $\exp(\vw_2^T \vx) \simeq 1 + \vw_2^T \vx$ and $\exp(\vw_3^T \vx) \simeq 1  + \vw_3^T \vx = 1 - \vw_2^T \vx$ for $\vx$ satisfying $x^{(2)} = x^{(1)}$.
Then, we have 
$p_2 = \frac{\exp(\vw_2^T \vx)}{ \exp(\vw_1^T \vx) + \exp(\vw_2^T \vx) + \exp(\vw_3^T \vx)} \simeq \frac{1+\vw_2^T \vx}{ 0 + 1 + \vw_2^T \vx + 1 - \vw_2^T \vx} = \frac{1}{2} (1 + \vw_2^T \vx),$
which is linear in $\vx$. Similarly, $p_3 = \frac{1}{2} (1+\vw_3^T \vx)$ holds.
Specifically, if we set $\vw_2 = [-1 + \frac{1}{2d}, 1 + \frac{1}{2d}]$ and 
$\vw_3 = [-1 - \frac{1}{2d}, 1 - \frac{1}{2d}]$, then for the mixed points $\vx^{\op{mix}} = [(2\lambda-1)d, (2\lambda-1)d]$, we have 
$p_2 = \frac{1}{2} (1+\vw_2^T \vx^{\op{mix}}) = \lambda = y_2$ and $p_3 = \frac{1}{2} (1+\vw_3^T \vx^{\op{mix}}) = 1 -\lambda = y_3$. This implies that the model $\mW$ trained to set $\vp =\vy^{\op{lin}}$ will look like the solution in Fig.~\ref{Fig:toy_3_class}c, especially when $d$ is large. 
Thus, fitting the softmax regression model to the linear labeling strategy of mixup reduces the margin in this toy dataset.

\begin{figure}[t]
	\vspace{-2mm}
	\centering
	\small
	\includegraphics[width=.9\linewidth ]{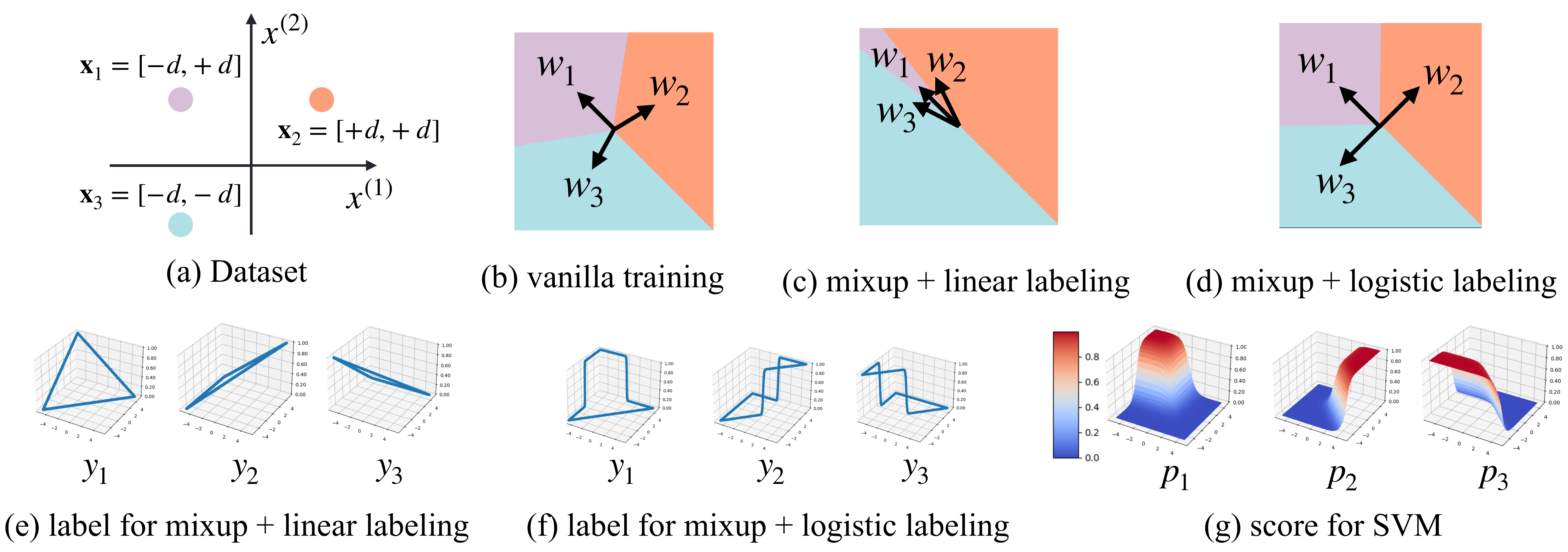}
	\caption{The effect of linear/logistic labeling in mixup training for a toy dataset in (a). 
	From the decision boundary plots in (b), (c), (d), it is shown that using the linear labeling method in mixup is significantly reducing the margin of the purple class than vanilla training, while the logistic labeling method leads to the max-margin classifier.
	This phenomenon is explained  
	in Section~\ref{sec:softmax_regression}, based on the label/score plots in (e), (f), and (g).
	This example provides two messages: (1) the conventional linear labeling is not an appropriate choice for mixup, to maximize the margin of a classifier, (2) using an appropriate label for the mixup samples can lead the classifier to become a max-margin solution.
	}
	\label{Fig:toy_3_class}
	\vspace{-3mm}
\end{figure}

\begin{wrapfigure}{r}{0.35\textwidth}
\vspace{-5mm}
  \begin{center}
    \includegraphics[width=0.32\textwidth]{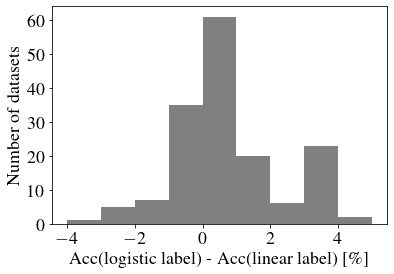}
  \end{center}
  \vspace{-3mm}
  \caption{
  Comparison of classification accuracy of linear labeling and logistic labeling in mixup. 
  Among 160 tested datasets in OpenML, 87 of them has positive gain by using logistic labeling instead of linear labeling, 23 of them has zero gain, and 50 of them has negative gain.
  This shows the linear labeling conventionally used in mixup is sub-optimal for a large number of low-dimensional real datasets.
  }
  \vspace{-8mm}
  \label{Fig:labeling_openML}
\end{wrapfigure}
We now explain how the logistic labeling enjoys a large margin in Fig.~\ref{Fig:toy_3_class}d. 
Consider mixup points $\vx^{\op{mix}} = \lambda \vx_2 + (1-\lambda)\vx_3 = [ (2\lambda-1) d, (2\lambda-1) d]$, generated by mixing $\vx_2$ and $\vx_3$, where $\lambda \in [0,1]$. %
As in Fig.~\ref{Fig:toy_3_class}f, the logistic labeling with $\sigma = \frac{1}{2\sqrt{d}}$ assigns the label $\vy^{\op{log}}=[y_1, y_2, y_3] = [0, \rho, 1-\rho]$ for these mixup points, where $\rho = \frac{1}{1 + \exp (-8d (\lambda - 1/2))}$.
Under this setting, the model $\mW$ is trained in a way that $\vp$ resembles $\vy^{\op{log}}$, \ie set $p_1 = 0$ and set both $p_2$ and $p_3$ as a logistic function of $\vx$ along the line $x^{(2)} = x^{(1)}$. This is true for the support vector machine (SVM) solution, $\vw_1=[-1, 1], \vw_2 = [1,1],$ and $\vw_3 = [-1,-1]$, which has $p_2 = \frac{\exp(\vw_2^T \vx^{\op{mix}})}{ \exp(\vw_1^T \vx^{\op{mix}}) + \exp(\vw_2^T \vx^{\op{mix}}) + \exp(\vw_3^T \vx^{\op{mix}})} = \frac{\exp( 2d (2\lambda-1) ) }{ 0 + \exp( 2d  (2\lambda-1)   ) + \exp(-2d (2\lambda-1)  )} = 
\frac{1}{1+\exp (-8 d (\lambda - 1/2)  ) } = y_2$, and similarly $p_3 = y_3$. One can confirm that the logistic label $\vy$ in Fig.~\ref{Fig:toy_3_class}f resembles the score $\vp$ of SVM solution in Fig.~\ref{Fig:toy_3_class}g, which corroborates the fact that logistic labeling guides us to achieve the maximum margin.

The above analysis shows that the \emph{linear labeling} method is harming the margin of a mixup-trained classifier, while the \emph{logistic labeling} method is allowing mixup to enjoy the maximum margin. This clearly shows that the conventional labeling method of mixup is sub-optimal, and an appropriate re-labeling method improves the margin significantly.

Now we confirm the effect of linear/logistic labeling on the classification accuracy (generalization performance) of mixup-trained classifiers for real datasets in OpenML~\citep{OpenML2013}.
To be specific, we compared 
the accuracy difference of two labeling schemes for 160 datasets in OpenML %
having no more than 20 features, when we use the logistic regression model. 
Here, in order to decouple the effect of labeling and the effect of manifold intrusion, we removed the mixed points incurring the manifold intrusion by following the criterion in Section~\ref{sec:manifold_intrusion_issue}: a mixed point $\vx^{\op{mix}} = \lambda \vx_1 + (1-\lambda) \vx_2$ is regarded as causing manifold intrusion if the label of the nearest neighbor $\vx^{\op{nn}} = \argmin_{\vx \in X} d(\vx, \vx^{\op{mix}})$ is different from $\vy_1$ and $\vy_2$.
As shown in Fig.~\ref{Fig:labeling_openML}, it turns out that the accuracy gain by using logistic labeling, i.e., (accuracy of the logistic labeling) - (accuracy of the linear labeling), is positive for 87 datasets, zero for 23 datasets, and negative for 50 datasets. In other words, using logistic labeling instead of linear labeling improves the accuracy for more than half of the tested real datasets in OpenML.
This experimental results show that the conventional method of labeling mixup points is not optimal in terms of accuracy in numerous low-dimensional real datasets.

\section{\gl{}}\label{sec:GenLabel}
\vspace{-2mm}

In the previous section, we observed two main issues of mixup. First, mixup points may intrude the manifold of a third class, which is so-called \emph{manifold intrusion} issue. This is due to the fact that mixup blindly interpolates randomly chosen two samples, without the knowledge on the underlying data distribution.
Second, the conventional method of labeling mixup samples is sub-optimal, in terms of margin and accuracy. 

Motivated by these observations, we propose \gl{}, a method of re-labeling mixup samples based on the underlying data distribution estimated by generative models.
The suggested algorithm contains three steps. First, we estimate the class-conditional data distribution $p_c(\vx)$ for each class $c$. Second, we apply the conventional mixup-based data augmentation, generating mixup sample $\vx^{\op{mix}}$ originally labeled as  $\vy^{\op{mix}}$.
Finally, we relabel the generated mixup sample $\vx^{\op{mix}}$ based on the estimated class-conditional likelihood, \ie we define the new label as 
$\vy^\text{gen} = \text{softmax}( \log {p_1}(\vx^\text{mix}), \cdots, \log {p_k}(\vx^\text{mix}))$.
This new label is called \gl{} since it makes use of generative models for labeling.

We consider the case when the generative model has the explicit density function $p_c(\vx)$, e.g., Gaussian mixture model~\citep{shalev2014understanding} and kernel density estimator~\citep{friedman2017elements}.
In Section~\ref{sec:GenLabel_input}, we provide the \gl{} algorithm when we use generative models in the input feature space. Then, the algorithm is extended to the case when the generative models learn the latent feature space, in Section~\ref{sec:GenLabel_latent}.
Note that the suggested \gl{} is a re-labeling method, and we follow the mixing strategy of mixup by default. Throughout the paper, the scheme called ``\gl{}'' refers to ``mixup+\gl{}'', unless specified otherwise.

\begin{algorithm}[t!]
	\small
	\textbf{Input} Dataset $S = \{(\vx_i, \vy_i)\}_{i=1}^n$, learning rate $\eta$, loss ratio $\gamma$
	\\
	\textbf{Output} Trained discriminative model $f_{\theta} (\cdot)$
	\begin{algorithmic}[1]
	    \STATE $\theta \leftarrow$ Random initial model parameter
	    \STATE $p_c(\vx) \leftarrow$ Density estimated by generative model for input feature $\vx \in X$%
	    , conditioned on class $c \in [k]$
        \FOR{$(\vx_i, \vy_i), (\vx_j, \vy_j) \in S$}
        \STATE $(\vx^{\op{mix}}, \vy^{\op{mix}}) = (\lambda \vx_i + (1-\lambda) \vx_j, \lambda \vy_i + (1-\lambda) \vy_j )$
        \STATE $\vy^{\op{gen}} \leftarrow \sum_{c=1}^k \frac{  p_{c}(\vx^{\op{mix}})   }{ \sum_{c'=1}^k p_{c'}(\vx^{\op{mix}})} \ve_{c}$
        \STATE $\theta \leftarrow \theta - \eta  \nabla_{\theta} \{ \gamma \cdot  \ell_{\text{CE}}(\vy^{\op{gen}}, f_{\theta}(\vx^{\op{mix}})) + (1-\gamma) \cdot  \ell_{\text{CE}}(\vy^{\op{mix}}, f_{\theta}(\vx^{\op{mix}}))\}$ 
        \ENDFOR
	\end{algorithmic}
	\caption{GenLabel}
	\label{Algo:GenLabel_input}
\end{algorithm}

\subsection{Vanilla setting: when generative models learn the density in the input feature space}
\label{sec:GenLabel_input}
\vspace{-1mm}

Given a dataset $S$, we first train class-conditional generative model, thereby learning the underlying data distribution $p_c(\vx)$. 
Then, for randomly chosen data pair $(\vx_i, \vy_i), (\vx_j, \vy_j) \in S$, we apply mixup scheme, generating the mixed feature $\vx^{\op{mix}} = \lambda \vx_i + (1-\lambda) \vx_j$ and the mixed label $\vy^{\op{mix}} = \lambda \vy_i + (1-\lambda) \vy_j$. Here, the mixing coefficient follows the beta distribution, i.e., $\lambda \sim \text{Beta}(\alpha, \alpha)$ for some $\alpha > 0$.
Finally, we re-label this augmented data $\vx^{\op{mix}}$ based on the estimated class-conditional likelihood $p_c(\vx^{\op{mix}})$ for class $c \in [k]$. 
To be specific, we label the mixed point $\vx^{\op{mix}}$ as 
\begin{align}\label{eqn:GenLabel}
\vy^{\op{gen}} = \sum_{c=1}^k \frac{  p_{c}(\vx^{\op{mix}})   }{ \sum_{c'=1}^k p_{c'}(\vx^{\op{mix}})} \ve_{c},
\end{align}
which is nothing but the softmax of $\{ \log (p_{c}(\vx^{\op{mix}})) \}_{c=1}^k$.
Note that $\frac{{p_c}(\vx^\text{mix})}{\sum_{c'=1}^k {p_{c'}}(\vx^\text{mix})}$ in~(\ref{eqn:GenLabel}) is equal to the posterior probability $\mathbb{P}(y=c | \vx^\text{mix}) = \frac{p_c(\vx^\text{mix}) \mathbb{P}(y=c)}{\sum_{c'=1}^k p_{c'}(\vx^\text{mix}) \mathbb{P}(y=c') }$, when we have a balanced dataset, i.e., $\mathbb{P}(y=c) = \mathbb{P}(y=c')$ for all classes $c, c' \in [k]$. \bmt{
Thus, \gl{} assigns the posterior probability of each class for a given mixed sample $\vx^\text{mix}$, when the dataset is balanced.
}

Since our generative model is an imperfect estimate on the data distribution, $\vy^{\op{gen}}$ may be incorrect for some samples. Thus, we can use a combination of mixup labeling and the suggested labeling, \ie define the label of mixed point as $\gamma \vy^{\op{gen}} + (1-\gamma) \vy^{\op{mix}}$ for some $\gamma \in [0,1]$. Note that our scheme reduces to the mixup labeling scheme when $\gamma = 0$. 
Using the relabeled augmented data, the algorithm trains the classification model $f_{\theta} : \mathbb{R}^n \rightarrow [0,1]^k$ that predicts the label $\vy = [y_1, \cdots, y_k]$ of the input data. 
Here, the cross-entropy loss $\ell_{\text{CE}} (\cdot)$ is used while optimizing the model.
The pseudocode of \gl{} is provided in Algorithm~\ref{Algo:GenLabel_input}.

In summary, the proposed scheme is a novel label correction method for mixup, which first learns the data distributions for each class using class-conditional generative models, and then re-label the mixup data based on the conditional likelihood of the mixup data sampled from each class. More precisely, \gl{} sets the label of a mixup data as the softmax of the class-conditional log-likelihood, which matches with the posterior probability for the balanced datasets.

\subsection{When generative models learn the density in the latent feature space}\label{sec:GenLabel_latent}

\gl{} described in Algorithm~\ref{Algo:GenLabel_input} assumes that the generative model learns the input feature space. However, for some datasets, it is beneficial to learn the underlying distribution in the latent feature space. For such cases, we can apply \gl{} combined with generative models for the latent feature space, as below.

Consider a discriminative model $f_{\theta} = f_{\theta}^{\op{cls}} \circ f_{\theta}^{\op{feature}}$ parameterized by $\theta$, which is composed of the feature extractor part $f_{\theta}^{\op{feature}}$ and the classification part $f_{\theta}^{\op{cls}}$. For example, we can consider $f_{\theta}^{\op{feature}}$ as the neural network from the input layer to the penultimate layer, and $f_{\theta}^{\op{cls}}$ as the final fully-connected layer. Consider another discriminative model $f_{\phi} = f_{\phi}^{\op{cls}} \circ f_{\phi}^{\op{feature}}$ having the same architecture with $f_{\theta}$.
We first randomly initialize the model parameters $\theta$ and $\phi$, and train only the second model $f_{\phi}$ using the vanilla training method on dataset $S = \{(\vx_i, \vy_i) \}_{i=1}^n$. Given the trained feature extractor $f_{\phi}^{\op{feature}}$, we train a class-conditional generative model for the latent feature $\vz = f_{\phi}^{\op{feature}}(\vx)$%
, and denote the learned density for class $c$ by $p_{c}(\vz)$.
Finally, we train the first model $f_{\theta}$ using the following manner.
We first follow the mixup process: for $(\vx_i, \vy_i), (\vx_j, \vy_j) \in S$, we generate the augmented data $\vx^{\op{mix}} = \lambda \vx_i + (1-\lambda) \vx_j$ having label $\vy^{\op{mix}} = \lambda \vy_i + (1-\lambda) \vy_j$.
Then, we re-label this augmented data by 
$\vy^{\op{gen}} = \sum_{c=1}^k \frac{  p_{c}(\vz^{\op{mix}})   }{ \sum_{c'=1}^k p_{c'}(\vz^{\op{mix}})} \ve_{c}$,
where $\vz_{\op{mix}} = f_{\phi}^{\op{feature}}(\vx_{\op{mix}})$.
The remaining part for optimizing $\theta$ is identical to that of vanilla \gl{} using generative models for the input feature. The pseudocode of this \gl{} variant (for the latent feature) is given in Algorithm~\ref{Algo:GenLabel_latent}.

\begin{algorithm}[t!]
	\small
	\textbf{Input} Dataset $S = \{(\vx_i, \vy_i)\}_{i=1}^n$, input feature set $X = \{\vx_i\}_{i=1}^n$, learning rate $\eta$, loss ratio $\gamma$
	\\
	\textbf{Output} Trained discriminative model $f_{\theta}$
	\begin{algorithmic}[1]
	    \STATE $\theta \leftarrow$ Random initial parameter for model $f_{\theta} =   f_{\theta}^{\op{cls}} \circ f_{\theta}^{\op{feature}}$
	    \STATE $\phi \leftarrow$ Vanilla-trained parameter for model $f_{\phi} =  f_{\phi}^{\op{cls}} \circ f_{\phi}^{\op{feature}}$
	    \STATE $p_c(\vz) \leftarrow$ Density estimated by generative model for latent feature $\vz \in f_{\phi}^{\op{feature}}(X)$%
	    , conditioned on class $c \in [k]$
        \FOR{$(\vx_i, \vy_i), (\vx_j, \vy_j) \in S$}
        \STATE $(\vx^{\op{mix}}, \vy^{\op{mix}}) \leftarrow (\lambda \vx_i + (1-\lambda) \vx_j, \lambda \vy_i + (1-\lambda) \vy_j )$
        \STATE $\vz^{\op{mix}} \leftarrow f_{\phi}(\vx^{\op{mix}})$
        \STATE  $\vy^{\op{gen}} \leftarrow \sum_{c=1}^k \frac{  p_{c}(\vz^{\op{mix}})   }{ \sum_{c'=1}^k p_{c'}(\vz^{\op{mix}})} \ve_{c}$
        \STATE $\theta \leftarrow \theta - \eta  \nabla_{\theta} \{ \gamma \cdot  \ell_{\text{CE}}(\vy^{\op{gen}}, f_{\theta}(\vx^{\op{mix}})) + (1-\gamma) \cdot  \ell_{\text{CE}}(\vy^{\op{mix}}, f_{\theta}(\vx^{\op{mix}}))\}$ 
        \ENDFOR
	\end{algorithmic}
	\caption{\gl{} (using generative models for the latent feature)  
	}
	\label{Algo:GenLabel_latent}
\end{algorithm}

\section{
Analysis of \gl{}
}
\vspace{-2mm}

We analyze the effect of \gl{} in various perspectives. 
In Section~\ref{sec:GenLabel_solves_intrusion}, we visualize \gl{} for toy datasets, empirically showing that \gl{} fixes the manifold intrusion issue of mixup. 
In Section~\ref{sec:GenLabel_solves_linear_labeling_issue}, we provide mathematical analysis on the margin achievable by \gl{}, showing that \gl{} solves the margin reduction issue of linear labeling in mixup.
Finally, we observe how this margin improvement by \gl{} allows us to get a model that is robust against adversarial attacks. 
In Section~\ref{sec:GenLabel_improves_rob},
we theoretically show that \gl{} improves the adversarial robustness of mixup on logistic regression models and fully-connected ReLU networks.

\subsection{GenLabel solves the label conflict issue of manifold-intruding mixup points}\label{sec:GenLabel_solves_intrusion}
\vspace{-1mm}

\begin{figure}[t]
    \vspace{-2mm}
	\centering
    \includegraphics[height=45mm ]{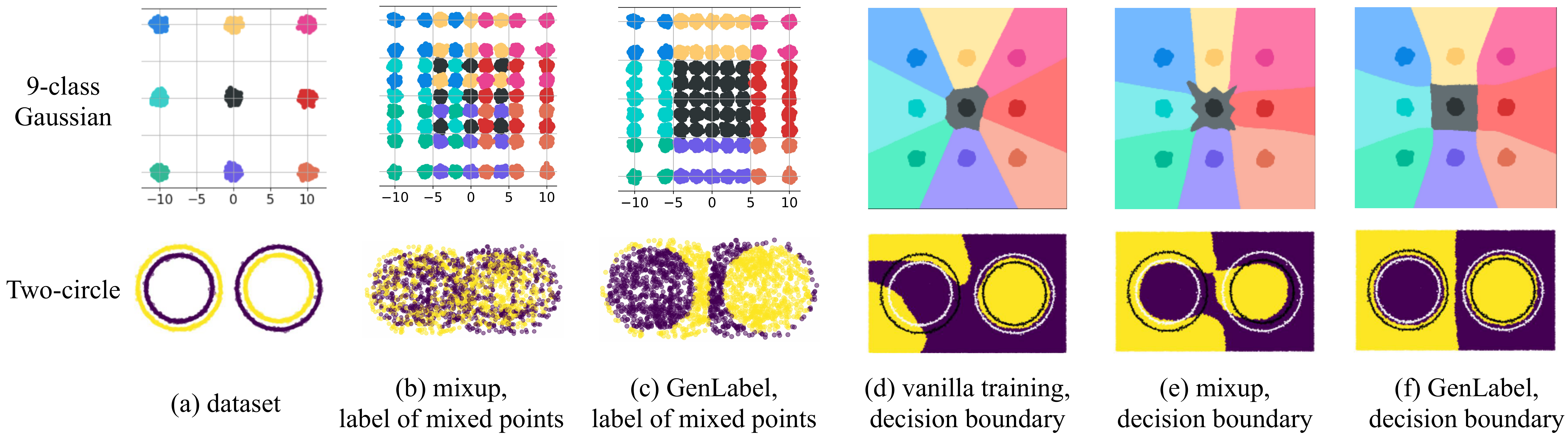}
	\caption
	{
	\bmt
	    {
	Comparison of vanilla training, mixup and the suggested \gl{} for toy datasets (top: 9-class Gaussian, bottom: Two-circle). 
	(a): training data points.
    (b): top-1 label of mixed points for the conventional soft-label $\vy^{\op{mix}}$. Here, the top-1 label $y^{\op{top-1}} = \arg\max_{c \in [k]} y_c$ of a soft-label $\vy = [y_1, \cdots, y_k]$ represents the index of the largest element.
    (c): top-1 label of mixed points for the suggested \gl{} $\vy^{\op{gen}}$.
    Figures in (b) and (c) show that the conventional labeling method $\vy^{\op{mix}}$ causes the label conflict issue for a large number of mixup points, while the suggested label $\vy^{\op{gen}}$ does not suffer from the conflict issue.
    (d), (e), (f): decision boundaries of vanilla training, mixup and \gl{}. 
    In the decision boundary plots, we can find that several classes/samples have small margins for vanilla training and mixup, while the suggested \gl{} does not have such issue. 
        }
    }
	\label{Fig:intrusion_fix_margin}
	\vspace{-3mm}
\end{figure}

As discussed in Section~\ref{sec:manifold_intrusion_issue}, one of the issues degrading the performance of mixup is manifold intrusion, which happens when the mixup sample generated by mixing samples from two classes is overlapping with another true sample from the third class. Here we show that \gl{} solves this \emph{label conflict} of manifold-intruding mixup samples.

Consider the datasets given in Fig.~\ref{Fig:intrusion_fix_margin}a: we have nine classes of 2-dimensional Gaussian dataset on the top row, and two classes having two circles at each class on the bottom row; we call the top one as ``9-class Gaussian'' and the bottom one as ``Two-circle'' dataset. Note that both datasets contain numerous mixed points suffering from the manifold intrusion, e.g., the mixed point of blue and orange samples lie on the black class in 9-class Gaussian dataset. 

In Fig.~\ref{Fig:intrusion_fix_margin}b and Fig.~\ref{Fig:intrusion_fix_margin}c, we illustrate the top-1 label (denoted by $y^{\op{top-1}}$) of a mixed point, for the conventional labeling in mixup and the suggested \gl{} scheme.
Given a soft-label $\vy=[y_1, \cdots, y_k]$, the top-1 label is defined as $y^{\op{top-1}} = \argmax_{c \in [k]} y_c$.
For the 9-class Gaussian data at the top row, we set the mixing coefficient as $\lambda=0.6$ for the purpose of illustration.
As shown in Fig.~\ref{Fig:intrusion_fix_margin}b, the conventional labeling method causes the label conflict issue for a large number of mixup samples. This issue has been resolved by \gl{} as shown in Fig.~\ref{Fig:intrusion_fix_margin}c. One can confirm that the label of mixed points assigned by \gl{} matches with the label of maximum margin classifier for each dataset.

We also checked the effect of this relabeling method on the margin of a classifier.
Fig.~\ref{Fig:intrusion_fix_margin}d, Fig.~\ref{Fig:intrusion_fix_margin}e and Fig.~\ref{Fig:intrusion_fix_margin}f show the decision boundary of vanilla training, mixup (with original labeling) and mixup+\gl{}, respectively.
While vanilla training and mixup have small margin for some classes/samples, mixup combined with \gl{} enjoys a large margin for all samples. This shows that suggested label correction mechanism is guiding the classifier to have a large margin.

\subsection{GenLabel 
solves the margin reduction issue of the linear labeling in mixup}
\label{sec:GenLabel_solves_linear_labeling_issue}

\vspace{-1mm}
In Section~\ref{sec:softmax_regression}, we have shown that even when there is no manifold intrusion, the \emph{linear labeling} method used in mixup is a sub-optimal choice in terms of margin of a classifier. Interestingly, we found an example (in Fig.~\ref{Fig:toy_3_class}) when the mixup with linear labeling method is reducing the margin of the vanilla-trained model. 
Here we show that changing the \emph{linear labeling} method to the suggested \gl{} allows us to fix this issue and to achieve the maximum margin for toy datasets in Examples~\ref{ex:3_dots} and~\ref{ex:n+2_dots}.

We start with showing that \gl{}
reduces to the logistic labeling for the Gaussian data.

\begin{proposition}\label{Prop:Gaussian}
Consider a binary classification problem when the class-conditional data distribution is $(x|y=0) \sim \mathcal{N}(0, \sigma^2)$ and $(x|y=1) \sim \mathcal{N}(1, \sigma^2)$. Let $x^{\op{mix}} = \lambda$ be the mixed point generated by mixup. For small $\sigma > 0$, the label of $x^{\op{mix}}$ for mixup and \gl{} are
\begin{align*}
\vspace{-4mm}
\small{
y^{\op{mix}} = \lambda , \quad \quad \quad 
y^{\op{gen}}
= \frac{1}{1 + \exp(- (\lambda - 1/2)/\sigma^2)}.
}
\end{align*} 
\vspace{-4mm}
\end{proposition}
The proof of this proposition is given in Section~\ref{sec:proof_prop:Gaussian} in \sm. 
Note that the conventional label $y^{\op{mix}}$ is a linear function of $\lambda$, while the GenLabel $y^{\op{gen}}$ follows a logistic function of $\lambda$.

As shown in the analysis for Example~\ref{ex:3_dots}, the logistic labeling achieves the maximum margin for the dataset in Fig.~\ref{Fig:toy_3_class}a. Since each class of this dataset can be viewed as a Gaussian distribution with variance $\sigma^2 \rightarrow 0$, we can apply Proposition~\ref{Prop:Gaussian}. Thus, from the analysis in Section~\ref{sec:softmax_regression}, we can conclude that
\vspace{2mm}
\begin{align}\label{eqn:margin_compare}
    \op{margin}(\op{SVM}) = \op{margin}({\op{GenLabel}}) > \op{margin}({\op{vanilla}}) > \op{margin}({\op{mixup}})
\end{align}%
holds for the dataset in Fig.~\ref{Fig:toy_3_class}a, where 
SVM represents the support vector machine~\citep{cortes1995support} achieving the maximum $L_2$ margin. In other words, mixup combined with \gl{} achieves the maximum margin, while the conventional mixup (using the linear labeling method) is having even worse margin than the vanilla-trained model.

Below we provide another example satisfying (\ref{eqn:margin_compare}).

\begin{example}\label{ex:n+2_dots}
Consider a dataset $S = \{ (\vx_i, y_i) \}_{i=1}^{n+2}$, where the feature $\vx_i \in \mathbb{R}^2$ and the label $y_i \in \{+1, -1\}$ of each point is specified in Fig.~\ref{Fig:toy_label_correction_dataset}.
Let 
$\bm{\theta} = (r \cos \phi, r \sin \phi)$ be the model parameter for a fixed $r > 0$. Fig.~\ref{Fig:loss_minimizer} shows 
$\phi^{\star} = \argmin_{\phi} \ell(r, \phi)$ for various $r$, where $\ell$ is the logistic loss applied to the (augmented) dataset.
It turns out that the optimal $\phi^{\star}$ of \gl{} approaches to the SVM solution $\phi_{\op{svm}} = 0.25 \pi $ as $r$ increases.
The detailed derivation of $\phi^{\star}$ for each scheme is given in Section~\ref{Sec:toy_proof} in \sm.
\end{example}

\begin{figure}[t]
    \vspace{-2mm}
	\centering
	\subfloat[][\centering{Dataset}]{\includegraphics[height=20mm ]{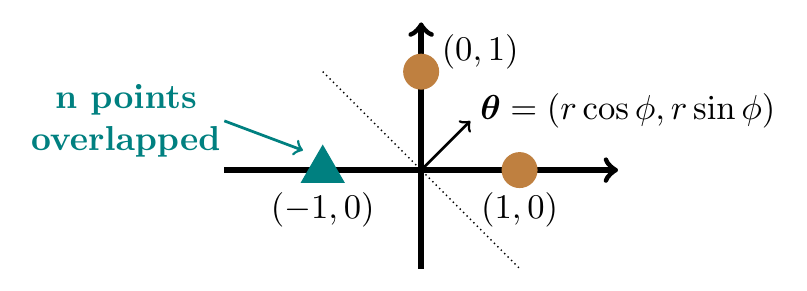}\label{Fig:toy_label_correction_dataset}} 
	\quad \quad \quad
 	\subfloat[][\centering{Comparison of $\phi^{\star}$} ]{\includegraphics[height=25mm ]{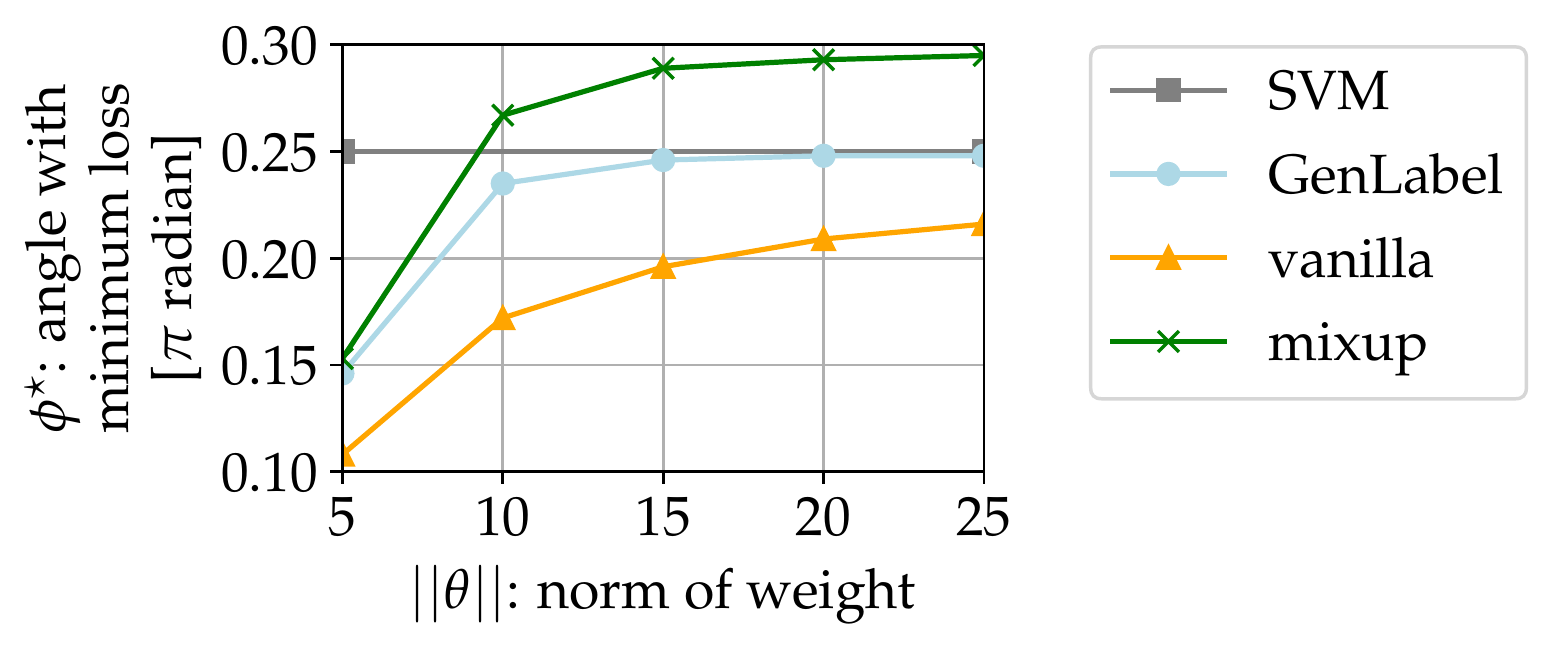}\label{Fig:loss_minimizer}}
	\vspace{-2mm}
	\caption
	{
	Logistic regression on a toy dataset. (a) The feature-label pairs are defined as
	$(\vx_1, y_1) = ([1, 0],  +1)$, $(\vx_2, y_2) = ([0, 1],  +1)$, and $(\vx_i, y_i) = ([-1, 0], -1)$ for    $i=3, 4, \cdots, n+2$.
	(b)
	When $\lVert \bm{\theta} \rVert =25$, we have $\phi^{\star}_{\op{GenLabel}} = 0.25 \pi, \phi^{\star}_{\op{mixup}} = 0.30 \pi$ and $\phi^{\star}_{\op{vanilla}} = 0.22 \pi$. Note that \gl{} quickly saturates to the SVM solution $\phi_{\op{svm}} = 0.25\pi$ as $\lVert \bm{\theta} \rVert$ increases, while vanilla scheme slowly saturates to $\phi_{\op{svm}}$ and original mixup does \textit{not} converge to $\phi_{\op{svm}}$.
    }
	\label{Fig:toy_label_correction}
	\vspace{-3mm}
\end{figure}

\begin{remark}
In the asymptotic regime of large $r = \lVert \bm{\theta} \rVert$, the original mixup does not converge to the max-margin solution, while the mixup relabeled by \gl{} approaches to the max-margin solution. 
\end{remark}

The analysis on Examples~\ref{ex:3_dots} and~\ref{ex:n+2_dots} shows that (1) the original mixup method has a smaller margin compared with vanilla training, and (2) simply changing the label of the mixed points (using \gl{}) can fix this issue and approaches to the max-margin solution.

\subsection{GenLabel improves the adversarial robustness of mixup}\label{sec:GenLabel_improves_rob}
\vspace{-1mm}

Here we analyze the adversarial robustness of a model trained by mixup+\gl{}, and show that \gl{} is beneficial for improving the robustness of mixup under the logistic regression models and the fully-connected (FC) ReLU networks.
We first describe the basic setting considered in our analysis, and then provide the results.
All proofs are given in Section~\ref{Sec:proofs} in \sm.

\paragraph{Basic setting and notations}

Consider $d$-dimensional Gaussian dataset defined as 
$(\vx|y=0)\sim \mathcal{N}(-\ve_1,\frac{\boldsymbol{\Sigma}}{\sigma_1^2})$ and $(\vx|y=1)\sim \mathcal{N}(\ve_1,\frac{\boldsymbol{\Sigma}}{\sigma_2^2})$, where $\Sigma_{ij} = 1$ for $i=j$ and $\Sigma_{ij} = \tau$ for $i \neq j$. 
Here we assume that $-1<\tau<1$, $\tau \notin \{ \frac{-1}{d-1}, \frac{-1}{d-2} \}$ and  $\sigma_2=c\sigma_1$ with $2-\sqrt{3}<c < 2+\sqrt{3}$.
We consider the loss function 
$\ell(\boldsymbol{\theta},(\vx, y))=h\left(f_{\boldsymbol{\theta}}(\vx)\right)-y f_{\boldsymbol{\theta}}(\vx)$, 
where  
$f_{\boldsymbol{\theta}}(\vx)$ 
is the prediction of a model parameterized by $\theta$ for a given input $\vx$,
and 
$h (w)=\log(1+\exp(w))$.

We assume the following labeling setting: when we mix $(\vx_i, y_i)$ and $(\vx_j, y_j)$ which generates the mixed point $\vx^{\op{mix}}_{ij} = \lambda \vx_i + (1-\lambda) \vx_j$, we label it as $y^{\op{mix}}_{ij} = y$ if $y_i = y_j = y$, and we label it as $y_{ij}^{\op{gen}}$ in (\ref{eqn:GenLabel}) otherwise. We assume the mixing coefficient follows the uniform distribution $\lambda \sim \op{Unif}[0,1]$, \textit{i.e.}, $\lambda \sim \op{Beta}(\alpha, \alpha)$ with $\alpha=1$.

For a given model parameter $\boldsymbol{\theta}$ and the dataset $S$, we define the notations for several losses as below. The standard loss is denoted by $L_n^{\op{std}} (\boldsymbol{\theta}, S) =\frac{1}{n}\sum_{i=1}^n \ell(\boldsymbol{\theta}, \vz_i)$. The mixup loss and \gl{} loss are denoted by $L_n^{\op{mix}} (\boldsymbol{\theta}, S) = \frac{1}{n^2}\sum_{i,j=1}^n \mathbb{E}_{\lambda} [ \ell(\boldsymbol{\theta}, \vz_{ij}^{\op{mix}} ) ]$ and $L_n^{\op{gen}} (\boldsymbol{\theta}, S) =  \frac{1}{n^2}\sum_{i,j=1}^n \mathbb{E}_{\lambda} [ \ell(\boldsymbol{\theta}, \vz_{ij}^{\op{gen}} ) ]$, respectively, where $\vz_{ij}^{\op{mix}} = ({\vx}_{ij}^{\op{mix}}, {y}_{ij}^{\op{mix}})$ and $\vz_{ij}^{\op{gen}} = ({\vx}_{ij}^{\op{mix}}, {y}_{ij}^{\op{gen}})$. 
The adversarial loss with $L_2$ attack of radius $\eps \sqrt{d}$ is defined as 
$L_{n}^{\op{adv}} (\boldsymbol{\theta}, S) = \frac{1}{n} \sum_{i=1}^{n} \max _{\left\|\boldsymbol{\delta}_{i}\right\|_{2} \leq \varepsilon \sqrt{d}} \ell\left(\boldsymbol{\theta},\left(\vx_{i}+\boldsymbol{\delta}_{i}, y_{i}\right)\right)$.

\paragraph{Mathematical results}

Before stating our result, we denote the Taylor approximation of mixup loss by $\tilde{L}_n^{\op{mix}}(\boldsymbol{\theta},S)$, the expression of which is given in Lemma~\ref{Lemma: mixup loss} in \sm. Similarly, the Taylor approximation of each term in the adversarial loss is denoted by $\tilde{\ell}_{\op{adv}}(\eps \sqrt{d}, (\boldsymbol{x_i}, y_i))$, which is expressed in Lemma~\ref{Lemma:adv_loss_approx} in \sm.
Finally, the approximation of \gl{} loss, denoted by $\tilde{L}_n^{\op{gen}}(\boldsymbol{\theta},S)$, is expressed in Lemma~\ref{Lemma: GenLabel loss for 2 class} in \sm.

In the theorem below, we state the relationship between Taylor approximations of mixup loss, \gl{} loss, and adversarial loss, for the logistic regression models. In this theorem, we consider the set of model parameters
\[\Theta \coloneqq \{\boldsymbol{\theta} \in \mathbb{R}^d: 
(2y_i-1)f_{\boldsymbol{\theta}}(\vx_i)\geq 0 \text{ for all }i=1,2,\cdots,n\}\]
which contains the set of all $\boldsymbol{\theta}$ with zero training errors.

\begin{wrapfigure}{r}{0.35\textwidth}
\vspace{-5mm}
  \begin{center}
    \includegraphics[width=0.3\textwidth]{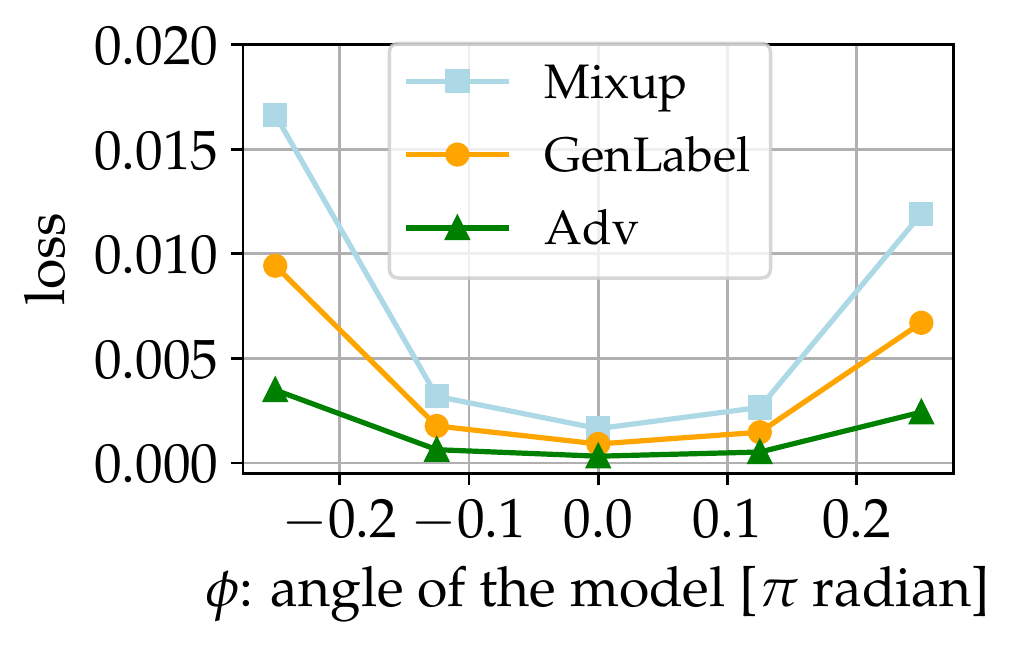}
  \end{center}
  \vspace{-3mm}
  \caption{
  Comparison between the mixup loss, \gl{} loss and adversarial loss, for the logistic regression model $\boldsymbol{\theta} = (10 \cos \phi, 10 \sin \phi)$ on a two-dimensional Gaussian dataset. This plot coincides with the result in Theorem~\ref{Theorem:robustness_logreg}.  
  }
  \vspace{-5mm}
  \label{Fig:compare_mixup_loss_GenLabel_loss}
\end{wrapfigure}

\begin{theorem}
\label{Theorem:robustness_logreg}
Consider the logistic regression setting having $f_{\boldsymbol{\theta}}(\vx)=\boldsymbol{\theta}^T \vx$. 
Suppose there exists a constant $c_x>0$ such that $\Vert \vx_i\Vert_2 \geq c_x $ for all $i\in \{1,2,\cdots,n\}$.
Then, in the asymptotic regime of large $\sigma_1$, for any $\boldsymbol{\theta} \in \Theta$, we have
\begin{equation*}
\tilde{L}_n^{\op{mix}}(\boldsymbol{\theta},S) > \tilde{L}_n^{\op{gen}}(\boldsymbol{\theta},S) \geq 
\frac{1}{n}\sum_{i=1}^n \tilde{\ell}_{\op{adv}}(\delta_{\op{gen}},(\boldsymbol{x_i},y_i)).
\end{equation*}
Here, 
$\delta_{\op{gen}} = R\cdot c_x A^i_{\sigma_1,c,\tau,d}$ with $R = \min_{i\in \{1,\cdots,n\}}|\cos (\boldsymbol{\theta},\boldsymbol{x_i})|$, where $A^i_{\sigma_1,c,\tau,d}$ is defined in (\ref{A,B}) in the \sm.
\end{theorem}

\bmt{
This shows that the adversarial loss of a model is upper bounded by the \gl{} loss of the model, i.e., if we find a model with \gl{} loss smaller than or equal to a threshold $l_{\op{th}}$, then the adversarial loss of this model is at most $l_{\op{th}}$. 
}
Moreover, compared with the mixup loss, the \gl{} loss is a tighter upper bound on the adversarial loss. This implies that \gl{} improves the robustness of mixup. 

Fig.~\ref{Fig:compare_mixup_loss_GenLabel_loss} compares the second-order Taylor approximation of mixup loss  $\tilde{L}_n^{\op{mix}}(\boldsymbol{\theta},S)$, \gl{} loss $\tilde{L}_n^{\op{gen}}(\boldsymbol{\theta},S)$, and adversarial loss 
$\frac{1}{n}\sum_{i=1}^n \tilde{\ell}_{\op{adv}}(\delta_{\op{gen}},(\vx_i,y_i))$,
for the logistic regression model $\boldsymbol{\theta} = (10 \cos \phi, 10 \sin \phi)$ parameterized by the angle $\phi$. Here, we use the dataset $S = \{ (\vx_i^{+}, +1), (\vx_i^{-}, -1) \}_{i=1}^{20}$,
where each sample at class $+1$ and $-1$ follows the distribution of $\vx_i^{+} \sim \mathcal{N}([+1, 0], \frac{1}{100}\mI_2 )$ and $\vx_i^{-} \sim \mathcal{N}([-1, 0], \frac{1}{100}\mI_2)$, respectively.
One can confirm that the model $\boldsymbol{\theta} = (10, 0)$, which corresponds to $\phi=0$, has the smallest mixup/\gl{}/adversarial loss. In every angle $\phi \in [-\frac{\pi}{4}, \frac{\pi}{4}]$, the \gl{} loss is strictly smaller than mixup loss, which coincides with the result of Theorem~\ref{Theorem:robustness_logreg}.
We can also extend the result of Theorem~\ref{Theorem:robustness_logreg} to fully-connected ReLU networks as below.

\begin{theorem}
\label{Theorem:robustness_relu}
Consider fully-connected ReLU network $f_{\boldsymbol{\theta}}(\vx)=\boldsymbol{\beta}^T \sigma(\mW_{N-1}\cdots (\mW_2 \sigma(\mW_1 \vx)))$ where $\sigma$ is the activation function and the parameters contain matrices $\mW_i$ and a vector $\boldsymbol{\beta}$. Suppose there exists a constant $c_x>0$ such that $\Vert \vx_i \Vert_2 \geq c_x $ for all $i\in \{1,2,\cdots,n\}$. Then, in the asymptotic regime of large $\sigma_1$, for any $\boldsymbol{\theta} \in \Theta$, we have
\begin{equation*}
\tilde{L}_n^{\op{mix}}(\boldsymbol{\theta},S) > \tilde{L}_n^{\op{gen}}(\boldsymbol{\theta},S) \geq 
\frac{1}{n}\sum_{i=1}^n \tilde{\ell}_{\op{adv}}(\delta_{\op{gen}},(\vx_i,y_i)).
\end{equation*}
Here, 
$\tilde{\ell}_{\op{adv}}(\delta,(\vx, y))=\ell(\boldsymbol{\theta},(\vx, y))+\delta\left|g \left(f_{\boldsymbol{\theta}}(\vx)\right)-y\right|\left\|\nabla f_{\boldsymbol{\theta}}(\vx)\right\|_{2}+\frac{\delta^{2} d}{2}\left|h^{\prime \prime}\left(f_{\boldsymbol{\theta}}(\vx)\right)\right|\left\|\nabla f_{\boldsymbol{\theta}}(\vx)\right\|_{2}^{2}$
is the Taylor approximation of adversarial loss for ReLU network, and we have
$\delta_{\op{gen}} = Rc_x A_{\sigma_1,c,\tau,d}^i$ and $R=\min_{i\in \{1,\cdots,n\}}|\cos (\nabla f_{\boldsymbol{\theta}} (\vx_i), \vx_i)|$, where $A^i_{\sigma_1,c,\tau,d}$ is in (\ref{A,B}) in the Appendix and $g(x) = e^x / (1+e^x)$.
\end{theorem}

\vspace{-2mm}
\section{Experimental results}\label{sec:exp}
\vspace{-2mm}

Now we investigate the effect of \gl{} on various real datasets. 
To be specific, we provide empirical results showing that \gl{} improves the generalization performance and adversarial robustness of mixup.
Among the datasets in OpenML~
\citep{OpenML2013}, we first choose 160 low-dimensional datasets 
having no more than 20 features and less than 5000 data points. Among 160 datasets, we finally choose 109 datasets which fit well on the suggested generative model (either Gaussian mixture or kernel density estimator); we used a dataset if the generative model has more than 95\% of train accuracy.
Recall that we allow the combination of the mixup labeling $\vy^{\op{mix}}$ and the suggested labeling $\vy^{\op{gen}}$, i.e., re-label the mixed point by $\gamma \vy^{\op{gen}} + (1-\gamma) \vy^{\op{mix}}$ for $\gamma \in [0,1]$. Here we choose the optimal mixing ratio $\gamma$ using cross-validations. 
For measuring the adversarial robustness, we test under decision-based black-box attack~\citep{brendel2017decision}.
We consider two types of network models: logistic regression and fully-connected (FC) ReLU networks with 2 hidden layers. All the results in the main manuscript are for logistic regression model, while we have similar pattern for the FC network with 2 hidden layers, the results of which are provided in \sm.
All algorithms are implemented in PyTorch~\citep{pytorch}, and the experimental details including network architectures, cross-validation setting, hyperparameters, and attack radius are summarized in Section~\ref{sec:exp_setup} in \sm.

\paragraph{Suggested schemes}

For the datasets in OpenML, we tested \gl{} on mixup. We considered two types of generative models: Gaussian mixture (GM) and kernel density estimator (KDE). We denote each scheme by mixup+\gl{} (GM) and mixup+\gl{} (KDE), respectively. We also considered choosing which generative model to use, based on the cross-validation (CV): this scheme is denoted by mixup+\gl{} (CV).

We considered two domains for applying \gl{}: one is applying it on the input feature space following Algorithm~\ref{Algo:GenLabel_input}, and the other is applying it on the hidden feature space as in Algorithm~\ref{Algo:GenLabel_latent}.
All the experimental results in the main manuscript are obtained when we apply \gl{} on the input feature space, and we use the logistic regression model. The results in Appendix use \gl{} on the hidden feature space at the penultimate layer of FC ReLU networks with 2 hidden layers.
For image datasets (MNIST, CIFAR-10, CIFAR-100 and TinyImageNet), we tested \gl{} on mixup and manifold-mixup, the results of which are given in the Appendix.

\paragraph{Compared schemes}

We first compared our scheme with vanilla training, mixup, and adamixup~\citep{guo2019mixup}. For OpenML datasets, we added the comparison with other schemes that are closely related with \gl{}. First, we tested the performance of generative classifier using Gaussian mixture (GM) as the generative model. This is denoted by generative classifier (GM). Note that mixup+\gl{} is making use of both data augmentation (mixup) and generative models, while the generative classifier is making use of the generative models only. Comparing mixup+\gl{} (GM) with generative classifier (GM) shows whether combining mixup and generative model is beneficial than when we rely solely on the generative model.
Second, we tested another method we come up with, dubbed as \emph{excluding MI}, which is excluding the mixup points suffering from manifold intrusion (MI). Here, we call a mixed point $\vx^{\op{mix}} = \lambda \vx_1 + (1-\lambda) \vx_2$ is suffering from manifold intrusion if the label of the nearest neighbor $\vx^{\op{nn}} = \argmin_{\vx \in X} d(\vx, \vx^{\op{mix}})$ is different from $\vy_1$ and $\vy_2$.
For image datasets, we tested the performances of mixup and manifold-mixup and compared them with the performances of mixup+\gl{} and ``manifold-mixup''+\gl{}.

\begin{figure}[t]
    \vspace{-2mm}
	\centering
	\subfloat[][\centering{ $\alpha=1.0$}]{\includegraphics[height=30mm ]{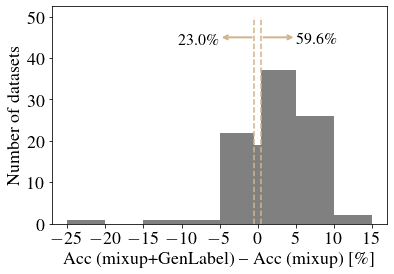}\label{Fig:openml_histogram_cln_alpha1}} 
	\quad \quad \quad
 	\subfloat[][\centering{$\alpha=2.0$
 	} ]{\includegraphics[height=30mm ]{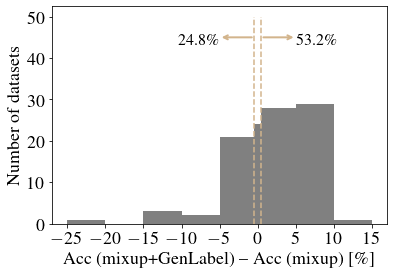}\label{Fig:openml_histogram_cln_alpha2}}
	\vspace{-2mm}
	\caption
	{
	Histogram of 
	the generalization performance increase (\%) of mixup+\gl{} (CV) compared with mixup for 109 OpenML datasets, for a given $\alpha$ used for sampling the mixing coefficient $\lambda \sim \op{Beta}(\alpha, \alpha)$ of mixup.
	When $\alpha=1.0$, mixup+\gl{} outperforms mixup by more than 0.5\% accuracy for 59.6\% of tested datasets, while mixup outperforms mixup+\gl{} by more than 0.5\% accuracy for 23.0\% of tested datasets.
    For both $\alpha$ values, a large portion of tested datasets enjoy the accuracy improvement by combining \gl{} with mixup. 
	}
	\label{Fig:openml_histogram_cln}
	\vspace{-3mm}
\end{figure}

\subsection{Results on generalization performance}

We compare the clean accuracy (generalization performance) of mixup+\gl{} with various baselines. First, we show the statistics of comparison for 109 tested OpenML datasets. Afterwards, we show the performance for some selected OpenML datasets when mixup+\gl{} performs well.
Note that the experimental results in this section is obtained for logistic regression models, and we added similar results for FC ReLU networks with 2 hidden layers in Section~\ref{sec:GenLabel_FC_ReLU} in the Appendix.

\paragraph{Statistics of mixup+\gl{} vs baselines} 

\begin{wraptable}{r}{8.0cm}
	\caption{Statistics of comparing the accuracy of \gl{} (GM) and generative classifier (GM), for 109 OpenML datasets. 
		\gl{} (GM) has a higher accuracy than generative classifier (GM) for 62.4\% of tested datasets, and the accuracy of \gl{} (GM) is lower than that of generative classifier (GM) for 33.9\% of tested datasets.	%
	}%
	\scriptsize
	\centering
	\setlength{\tabcolsep}{3pt} %
	\renewcommand{\arraystretch}{1} %
	\vspace{-2mm}
	\begin{tabular}{l|c}%
		\midrule  
		\textbf{Mixup+GenLabel (GM) }versus  
		& \textbf{Generative Classifier (GM)}
		\\
		\midrule		 
		\textbf{Higher ($> 0\%$) }
		& \textbf{62.4}\%
		\\
		\textbf{On-par ($= 0\%$)}
		& 3.7\%
		\\    		
		\textbf{Lower ($< 0\%$) }
		& 33.9\%
		\\    
		\bottomrule      
	\end{tabular}
	\label{Table:OpenML_stat_gm}
	\vspace{-5mm}	
\end{wraptable} 
Fig.~\ref{Fig:openml_histogram_cln} compares the generalization performances of mixup+\gl{} (CV) and mixup on 109 OpenML datasets. The x-axis represents 
$(A_{\op{GenLabel}}^{\op{cln}} - A_{\op{mixup}}^{\op{cln}})$, the increase of clean accuracy $A^{\op{cln}}$ with the aid of \gl{}, and the y-axis represents the number of datasets having the accuracy gain at the x-axis.
Recall that the mixing coefficient is sampled as $\lambda \sim \op{Beta}(\alpha, \alpha)$ for some $\alpha > 0$. We plotted the histogram for two popular settings, $\alpha=1$ and $\alpha= 2$.
It turns out that for both settings, the clean accuracy of \gl{} is greater than
that of mixup for more than 50$\%$ of the tested datasets.

\begin{table}[t]
\centering
	\caption{
	The statistics of the comparison of 	generalization performances of \gl{} and baselines, for 109 OpenML datasets. Each column represents the statistics of the accuracy of each scheme, compared with \gl{}.
	For example, \gl{} has more than 0.5\% higher accuracy than mixup for 59.6\% of tested datasets, \gl{} performs similar with mixup for 17.4\% of tested datasets, and the accuracy of \gl{} is more than 0.5\% lower than that of mixup for 23.0\% of tested datasets.
	This shows that for each baseline, 
	the number of datasets where \gl{} outperforms the baseline is larger than the number of datasets where \gl{} is worse than the baseline.
	Here, we used $\alpha=1$ for generating mixup samples.
	}
	\scriptsize
	\setlength{\tabcolsep}{3pt} %
    \renewcommand{\arraystretch}{1} %
    \begin{tabular}{l|c|c|c|c|c}%
        \midrule  
         \textbf{Mixup+GenLabel (CV)} versus  
         & \textbf{Vanilla}
        & \textbf{Adamixup}
        & \textbf{Mixup}
        & \textbf{Mixup + exclude MI}
        & \textbf{Generative Classifier (GM)}
         \\
		\midrule		 
        \textbf{Higher ($> 0.5\%)$ }
        & \textbf{37.6}\%
        & \textbf{46.8}\%
        & \textbf{59.6}\%

        & \textbf{56.9}\%
        & \textbf{44.0}\%
         \\
        \textbf{On-par (within $0.5\%$)}
        & 31.2\%
        & 25.7\%
        & 17.4\%
        & 16.5\%
        & 23.9\%
         \\    		
        \textbf{Lower ($< -0.5\%$) }
        & 31.2\%
        & 27.5\%
        & 23.0\%
        & 26.6\%
        & 32.1\%
         \\    
		\bottomrule      
    \end{tabular}
    \quad \quad 
\label{Table:OpenML_stat}
\end{table}

Table~\ref{Table:OpenML_stat} compares 
generalization performances of \gl{} and baselines, for 109 OpenML datasets. Each column corresponds to each baseline, and each cell in the table represents the number of datasets satisfying the condition.  
For example, \gl{} has more than 0.5\% accuracy than adamixup for 46.8\% of tested datasets, while adamixup has more than 0.5\% accuracy than \gl{} for 27.5\% of tested datasets. 
One can confirm that for each tested baseline (vanilla training, adamixup, mixup, and mixup + exclude MI, generative classifier), the number of datasets where mixup+\gl{} (CV) outperforms the baseline is larger than the number of datasets where the baseline outperforms mixup+\gl{} (CV). 
Table~\ref{Table:OpenML_stat_gm} compares the performances of mixup+\gl{} (GM) and generative classifier (GM). It turns out that using the suggested \gl{} has gain for more than 62\% of the tested datasets, showing that it is worth making use of both mixup and generative models at the same time.

\paragraph{Performance of \gl{} on selected OpenML datasets} 

In Table~\ref{Table:OpenML}, we take a closer look at some selected datasets when \gl{} performs better than mixup.
For example, mixup+\gl{} enjoys over $7 \%$ accuracy gain compared with mixup, for dataset IDs 61, 830, and 938.
One can confirm that mixup+\gl{} performs the best in 13 out of 16 selected datasets. Note that the best generative model (GM or KDE) for mixup+\gl{} varies depending on the dataset, which coincides with our intuition. Moreover, mixup+\gl{} (CV), which chooses the generative model based on the cross-validation successfully finds the appropriate generative model and outperforms other baselines in a large number of datasets in Table~\ref{Table:OpenML}. 

Interestingly, for datasets having ID number 61, 721, 817, 830, 869, 885, 907, 915, 925, 938 and 40981, mixup has a worse performance than vanilla training, but mixup combined with \gl{} overcomes the limitation of mixup and achieves accuracy higher than vanilla training. For datasets having ID number 36 and 778, although mixup+\gl{} has worse accuracy than vanilla training, it is still true that mixup combined with \gl{} far outperforms mixup.
This shows that \gl{} fixes the issue of mixup and guides towards more accurate classification, for the selected datasets. 
Even in other datasets (IDs 855 and 1006) where mixup has a better performance than vanilla training, mixup+\gl{} with an appropriate choice of generative model outperforms both mixup and vanilla training.
In Table~\ref{Table:OpenML}, we also compared mixup+\gl{} with adamixup, a method for avoiding manifold intrusion. It turns out that for 13 out of 16 selected datasets, mixup+\gl{} outperforms adamixup.

\begin{table}[t]
\centering
	\caption{Generalization performances (clean accuracy in \%) on selected OpenML datasets. The performance comparisons for all tested datasets are in Table~\ref{Table:OpenML_stat}.
	}
	\scriptsize
	\setlength{\tabcolsep}{3pt} %
    \renewcommand{\arraystretch}{1} %
	\begin{tabular}{l|c|c|c|c|c|c|c|c}
        \midrule  
         Methods $\backslash$ OpenML Dataset ID 
        & 36
        & 61
        & 721
        & 778
        & 817
        & 830
        & 855
        & 869 
         \\
		\midrule		 
        \textbf{Vanilla}
        & \textbf{93.82}$\pm$0.06
        & 95.56$\pm$0.00
        & 79.67$\pm$0.67
        & \textbf{98.42}$\pm$0.53
        & 61.33$\pm$2.67
        & 77.60$\pm$0.53
        & 63.33$\pm$4.50
        & 74.40$\pm$2.05
         \\
        \textbf{AdaMixup}
        & 89.75$\pm$0.29
        & 92.44$\pm$1.09
        & 80.33$\pm$0.67
        & 96.32$\pm$0.53
        & 60.00$\pm$0.00
        & 78.40$\pm$1.55
        & 66.67$\pm$0.42
        & 69.60$\pm$3.54
         \\    		
        \textbf{Mixup }
        & 88.98$\pm$0.78
        & 88.00$\pm$1.09
        & 79.33$\pm$0.82
        & 95.00$\pm$1.53
        & 60.00$\pm$0.00
        & 76.27$\pm$1.31
        & 66.00$\pm$1.74
        & 71.73$\pm$3.79
         \\    
         \grayl{\textbf{Mixup  + GenLabel (GM)}}
        & \grayc{92.21$\pm$0.58}
        & \grayc{96.00$\pm$1.67}
        & \grayc{\textbf{81.00}$\pm$1.33}
        & \grayc{97.11$\pm$0.98}
        & \grayc{64.00$\pm$5.33}
        & \grayc{\textbf{86.13}$\pm$1.36}
        & \grayc{66.40$\pm$2.88}
        & \grayc{\textbf{76.27}$\pm$3.17}
         \\  
        \grayl{\textbf{Mixup  + GenLabel (KDE)}}
        & \grayc{92.64$\pm$0.26}
        & \grayc{96.00$\pm$0.89}
        & \grayc{79.67$\pm$1.25}
        & \grayc{96.05$\pm$0.83}
        & \grayc{\textbf{66.67}$\pm$5.96}
        & \grayc{77.33$\pm$4.84}
        & \grayc{\textbf{67.60}$\pm$0.90}
        & \grayc{74.53$\pm$3.99}
         \\ 
        \grayl{\textbf{Mixup  + GenLabel (CV)}}
        & \grayc{92.55$\pm$0.15}
        & \grayc{\textbf{96.44}$\pm$1.09}
        & \grayc{80.33$\pm$1.63}
        & \grayc{96.05$\pm$1.86}
        & \grayc{65.33$\pm$4.99}
        & \grayc{84.53$\pm$1.81}
        & \grayc{67.33$\pm$2.76}
        & \grayc{73.87$\pm$2.13}
         \\
		\bottomrule      
    \end{tabular} \begin{tabular}{l|c|c|c|c|c|c|c|c}
    \midrule  
    Methods $\backslash$ OpenML Dataset ID 
    & 885
    & 907
    & 915
    & 925
    & 938
    & 1006
    & 40710
    & 40981
    \\
    \midrule		 
   \textbf{Vanilla}
    & 96.50$\pm$1.22
   & 45.67$\pm$3.43
   & 48.00$\pm$0.84
   & 93.81$\pm$0.00
   & 93.85$\pm$7.54
   & 78.67$\pm$2.67
   & 68.13$\pm$0.00
   & 74.98$\pm$0.19
   \\
   \textbf{AdaMixup}
    & \textbf{97.50}$\pm$0.00
   & 46.50$\pm$1.11
   & 48.00$\pm$0.52
   & \textbf{94.43}$\pm$1.05
   & \textbf{98.46}$\pm$3.08
   & 78.67$\pm$2.67
   & 67.91$\pm$0.44
   & 73.14$\pm$3.48
   \\    		
   \textbf{Mixup }
    & 94.50$\pm$1.00   
   & 44.67$\pm$3.14
   & 46.11$\pm$3.08
   & 92.99$\pm$1.37
   & 90.77$\pm$5.76
   & 80.00$\pm$0.00
   & 68.13$\pm$ 0.70
   & 74.78 $\pm$0.56
   \\    
 \grayl{\textbf{Mixup  + GenLabel (GM)}}
        & \grayc{97.00$\pm$1.00}
  & \grayc{47.83$\pm$3.56}
  & \grayc{46.74$\pm$7.37}
  & \grayc{93.61$\pm$0.77}
  & \grayc{\textbf{98.46}$\pm$3.08}
  & \grayc{81.33$\pm$1.09}
  & \grayc{69.67$\pm$0.54}
  & \grayc{75.07$\pm$1.08}
  \\  
  \grayl{\textbf{Mixup  + GenLabel (KDE)}}
        & \grayc{96.50$\pm$1.22}
  & \grayc{45.67$\pm$4.39}
  & \grayc{46.11$\pm$7.57}
      & \grayc{93.61$\pm$0.77}
  & \grayc{\textbf{98.46}$\pm$3.08}
  & \grayc{80.00$\pm$3.44}
  & \grayc{69.45$\pm$0.44}
  & \grayc{\textbf{76.43}$\pm$0.36}
  \\ 
  \grayl{\textbf{Mixup  + GenLabel (CV)}}
        & \grayc{97.00$\pm$1.00}
  & \grayc{\textbf{48.83}$\pm$4.46}
  & \grayc{\textbf{48.42}$\pm$3.82}
      & \grayc{94.23$\pm$0.82}
  & \grayc{95.38$\pm$6.15}
  & \grayc{\textbf{81.78}$\pm$0.89}
  & \grayc{\textbf{70.11}$\pm$0.82}
  & \grayc{76.23$\pm$0.71}
  \\         
		\bottomrule      
    \end{tabular}    
	\label{Table:OpenML}
\end{table}

\begin{table}[t]
\centering
	\caption{Ablation study on \gl{}: comparing clean accuracy (\%) on selected OpenML datasets. The performance comparisons for all tested datasets are in Table~\ref{Table:OpenML_stat}.
	}
	\scriptsize
	\setlength{\tabcolsep}{3pt} %
    \renewcommand{\arraystretch}{1} %
    	\begin{tabular}{l|c|c|c|c|c|c|c|c}
        \midrule  
         Methods $\backslash$ OpenML Dataset ID 
         & 36
        & 61
        & 721
        & 778
        & 817
        & 830
        & 855
        & 869
         \\
		\midrule
		 \textbf{Generative classifier (GM)}
		 & 89.75$\pm$0.00
        & 95.56$\pm$0.00
        & 78.33$\pm$0.00
        & 89.47$\pm$0.00
        & 60.00$\pm$0.00
        & 78.67$\pm$0.00
        & 65.33$\pm$0.00
        & 73.33$\pm$0.00
        \\                              
        \textbf{Mixup }
        & 88.98$\pm$0.78
        & 88.00$\pm$1.09
        & 79.33$\pm$0.82
        & 95.00$\pm$1.53
        & 60.00$\pm$0.00
        & 76.27$\pm$1.31
        & 66.00$\pm$1.74
        & 71.73$\pm$3.79
         \\    
    \textbf{Mixup + Excluding MI}
        & 89.21$\pm$0.52
        & 93.33$\pm$0.00
        & 79.67$\pm$0.67
        & 95.00$\pm$1.53
        & 61.33$\pm$2.67
        & 78.13$\pm$1.36
        & 66.40$\pm$1.37
        & 72.00$\pm$3.55
        \\                 
        \grayl{\textbf{Mixup  + GenLabel (GM)}}
        & \grayc{92.21$\pm$0.58}
        & \grayc{96.00$\pm$1.67}
        & \grayc{\textbf{81.00}$\pm$1.33}
        & \grayc{\textbf{97.11}$\pm$0.98}
        & \grayc{64.00$\pm$5.33}
        & \grayc{\textbf{86.13}$\pm$1.36}
        & \grayc{66.40$\pm$2.88}
        & \grayc{\textbf{76.27}$\pm$3.17}
         \\  
        \grayl{\textbf{Mixup  + GenLabel (KDE)}}
        & \grayc{\textbf{92.64}$\pm$0.26}
        & \grayc{96.00$\pm$0.89}
        & \grayc{79.67$\pm$1.25}
        & \grayc{96.05$\pm$0.83}
        & \grayc{\textbf{66.67}$\pm$5.96}
        & \grayc{77.33$\pm$4.84}
        & \grayc{\textbf{67.60}$\pm$0.90}
        & \grayc{74.53$\pm$3.99}
         \\ 
        \grayl{\textbf{Mixup  + GenLabel (CV)}}
        & \grayc{92.55$\pm$0.15}
        & \grayc{\textbf{96.44}$\pm$1.09}
        & \grayc{80.33$\pm$1.63}
        & \grayc{96.05$\pm$1.86}
        & \grayc{65.33$\pm$4.99}
        & \grayc{84.53$\pm$1.81}
        & \grayc{67.33$\pm$2.76}
        & \grayc{73.87$\pm$2.13}
         \\         
		\bottomrule      
    \end{tabular} \begin{tabular}{l|c|c|c|c|c|c|c|c}
    \midrule  
    Methods $\backslash$ OpenML Dataset ID 
    & 885
    & 907
    & 915
        & 925
    & 938
    & 1006
    & 40710
    & 40981
    \\
    \midrule	
    \textbf{Generative classifier (GM)}
        & 95.00$\pm$0.00
        & 47.50$\pm$0.00
        & 42.11$\pm$0.00
        & 90.72$\pm$0.00
        & 92.31$\pm$0.00
        & 77.78$\pm$0.00
        & 69.23$\pm$0.00
        & 73.91$\pm$0.00
        \\       
  \textbf{Mixup }
    & 94.50$\pm$1.00  
  & 44.67$\pm$3.14
  & 46.11$\pm$3.08
  & 92.99$\pm$1.37
  & 90.77$\pm$5.76
  & 80.00$\pm$0.00
  & 68.13$\pm$0.70
  & 74.78$\pm$0.56
  \\    
    \textbf{Mixup + Excluding MI}
        & 94.50$\pm$1.00
        & 44.00$\pm$4.06
        & 47.79$\pm$3.37
        & 93.20$\pm$1.40
        & 89.23$\pm$6.15
        & 77.33$\pm$4.31
        & 68.57$\pm$1.12
        & 74.69$\pm$0.84
        \\        
  \grayl{\textbf{Mixup  + GenLabel (GM)}}
        & \grayc{\textbf{97.00}$\pm$1.00}
  & \grayc{47.83$\pm$3.56}
  & \grayc{46.74$\pm$7.37}
  & \grayc{93.61$\pm$0.77}
  & \grayc{\textbf{98.46}$\pm$3.08}
  & \grayc{81.33$\pm$1.09}
  & \grayc{69.67$\pm$0.54}
  & \grayc{75.07$\pm$1.08}
  \\  
  \grayl{\textbf{Mixup  + GenLabel (KDE)}}
        & \grayc{96.50$\pm$1.22}
  & \grayc{45.67$\pm$4.39}
  & \grayc{46.11$\pm$7.57}
      & \grayc{93.61$\pm$0.77}
  & \grayc{\textbf{98.46}$\pm$3.08}
  & \grayc{80.00$\pm$3.44}
  & \grayc{69.45$\pm$0.44}
  & \grayc{\textbf{76.43}$\pm$0.36}
  \\ 
  \grayl{\textbf{Mixup  + GenLabel (CV)}}
        & \grayc{\textbf{97.00}$\pm$1.00}
  & \grayc{\textbf{48.83}$\pm$4.46}
  & \grayc{\textbf{48.42}$\pm$3.82}
      & \grayc{\textbf{94.23}$\pm$0.82}
  & \grayc{95.38$\pm$6.15}
  & \grayc{\textbf{81.78}$\pm$0.89}
  & \grayc{\textbf{70.11}$\pm$0.82}
  & \grayc{76.23$\pm$0.71}
  \\         
	\bottomrule      
    \end{tabular}    
	\label{Table:OpenML_ablation}
\end{table}

Table~\ref{Table:OpenML_ablation} shows the ablation study results on mixup+\gl{}. First, instead of using both mixup and generative models as in mixup+\gl{}, one can consider using generative model only, which is generative classifier. For datasets in Table~\ref{Table:OpenML_ablation}, mixup + \gl{} (GM) outperforms generative classifier (GM), showing that making use of both mixup and generative models is beneficial for improving the generalization performances.
Second, we compare \gl{} with \emph{excluding MI}, an alternative to handle the manifold intrusion (MI) issue. Instead of re-labeling mixup points suffering from MI issue, this alternative excludes the mixup points having MI issue and train the model using the remaining mixup point. It turns out that \gl{} outperforms \emph{excluding MI} for all datasets in Table~\ref{Table:OpenML_ablation}.

Recall that the comparisons in this paragraph is for 16 selected datasets, and the comparisons for all 109 tested datasets are given in Table~\ref{Table:OpenML_stat}.

\begin{table}[t]
\centering
	\caption{
	Robust accuracy (\%) under FGSM attack, on OpenML datasets when \gl{} performs the best
	}
	\scriptsize
	\setlength{\tabcolsep}{3pt} %
    \renewcommand{\arraystretch}{1} %
	\begin{tabular}{l|c|c|c|c|c|c|c|c}
        \midrule  
         Methods $\backslash$ OpenML Dataset ID 
         & 3
         & 223
         & 312
         & 313
         & 346
         & 463
         & 753
         & 834
         \\
		\midrule		 
        \textbf{Vanilla}
        & 45.61$\pm$14.11
        & 11.12$\pm$4.33
        & 21.78$\pm$9.27
        & 12.23$\pm$2.34
        & 42.92$\pm$12.29
        & 71.60$\pm$13.57
        & 28.35$\pm$6.60
        & 14.78$\pm$3.73
         \\
        \textbf{Mixup }
        & 43.36$\pm$14.10
        & 11.16$\pm$5.22
        & 23.40$\pm$9.82
        & 14.22$\pm$3.73
        & 40.83$\pm$11.90
        & 68.80$\pm$12.04
        & 29.42$\pm$6.03
        & 15.20$\pm$3.91
         \\ 
        \grayl{\textbf{Mixup  + GenLabel (GM)}} 
        & \grayc{\textbf{51.87}$\pm$5.45}
        & \grayc{\textbf{13.99}$\pm$6.65}
        & \grayc{\textbf{36.61}$\pm$14.07}
        & \grayc{\textbf{18.39}$\pm$2.57} 
        & \grayc{\textbf{52.92}$\pm$12.79}
        & \grayc{\textbf{84.46}$\pm$2.62}
        & \grayc{\textbf{38.23}$\pm$6.29}
        & \grayc{\textbf{21.94}$\pm$3.69}
         \\  
		\bottomrule      
    \end{tabular} \begin{tabular}{l|c|c|c|c|c|c|c|c}
    \midrule  
    Methods $\backslash$ OpenML Dataset ID 
    & 952
    & 954
    & 978
    & 987
    & 988 
    & 1022 
    & 1045
    & 1059
    \\
    \midrule		 
    \textbf{Vanilla}
        & 30.85$\pm$3.12
        & 71.40$\pm$4.86
        & 28.89$\pm$12.62
        & 68.04$\pm$18.17
        & 50.35$\pm$12.41
        & 35.85$\pm$12.20
        & 57.89$\pm$10.42
        & 68.05$\pm$21.90
          \\
        \textbf{Mixup }
        & 31.38$\pm$3.12
        & 69.32$\pm$5.02
        & 34.89$\pm$11.32
        & 67.82$\pm$14.56
        & 53.54$\pm$10.91
        & 41.25$\pm$12.11
        & 59.88$\pm$14.38
        & 67.18$\pm$25.48
         \\    
        \grayl{\textbf{Mixup  + GenLabel (GM)}} 
        & \grayc{\textbf{42.19}$\pm$7.61}
        & \grayc{\textbf{85.16}$\pm$7.55}
        & \grayc{\textbf{43.74}$\pm$15.93} 
        & \grayc{\textbf{83.21}$\pm$1.77}
        & \grayc{\textbf{63.76}$\pm$9.89}
        & \grayc{\textbf{56.61}$\pm$15.00}
        & \grayc{\textbf{66.22}$\pm$16.60}
        &\grayc{ \textbf{74.64}$\pm$20.14}
        \\
		\bottomrule      
    \end{tabular}

	\label{Table:OpenML_rob}
\end{table}

\subsection{Results on adversarial robustness}

We now check the adversarial robustness of \gl{} on OpenML datasets, under the FGSM attack~\citep{goodfellow2014explaining}.
Table~\ref{Table:OpenML_rob} shows the robust accuracy of the logistic regression model, for the selected 16 OpenML datasets where mixup+\gl{} far outperforms mixup and vanilla training. One can confirm that \gl{} improves the robust accuracy of mixup by $5\sim 15 \%$. Similar to the result in Table~\ref{Table:OpenML}, mixup+\gl{} enjoys a huge gap with vanilla training, even when mixup performs worse than the vanilla training. We added similar results for FC ReLU networks with 2 hidden layers in Section~\ref{sec:GenLabel_FC_ReLU} in the Appendix.
These results coincide with the theoretical results in Section~\ref{sec:GenLabel_improves_rob}, showing that \gl{} improves the robustness of mixup in logistic regression and FC ReLU networks.

\subsection{Extension to high-dimensional image datasets}

So far we have discussed the performance of \gl{} on low-dimensional datasets in OpenML. We have also tested the generalization and adversarial robustness performances of \gl{} on high-dimensional image datasets including MNIST, CIFAR-10, CIFAR-100 and TinyImageNet-200. 
It turns out that \gl{} has a marginal gain in those datasets, for both mixup and manifold-mixup. 
The details of the result are provided in Section~\ref{sec:high_dim_image_result} in \sm.

\section{Discussions}

Here we provide additional discussion topics for the suggested \gl{}. First, we suggest ideas on how to extend \gl{} to the scenario of using generative models having implicit/approximate density. Second, we propose a method of using generative models not only for labeling mixed points, but also for mixing data points. 

\subsection{Extension to generative models with implicit/approximate density}\label{Sec:generative_model_general}

In this paper, we used \gl{} for generative models learning the explicit density, but
our method can be also applied to a broad range of generative models having implicit or approximate density. When the available generative model only provides approximated density $\tilde{p}_c(\vx)$, as in VAEs~\citep{kingma2019introduction}
, we can replace $p_c(\vx)$ by $\tilde{p}_c(\vx)$ in Algorithm~\ref{Algo:GenLabel_input} and apply our \gl{} scheme. For GANs~\citep{xia2021gan} 
which only provide implicit density, we use a proxy to the density $p_c(\vx)$ by inverting the generator~\citep{creswell2018inverting}.
To be specific,
let $\mathcal{M}_c = \{G(\vw , c) : \vw \in \mathbb{R}^d \}$ be the data manifold generated by generator $G$ for class $c$. 
Assuming the spherical Gaussian noise model used for manifold learning~\citep{hastie1989principal,chang2001unified}, we can estimate $p_c(\vx) = \int p(\vx | G(\vw, c) ) p(G(\vw, c)) d\vw \simeq \frac{1}{n} \sum_{i=1}^n \exp(-d (\vx, G(\vw_i, c))$ by choosing $n$ random samples of $\vw$. 
Then, we can simply approximate this summation with the dominant term, which is expressed as $\max_i \exp(-d (\vx, G(\vw_i, c)) = \exp( - d(\vx, \mathcal{M}_c) )$.
Thus, we replace $p_c(\vx)$ by $\exp( - d(\vx, \mathcal{M}_c) )$ in Algorithm~\ref{Algo:GenLabel_input} and apply \gl{}.

\subsection{Using generative models for both mixing and labeling}\label{sec:disc_GenMix}
In \gl{}, mixed points $\vx^{\op{mix}}$ are obtained by existing mixing strategies, \eg mixup and manifold-mixup, and generative models are used only for re-labeling the mixed points. Now, the question is, can we also use generative models not only for labeling, but also for making better mixed points?
Here we suggest a new data augmentation scheme using generative models for both mixing and labeling. 
For a target class pair $c_1$ and $c_2$, we first choose a mixing coefficient $\lambda \in [0,1]$, e.g., using a Beta distribution. Then, we find $\vx^{\op{mix}}$ satisfying 
$\frac{p_{c_1}}{p_{c_1}+p_{c_2}} = \lambda$ and label it as $\vy^{\op{mix}} = \frac{p_{c_1}}{p_{c_1}+p_{c_2}} \ve_{c_1}  + \frac{p_{c_2}}{p_{c_1}+p_{c_2}} \ve_{c_2}$, where
$p_{c_1} = p_{c_1}(\vx^{\op{mix}})$ and
$p_{c_2} = p_{c_2}(\vx^{\op{mix}})$.
In Section~\ref{sec:GenMix} in
\sm, we provide our suggested algorithm for finding mixed point $\vx^{\op{mix}}$ using generative models (Gaussian mixture models and GANs) and the experimental results of this algorithm on synthetic/real datasets.

\section{Conclusion}
In this paper, we closely examined the failure scenarios of mixup for low dimensional data, and specify two main issues of mixup: (1) the manifold intrusion of mixup reduces both margin and accuracy, and (2) even when there is no manifold intrusion, the linear labeling method of mixup harms the margin and accuracy. 
Motivated by these observations, we proposed \gl{}, a novel way of labeling the mixup points by making use of generative models.
We visualized \gl{} for toy datasets and empirically/mathematically showed that 
\gl{} solves the main issues of mixup and achieve maximum margin in various low-dimensional datasets. 
We also mathematically showed that \gl{} improves the adversarial robustness of mixup in logistic regression model and fully-connected ReLU networks. Finally, we provide empirical results on the generalization/robustness performance of \gl{} on 109 low-dimensional datasets in OpenML, showing that \gl{} improves both robustness and generalization performance of mixup with a sufficiently large gain.

\section*{Acknowledgement}

This work was supported by an American Family Insurance grant via American Family Insurance Data Science Institute at University of Wisconsin-Madison.

\bibliography{iclr2022_conference}
\bibliographystyle{iclr2022_conference}

\newpage
\appendix

\section{Additional experimental results}

\subsection{GenLabel on high dimensional image datasets}\label{sec:high_dim_image_result}

In the main manuscript, we focused on the result for low-dimensional datasets. Here, we provide our experimental results on high dimensional image datasets, including MNIST, CIFAR-10, CIFAR-100, and TinyImageNet-200.

We test GenLabel combined with existing data augmentation schemes of mixup~\citep{zhang2017mixup} and manifold-mixup~\citep{pmlr-v97-verma19a}.
We compare our schemes with mixup and manifold-mixup. 
We also compare with AdaMixup, which avoids the manifold intrusion of mixup, similar to our work. 
For measuring the adversarial robustness, we test under AutoAttack~\citep{croce2020reliable},
which is developed to overcome gradient obfuscation~\citep{athalye2018obfuscated}, containing four white/black-box attack schemes (including auto-PGD) that does not need any specification of free parameters. 
The attack radius $\epsilon$ for each dataset is specified in Section~\ref{sec:exp_setup}.

\paragraph{GenLabel variant used for image datasets}
For image datasets, we learn generative models in the latent space. To be specific, we use a variant of GenLabel, which learns the generative model (Gaussian mixture model) and the discriminative model at the same time.
The pseudocode of this variant is given in Algorithm~\ref{Algo:GenLabel_consecutive}, and below we explain the details of this algorithm.

Consider a neural network $f_{\theta} = f_{\theta}^{\op{cls}} \circ f_{\theta}^{\op{feature}}$ parameterized by $\theta$, which is composed of the feature extractor part $f_{\theta}^{\op{feature}}$ and the classification part $f_{\theta}^{\op{cls}}$. We train a Gaussian mixture (GM) model on $f_{\theta}^{\op{feature}}(\vx)$, the hidden representation of input $\vx$.
In this algorithm, we consider updating the estimated GM model parameters (mean and covariance) at each batch training. At each iteration $t$, we randomly choose $B$ batch samples $\{(\vx_i, \vy_i)\}_{i=1}^B$ from the dataset $D$.
Then, we estimate the class-conditional mean and covariance of GM model in the hidden feature space. In other words, we compute the mean
$\bm{\mu}_{c}^{(t)} = \frac{1}{\lvert S_c \rvert } \sum_{i \in S_c} f_{\theta}^{\op{feature}}(\vx_i)$ and the covariance
$\bm{\Sigma}_{c}^{(t)} = \frac{1}{\lvert S_c \rvert} \sum_{i \in S_c} (f_{\theta}^{\op{feature}}(\vx_i) - \bm{\mu}_{c}^{(t)}) (f_{\theta}^{\op{feature}}(\vx_i) - \bm{\mu}_{c}^{(t)})^T$ for each class $c \in [k]$, where $S_c = \{i : \vy_i = \ve_c \}$ is the set of samples with label $c$ within the batch. For simplicity, we approximate the covariance matrix as a multiple of identity matrix, by setting 
$\bm{\Sigma}_c^{(t)} \leftarrow \frac{1}{d}\text{trace}(\bm{\Sigma}_c^{(t)}) \mI_d$. 
Here, we consider making use of the parameters (mean and covariance) estimated in the previous batches as well, by introducing a memory ratio factor $\beta \in [0,1]$. Formally, the rule for updating mean $\bm{\mu}_{c}^{(t)}$ and covariance $\bm{\Sigma}_{c}^{(t)}$ are represented as $\bm{\mu}_{c}^{(t)} \leftarrow (1-\beta) \bm{\mu}_{c}^{(t)} + \beta \bm{\mu}_{c}^{(t-1)} $ and $\bm{\Sigma}_{c}^{(t)} \leftarrow (1-\beta) \bm{\Sigma}_{c}^{(t)} + \beta \bm{\Sigma}_{c}^{(t-1)}$. 
When $\beta = 0$, it reduces to the memoryless estimation. To avoid cluttered notation, we discard $t$ unless necessary.

In the second stage, we apply conventional mixup-based data augmentation. 
We first permute the batch data $\{(\vx_i, \vy_i)\}_{i=1}^B$ and obtain $\{(\vx_{\pi(i)}, \vy_{\pi(i)})\}_{i=1}^B$. 
Afterwards, for each $i \in [B]$, we select the data pair, $(\vx_i, \vy_i)$ and $(\vx_{\pi(i)}, \vy_{\pi(i)})$, 
and apply a mixup-based data augmentation scheme denoted by $\text{mix}(\cdot)$, to generate mixed point $\vx_i^{\text{mix}}$ labeled by $\vy_i^{\text{mix}}$. 
One can use any data augmentation as $\text{mix}(\cdot)$, \eg mixup, manifold-mixup.
For example, for vanilla mixup, we have \begin{align}\label{eqn:mixing_data}
\vx_i^{\text{mix}} = \lambda \vx_i + (1-\lambda) \vx_{\pi(i)}, \quad \vy_i^{\text{mix}} = \lambda \vy_i + (1-\lambda) \vy_{\pi(i)}
\end{align}
where $\lambda \sim \text{Beta}(\alpha, \alpha)$ for some $\alpha > 0$. 

In the third stage, we re-label this augmented data based on the estimated GM model parameters. To be specific, we compute the likelihood of the mixed data sampled from class $c$, denoted by $p_c = \text{det}(\bm{\Sigma}_c)^{-1/2} \exp\{-( f_{\theta}^{\op{feature}}(\vx_{i}^{\text{mix}}) -  \bm{\mu}_{c} )^T \bm{\Sigma}_c^{-1} (f_{\theta}^{\op{feature}}(\vx_{i}^{\text{mix}}) -  \bm{\mu}_{c}) \}$.
Then, we sort $k$ classes in a descending order of $p_c$, and select the top-2 classes $c_1$ and $c_2$ satisfying $p_{c_1} \geq p_{c_2} \geq p_c$ for $c \in [k]\backslash \{c_1, c_2\}$. Then, we label the mixed point $\vx_{i}^{\text{mix}}$ as 

\begin{align}\label{eqn:GenLabel_appendix}
\vy_i^{\text{gen}} = \frac{p_{c_1}}{p_{c_1}+p_{c_2}} \ve_{c_1}  + \frac{p_{c_2    }}{p_{c_1}+p_{c_2}} \ve_{c_2}.
\end{align}
Since our generative model is an imperfect estimate on the data distribution, $\vy_i^{\text{gen}}$ may be incorrect for some samples. Thus, we can use a combination of vanilla labeling and the suggested labeling, \ie define the label of mixed point as $\gamma \vy_i^{\text{gen}} + (1-\gamma) \vy_i^{\text{mix}}$ for some $\gamma \in [0,1]$. Note that our scheme reduces to the vanilla labeling scheme when $\gamma = 0$. 
Using the augmented data with updated label, the algorithm trains the classification model $f_{\theta} : \mathbb{R}^n \rightarrow [0,1]^k$ that predicts the label $\vy = [y_1, \cdots, y_k]$ of the input data, using the cross-entropy loss $\ell_{\text{CE}} (\cdot)$.

\begin{algorithm}[t!]
	\small
	\textbf{Input} Data $D$, mix function $\text{mix}(\cdot)$, learning rate $\eta$, loss ratio $\gamma$, memory ratio $\beta$, batch size $B$, max iteration $T$
	\\
	\textbf{Output} Trained model $f_{\theta} = f_{\theta}^{\op{cls}} \circ f_{\theta}^{\op{feature}}$  
	\\
	\vspace{-4mm}
	\begin{algorithmic}%
	    \STATE $\theta \leftarrow$ Random initial model parameter,  \quad \quad \quad 
	    $\pi \leftarrow $ Permutation of $[B]$
		\STATE $(\bm{\mu}_{c}^{(0)}, \bm{\Sigma}_c^{(0)}) \leftarrow (\bm{0}, \bm{I}_d)$ for $c \in [k]$ 
        \FOR{iteration $t = 1, 2, \cdots, T$}
        \STATE $\{(\vx_i, \vy_i)\}_{i=1}^B \leftarrow$ Randomly chosen batch samples in $D$ 
        \FOR{class $c \in [k]$}
        \STATE $S_c \leftarrow \{i: \vy_i = \ve_c \}$ %
        \STATE $\bm{\mu}_c^{(t)} \leftarrow \frac{1}{\lvert S_c \rvert } \sum_{i \in S_c} f_{\theta}^{\op{feature}}(\vx_i)$,
        \quad \quad \quad 
        $\bm{\mu}_c^{(t)} \leftarrow  (1-\beta) \bm{\mu}_c^{(t)} + \beta \bm{\mu}_c^{(t-1)} $
        \STATE $\bm{\Sigma}_c^{(t)} \leftarrow \frac{1}{\lvert S_c \rvert} \sum_{i \in S_c} (f_{\theta}^{\op{feature}}(\vx_i) - \bm{\mu}_c^{(t)}) (f_{\theta}^{\op{feature}}(\vx_i) - \bm{\mu}_c^{(t)})^T$
        \STATE $\bm{\Sigma}_c^{(t)} \leftarrow \frac{1}{d}\text{trace}(\bm{\Sigma}_c^{(t)}) \mI_d$,
        \quad \quad \quad  
        $\bm{\Sigma}_c^{(t)} \leftarrow  (1-\beta) \bm{\Sigma}_c^{(t)} + \beta \bm{\Sigma}_c^{(t-1)}$
        \ENDFOR

        \FOR{sample index $i \in [B]$}
        \STATE $(\vx_{i}^{\text{mix}}, \vy_{i}^{\text{mix}}) \leftarrow \text{mix}( (\vx_i, \vy_i), (\vx_{\pi(i)}, \vy_{\pi(i)}) )$
        \STATE $p_c \leftarrow \text{det}(\bm{\Sigma}_c^{(t)})^{-1/2} \exp\{-( f_{\theta}^{\op{feature}}(\vx_{i}^{\text{mix}}) -  \bm{\mu}_{c}^{(t)} )^T (\bm{\Sigma}_c^{(t)})^{-1} (f_{\theta}^{\op{feature}}(\vx_{i}^{\text{mix}}) -  \bm{\mu}_{c}^{(t)}) \}$ for $c \in [k]$
        \STATE $c_1 \leftarrow \argmin_{c \in [k]} p_{c}$, \quad \quad \quad
        $c_2 \leftarrow \argmin_{c \in [k]\backslash \{c_1 \}} p_{c}$
        \STATE $\vy_{i}^{\text{gen}} \leftarrow \frac{  p_{c_1}   }{p_{c_1}+p_{c_2}} \ve_{c_1}+\frac{  p_{c_2}   }{p_{c_1}+p_{c_2}} \ve_{c_2}$
        \ENDFOR
        \STATE $\theta \leftarrow \theta - \eta \sum_{i \in [B]} \nabla_{\theta} \{ \gamma \cdot  \ell_{\text{CE}}(\vy_i^{\text{gen}}, f_{\theta}(\vx_i^{\text{mix}})) + (1-\gamma) \cdot  \ell_{\text{CE}}(\vy_i^{\text{mix}}, f_{\theta}(\vx_i^{\text{mix}}))\}$ 
        \ENDFOR
	\end{algorithmic}
	\caption{GenLabel (learning generative/discriminative models at the same time)}
	\label{Algo:GenLabel_consecutive}
\end{algorithm}

\paragraph{Results}
Table~\ref{Table:GenLabel_real_datasets} shows the summary of results. Here, we tried two different validation schemes: one is to choose the best robust model against AutoAttack, and the other is to select the model with the highest generalization performance in terms of clean accuracy. 
For each scheme $X \in \{ \text{mixup}, \text{manifold-mixup} \}$, it is shown that $``X+\text{GenLabel}''$ has a minor improvement on both generalization performance and adversarial robustness, for all image datasets.
It is also shown that the suggested mixup+GenLabel achieves a higher generalization performance than AdaMixup, which requires 3x higher computational complexity than our method.

\begin{table}[t]
	\centering
	\caption{Generalization and robustness performances on real image datasets. 
	We consider two types of validation: selecting 
	the best robust model based on AutoAttack, or the best model based on the clean accuracy. Here, Mixup+GenLabel indicates that we applied the suggested labeling method to the augmented points generated by mixup. 
	GenLabel helps mixup-based data augmentations in terms of both robust accuracy and clean accuracy. 
	}
	\scriptsize
	\label{Table:GenLabel_real_datasets}
	\begingroup
    \setlength{\tabcolsep}{3pt} %
    \renewcommand{\arraystretch}{1} %
    \begin{tabular}{l|c|c|c|c|c|c|c|c}
        \midrule  
        \multicolumn{1}{c|}{\multirow{2}{*}{Methods}}
         &  \multicolumn{2}{c|}{MNIST} & \multicolumn{2}{c|}{CIFAR-10} & \multicolumn{2}{c|}{CIFAR-100} & \multicolumn{2}{c}{TinyImageNet-200} \\
        \cmidrule{2-9}         &  Robust & Clean &          Robust & Clean & Robust & Clean & Robust & Clean\\
		\midrule		
        \textbf{Vanilla}                    
        & 48.17 $\pm$ 13.1      
        & 99.34 $\pm$ 0.03
        & 16.89 $\pm$ 0.98  
        & 94.57 $\pm$ 0.25 
        &  17.19 $\pm$  0.20
        &  74.48	$\pm$ 0.28  
        &  13.19 $\pm$ 0.19     
        &  58.13 $\pm$ 0.09
         \\
        \textbf{AdaMixup}
        & -  
        & 99.32 $\pm$ 0.05
        & - 
        & 95.45 $\pm$ 0.13
        & - 
        & - 
        & - 
        & - 
        \\
        \textbf{Mixup}                    
        & 55.44 $\pm$ 1.80                 
        & 99.27 $\pm$ 0.03
        &  11.65 $\pm$ 1.96      
        &  95.68 $\pm$ 0.06
        &  18.44 $\pm$ 0.45
        &  77.65 $\pm$ 0.30
        &  14.91 $\pm$ 0.48     
        &  59.46 $\pm$ 0.30    
        \\
        \grayl{\textbf{Mixup+GenLabel}}          
        & \grayl{56.54 $\pm$ 1.03}         
        & \grayl{99.36 $\pm$ 0.06}
        &  \grayl{14.32 $\pm$ 1.23}        
        &  \grayl{\textbf{96.09} $\pm$ 0.01}
        &  \grayl{\textbf{19.58} $\pm$ 0.71}
        &  \grayl{78.04 $\pm$ 0.21}
        &  \grayl{\textbf{15.34} $\pm$ 0.30}  
        &  \grayl{59.78 $\pm$ 0.09}        
        \\
        \textbf{Manifold mixup}           
        & 55.56 $\pm$ 1.53                 
        & 99.32 $\pm$ 0.04
        &  18.14 $\pm$ 1.88     
        &  94.78 $\pm$ 0.49 
        &  19.25 $\pm$ 0.61    
        &  78.61 $\pm$  0.17
        &  14.78 $\pm$ 0.28     
        &  59.87 $\pm$ 0.63        
        \\
        \grayl{\textbf{Manifold mixup+GenLabel}}  
        & \grayl{\textbf{56.62} $\pm$ 1.31}
        & \grayl{\textbf{99.37} $\pm$ 0.07}
        &  \grayl{\textbf{18.91} $\pm$ 1.26}  
        &  \grayl{95.10 $\pm$ 0.10}
        &  \grayl{19.28 $\pm$ 1.04}
        &  \grayl{\textbf{78.99} $\pm$ 0.54}
        &  \grayl{15.19 $\pm$ 0.22}  
        &  \grayl{\textbf{60.02} $\pm$ 0.25} 
        \\
        \midrule
    \end{tabular}
    \endgroup
\end{table}

\subsection{GenLabel using FC ReLU networks}\label{sec:GenLabel_FC_ReLU}

In the main manuscript, we reported the experimental results on the OpenML datasets for the logistic regression model. Table~\ref{Table:OpenML_3nn} shows the generalization performances for fully-connected (FC) ReLU networks with 2 hidden layers. 
Here we show the results of 16 best performing OpenML datasets.
For these selected datasets, mixup+GenLabel with Gaussian mixture (GM) model has a slight performance gain compared with vanilla training and mixup. Our ablation study shows that GenLabel outperforms alternative methods -- generative classifier (GM) and a method dubbed as \emph{excluding MI points}.
Table~\ref{Table:OpenML_3nn_rob} compares the adversarial robustness of different methods on selected OpenML datasets, under FGSM attack. For these datasets, GenLabel has $10 \sim 30 \%$ gain in robustness, compared with mixup and vanilla training. These results FC ReLU networks have similar behavior with the results for logistic regression in Tables~\ref{Table:OpenML} and~\ref{Table:OpenML_rob}.

\begin{table}[t]
\centering
	\caption{Generalization performances (clean accuracy in \%) on the best performing OpenML datasets in FC ReLU networks %
	}
	\scriptsize
	\setlength{\tabcolsep}{3pt} %
    \renewcommand{\arraystretch}{1} %
	\begin{tabular}{l|c|c|c|c|c|c|c|c}
        \midrule  
         Methods $\backslash$ OpenML Dataset ID 
         & 719
         & 770
         & 774
         & 804
         & 818
         & 862
         & 900
         & 906
         \\
		\midrule		 
        \textbf{Vanilla}
        & 71.62$\pm$5.55
        & 65.58$\pm$11.38
        & 59.34$\pm$5.16
        & 81.44$\pm$8.45
        & 87.56$\pm$19.75
        & 81.51$\pm$6.06
        & 61.00$\pm$1.62
        & 53.74$\pm$2.21
         \\
        \textbf{Mixup}
        & 70.89$\pm$5.47
        & 65.26$\pm$11.11
        & 59.80$\pm$5.04
        & 80.05$\pm$10.22
        & 88.17$\pm$15.68
        & 80.40$\pm$5.51
        & 60.99$\pm$2.05
        & 53.74$\pm$2.51
         \\          %
        \textbf{Mixup+Excluding MI}
        & 71.62$\pm$5.55
        & 64.15$\pm$9.62
        & 59.50$\pm$7.09
        & 80.05$\pm$10.22
        & 88.49$\pm$15.56
        & 79.21$\pm$8.03
        & 60.74$\pm$2.31
        & 53.50$\pm$2.31
                 \\         
        \textbf{Generative classifier (GM)}
        & 67.90$\pm$6.00
        & 51.69$\pm$10.93
        & 49.22$\pm$6.03
        & 71.59$\pm$10.21
        & 82.66$\pm$18.41
        & 67.62$\pm$10.23
        & 57.50$\pm$6.31
        & 48.26$\pm$3.02
         \\                              
        \grayl{\textbf{Mixup+GenLabel (GM)}}
        & \grayc{\textbf{73.10}$\pm$7.66}
        & \grayc{\textbf{66.69}$\pm$11.06}
        & \grayc{\textbf{59.95}$\pm$5.28}
        & \grayc{\textbf{81.57}$\pm$9.57}
        & \grayc{\textbf{89.15}$\pm$17.40}
        & \grayc{\textbf{82.70}$\pm$7.56}
        & \grayc{\textbf{61.24}$\pm$2.49}
        & \grayc{\textbf{54.21}$\pm$4.75}
         \\ 
		\bottomrule      
    \end{tabular} \begin{tabular}{l|c|c|c|c|c|c|c|c}
    \midrule  
    Methods $\backslash$ OpenML Dataset ID 
    & 908
    & 949
    & 956
    & 1011
    & 1014
    & 1045
    & 1055
    & 1075
    \\
    \midrule		 
    \textbf{Vanilla}
    & 54.00$\pm$1.70
    & 85.69$\pm$0.46
    & 68.90$\pm$2.52
    & 96.14$\pm$3.46
    & 80.55$\pm$0.25
    & 94.53$\pm$1.96
    & 78.77$\pm$4.36
    & 92.35$\pm$2.23
         \\
        \textbf{Mixup}
        & 55.00$\pm$1.95
        & 85.69$\pm$0.46
        & 69.88$\pm$4.01
        & 96.14$\pm$3.46
        & 80.55$\pm$0.25
        & 94.53$\pm$1.96
        & 78.77$\pm$4.36
        & 92.35$\pm$2.23
         \\ 
        \textbf{Mixup+Excluding MI}
        & 54.50$\pm$2.34
        & 85.69$\pm$0.46
        & 69.88$\pm$4.01
        & \textbf{96.43}$\pm$3.74
        & 80.55$\pm$0.25
        & 94.53$\pm$1.96
        & 78.77$\pm$4.36
        & 92.35$\pm$2.23
                 \\
        \textbf{Generative classifier (GM)}
        & 47.99$\pm$3.86
        & 65.59$\pm$15.61
        & 67.97$\pm$2.03
        & 95.53$\pm$2.72
        & 48.43$\pm$4.80
        & 94.53$\pm$1.96
        & 40.75$\pm$9.62
        & 90.83$\pm$2.73
         \\             
        \grayl{\textbf{Mixup+GenLabel (GM)}}
        & \grayc{\textbf{55.75}$\pm$1.48}
        & \grayc{\textbf{87.14}$\pm$3.66}
        & \grayc{\textbf{70.81}$\pm$3.75}
        & \grayc{\textbf{96.43}$\pm$3.74}
        & \grayc{\textbf{80.80}$\pm$0.61}
        & \grayc{\textbf{95.19}$\pm$1.57}
        & \grayc{\textbf{79.81}$\pm$4.01}
        & \grayc{\textbf{93.11}$\pm$2.41}
        \\
		\bottomrule      
    \end{tabular}

	\label{Table:OpenML_3nn}
\end{table}

\begin{table}[t]
\centering
	\caption{Robustness performances (robust accuracy in \%) on the best performing OpenML datasets in FC ReLU network, under FGSM attack %
	}
	\scriptsize
	\setlength{\tabcolsep}{3pt} %
    \renewcommand{\arraystretch}{1} %
	\begin{tabular}{l|c|c|c|c|c|c|c|c}
        \midrule  
         Methods $\backslash$ OpenML Dataset ID 
         & 312
         & 715
         & 718
         & 723
         & 797
         & 806
         & 837
         & 866
         \\
		\midrule		 
        \textbf{Vanilla}
        & 54.22$\pm$14.88
        & 42.70$\pm$3.62
        & 28.40$\pm$2.04
        & 39.90$\pm$3.54
        & 32.79$\pm$3.31
        & 32.69$\pm$3.51
        & 30.30$\pm$2.42
        & 41.10$\pm$2.07
         \\
        \textbf{Mixup}
        & 66.23$\pm$11.75
        & 44.10$\pm$3.27
        & 40.29$\pm$3.27
        & 41.39$\pm$3.45
        & 38.70$\pm$3.24
        & 35.99$\pm$2.78
        & 30.60$\pm$1.92
        & 45.80$\pm$1.67
         \\ 
        \grayl{\textbf{Mixup+GenLabel (GM)}} 
        & \grayc{\textbf{82.09}$\pm$0.08}
        & \grayc{\textbf{54.00}$\pm$0.95}
        & \grayc{\textbf{54.70}$\pm$0.89}
        & \grayc{\textbf{52.10}$\pm$0.89}
        & \grayc{\textbf{55.10}$\pm$0.73}
        & \grayc{\textbf{53.39}$\pm$2.44}
        & \grayc{\textbf{49.99}$\pm$1.98}
        & \grayc{\textbf{58.00}$\pm$0.32}
         \\  
		\bottomrule      
    \end{tabular} \begin{tabular}{l|c|c|c|c|c|c|c|c}
    \midrule  
    Methods $\backslash$ OpenML Dataset ID & 871
    & 909
    & 917
    & 1038
    & 1043
    & 1130
    & 1138
    & 1166
    \\
    \midrule		 
    \textbf{Vanilla}
        & 32.48$\pm$4.20
        & 39.26$\pm$3.32
        & 37.90$\pm$3.51
        & 16.49$\pm$1.85
        & 57.65$\pm$2.54
        & 47.71$\pm$8.86
        & 53.97$\pm$4.77
        & 33.92$\pm$5.39
         \\
        \textbf{Mixup}
        & 32.07$\pm$2.69
        & 39.22$\pm$5.52
        & 41.20$\pm$3.71
        & 14.45$\pm$2.25
        & 61.95$\pm$1.82
        & 49.20$\pm$9.45
        & 62.00$\pm$4.31
        & 46.21$\pm$5.58
         \\    
        \grayl{\textbf{Mixup+GenLabel (GM)}} 
        & \grayc{\textbf{41.42}$\pm$0.73}
        & \grayc{\textbf{50.50}$\pm$0.61}
        & \grayc{\textbf{51.50}$\pm$2.30}
        & \grayc{\textbf{26.04}$\pm$2.87}
        & \grayc{\textbf{74.97}$\pm$0.16}
        & \grayc{\textbf{85.57}$\pm$2.01}
        & \grayc{\textbf{88.70}$\pm$2.48}
        & \grayc{\textbf{74.62}$\pm$3.09}
        \\
		\bottomrule      
    \end{tabular}

	\label{Table:OpenML_3nn_rob}
\end{table}

\section{Additional mathematical results}

Below we state the approximation of GenLabel loss $L_n^{\op{gen}} (\boldsymbol{\theta}, S)$. The proof of this lemma is in Section~\ref{sec:proof_lemma: GenLabel loss}.

\begin{lemma}\label{Lemma: GenLabel loss for 2 class} 
The second order Taylor approximation of the GenLabel loss is given by

 \begin{equation*}%
\tilde{L}_n^{\op{gen}}(\boldsymbol{\theta},S) = L_n^{\op{std}}(\boldsymbol{\theta},S)+ {R}_1^{\op{gen}}(\boldsymbol{\theta},S)+
{R}_2^{\op{gen}}(\boldsymbol{\theta},S)+
{R}_3^{\op{gen}}(\boldsymbol{\theta},S),
\end{equation*}
where
\begin{align*}
   & {R}_1^{\op{gen}}(\boldsymbol{\theta},S)=\frac{1}{n} \sum_{i=1}^n A_{\sigma_1,c,\tau,d} (h'(f_{\boldsymbol{\theta}}(\boldsymbol{x_i}))-y_i) \nabla f_{\boldsymbol{\theta}}(\boldsymbol{x_i})^T \mathbb{E}_{\boldsymbol{r_x}\sim D_X}[\boldsymbol{r_x}-\boldsymbol{x_i}], \\
   & {R}_2^{\op{gen}}(\boldsymbol{\theta},S)=\frac{1}{2n}\sum_{i=1}^n B_{\sigma_1,c,\tau,d} h''(f_{\boldsymbol{\theta}}(\boldsymbol{x_i}))\nabla f_{\boldsymbol{\theta}}(\boldsymbol{x_i})^T \mathbb{E}_{\boldsymbol{r_x}\sim D_X}[(\boldsymbol{r_x}-\boldsymbol{x_i})(\boldsymbol{r_x}-\boldsymbol{x_i})^T]\nabla f_{\boldsymbol{\theta}}(\boldsymbol{x_i}), \\
   & {R}_3^{\op{gen}}(\boldsymbol{\theta},S)=\frac{1}{2n}\sum_{i=1}^n B_{\sigma_1,c,\tau,d} (h'(f_{\boldsymbol{\theta}}(\boldsymbol{x_i}))-y_i) \mathbb{E}_{\boldsymbol{r_x}\sim D_X}[(\boldsymbol{r_x}-\boldsymbol{x_i})^T\nabla^2f_{\boldsymbol{\theta}}(\boldsymbol{x_i})(\boldsymbol{r_x}-\boldsymbol{x_i})],
\end{align*}
where $A_{\sigma_1,c,\tau,d}$ and $B_{\sigma_1,c,\tau,d}$ are two constants defined in (\ref{A,B}).
When $\sigma_1\to \infty,$ we have $\lim_{\sigma_1\to \infty} A_{\sigma_1,c,\tau,d} = \frac{c^2 + 1}{2(c+1)^2}<\frac{1}{3}, \lim_{\sigma_1\to \infty}B_{\sigma_1,c,\tau,d} = \frac{c^2-c+1}{3(c+1)^2}<\frac{1}{6}$.
\end{lemma}

\section{Proof of mathematical results}\label{Sec:proofs}

\subsection{Proof for Example~\ref{ex:intrusion}}\label{sec:proof_prop:toy_intrusion_margin}
We start with showing $\theta_{\op{mixup}}^{\star} = \frac{7}{16}$. 
First, we the prediction of the classifier can be represented as
\begin{align*}
f_{\theta}(x) = 
\begin{cases}
\frac{1}{2\theta} \lvert x \rvert, & \op{ if } 0 \leq \lvert x \rvert \leq  2 \theta, \\
1, & \op{ if } 2 \theta \leq \lvert x \rvert \leq 1.
\end{cases}
\end{align*}
Since we have three data points, we have $\binom{3}{2} = 3$ different way of mixing the data points: (1) mixing $x_1$ and $x_2$, (2) mixing $x_2$ and $x_3$, (3) mixing $x_3$ and $x_1$. We denote the loss value of $i$-th mix pair as $L_i$. 
We first compute $L_2$, the loss of mixing $x_2$ and $x_3$. The mixed point is $x^{\op{mix}} = \lambda x_3 + (1-\lambda)x_2 = \lambda$, which has label $\vy^{\op{mix}} = \lambda \ve_1 + (1-\lambda) \ve_2$ for $\lambda \in [0,1]$. Then, 
\begin{align*}
    L_2 &= \int_{0}^1 \left \lVert \vy^{\op{mix}} - \begin{bmatrix}
f_{\theta}(x^{\op{mix}}) \\
1 - f_{\theta}(x^{\op{mix}})
\end{bmatrix} \right \rVert_2^2  d\lambda = 2 \int_{0}^1 (\lambda - f_{\theta} (x^{\op{mix}}))^2 d\lambda = 2 \int_{0}^{2 \theta} \lambda^2 (1-\frac{1}{2\theta})^2 \dd\lambda + 2 \int_{2\theta}^{1} (\lambda-1)^2 \dd\lambda \\ 
&= \frac{2}{3} (2 \theta-1)^2.
\end{align*}
Since $f_{\theta}(x)$ is symmetric, we have $L_1 = L_2$.
Now, we compute $L_3$. The mixed point is represented as $x^{\op{mix}} = \lambda x_3 + (1-\lambda)x_1 = 2\lambda-1$, which is labeled as 
$\vy^{\op{mix}} = \lambda \ve_1 + (1-\lambda) \ve_1 = \ve_1$, for $\lambda \in [0,1]$. 
Then, 
\begin{align*}
    L_3 &= \int_{0}^1 \left \lVert \vy^{\op{mix}} - \begin{bmatrix}
f_{\theta}(x^{\op{mix}}) \\
1 - f_{\theta}(x^{\op{mix}})
\end{bmatrix} \right \rVert_2^2  d\lambda = \int_{0}^1 (1 - f_{\theta}(x^{\op{mix}}) )^2 d \lambda \\
&= \int_{\frac{1}{2}-\theta}^{\frac{1}{2}+ \theta} (1 - \frac{1}{2\theta} \lvert 2\lambda-1 \rvert )^2 \dd\lambda = 2 \int_{\frac{1}{2}}^{\frac{1}{2}+\theta} (1 - \frac{1}{2\theta} ( 2\lambda-1 ) )^2 \dd\lambda = \frac{2}{3} \theta.
\end{align*}
Thus, $\frac{\dd}{\dd\theta} (\frac{3}{2}(L_1 + L_2 + L_3)) = \frac{\dd}{\dd\theta} (8 \theta^2 - 7\theta + 2) = 0$ when $\theta=\frac{7}{16}$. This completes the proof of $\theta_{\op{mixup}}^{\star} = \frac{7}{16}$.

Finally, $\theta_{\op{mixup-without-MI}}^{\star} = \frac{1}{2}$ is trivial from the fact that $\frac{\dd}{\dd\theta} (\frac{3}{4}(L_1 + L_2)) = \frac{\dd}{\dd\theta} (2\theta-1)^2 = 0$ when $\theta=\frac{1}{2}$.

\subsection{Proof of Proposition~\ref{Prop:Gaussian}}\label{sec:proof_prop:Gaussian}

\begin{proof}
Denote the mean of Gaussian distribution for each class by $\mu_0 = 0$ and $\mu_1 = 1$. The variance of Gaussian distribution is denoted by $\sigma^2$.
Let $x_i$ be the feature sampled from class $i \in \{0, 1\}$.
For small $\sigma$, we have $x_0 \simeq \mu_0 = 0$ and $x_1 \simeq \mu_1 = 1$. Then, the mixed point is represented as 
$x^{\op{mix}} = (1-\lambda) x_0 + \lambda x_1 \simeq \lambda $. 
The label of mixup is represented as $y^{\op{mix}} = (1-\lambda) y_0 + \lambda y_1 = \lambda$. The label of mixup+GenLabel is given as 
\begin{align*}
y^{\op{gen}} &= \frac{p_0 }{p_0 + p_1} y_0 + \frac{p_1 }{p_0 + p_1} y_1 \\
&=  \frac{p_1 }{p_0 + p_1} \stackrel{\rm (a)}{=} \frac{ \exp(- (1-\lambda)^2 / 2\sigma^2) }{ \exp(- \lambda^2 / 2\sigma^2) + \exp(- (1-\lambda)^2 / 2\sigma^2) } = \frac{1}{1 + \exp(- (\lambda - 1/2) / \sigma^2)}
\end{align*}
where (a) is from 
\begin{align*}
    p_0 &= p(x^{\op{mix}} | y=0) = \frac{1}{\sqrt{2\pi}} \exp(-(\lambda-0)^2/2\sigma^2), \\
    p_1 &= p(x^{\op{mix}} | y=1) = \frac{1}{\sqrt{2\pi}} \exp(- (1 - \lambda)^2/2\sigma^2).
\end{align*}
This completes the proof.
\end{proof}

\begin{figure}[t]
    \vspace{-2mm}
	\centering
	\includegraphics[width=0.8\linewidth]{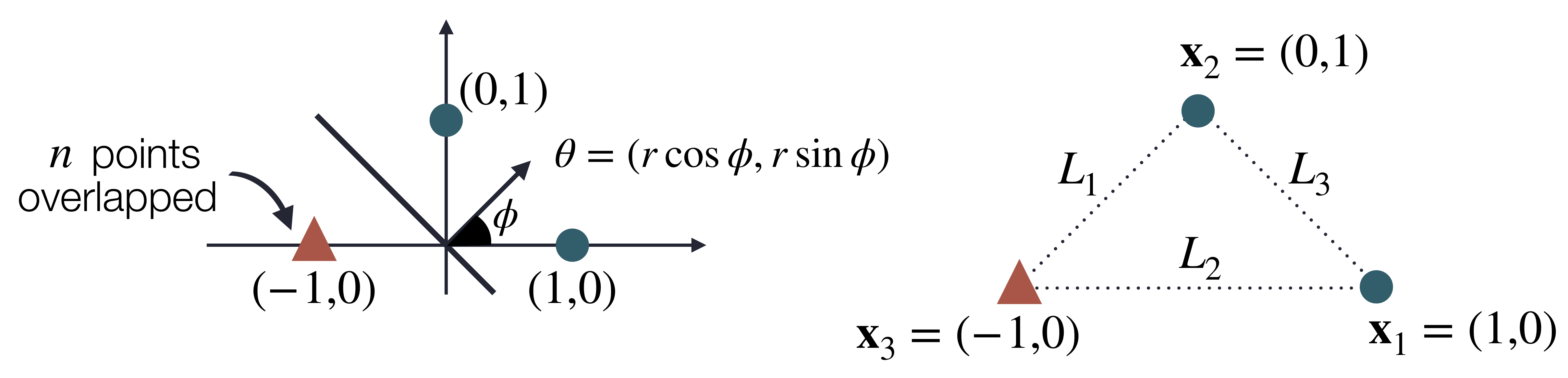}
	\vspace{-2mm}
	\caption
	{
	Left: the dataset $S$ used in Example~\ref{ex:n+2_dots}. 
	The feature-label pairs are defined as
	$(\vx_1, y_1) = ([1, 0],  +1)$, $(\vx_2, y_2) = ([0, 1],  +1)$, and $(\vx_i, y_i) = ([-1, 0], -1)$ for    $i=3, 4, \cdots, n+2$. Right: the line segments connecting training data points.
	}
	\label{Fig:toy_label_appendix}
\end{figure}

\subsection{Proof for Example~\ref{ex:n+2_dots}}\label{Sec:toy_proof}

Consider the problem of classifying $n+2$ data points $S = \{ (\vx_i, y_i) \}_{i=1}^{n+2}$, where the feature $\vx_i \in \mathbb{R}^2$ and the label $y_i \in \{+1, -1\}$ of each point is specified in Fig.~\ref{Fig:toy_label_appendix}.
We use the one-hot label $\vy_i = [1, 0]$ for class $+1$ and $\vy_i = [0, 1]$ for class $-1$.
Consider applying logistic regression to this problem, where the solution is represented as $\bm{\theta} = [r \cos \phi, r \sin \phi]$.
Here we compare three 
different schemes: (1) vanilla training, (2) mixup, and (3) mixup with GenLabel (dubbed as new-mixup).
The first scheme is nothing but training only using the given training data $S$. Both mixup and new-mixup generate mixed points using linear combination of data points, \ie $\vx_{ij} = \lambda \vx_{i} + (1- \lambda) \vx_{j}$ for some $\lambda \sim \text{Beta}(\alpha, \alpha)$, while the labeling method is different. The original mixup uses $\vy_{ij} = \lambda \vy_{i} + (1- \lambda) \vy_{j}$, whereas the new-mixup
uses $\vy_{ij} = \rho \vy_{i} + (1- \rho) \vy_{j}$ where $\rho = \frac{1}{1 + \exp\{- (\lambda - 1/2)/\sigma^2\}}$ for some small $\sigma$, according to Proposition~\ref{Prop:Gaussian}, assuming 
the class +1 is modeled as Gaussian mixture.
We analyze the solutions of these schemes, denoted by $\bm{\theta}_{\text{vanilla}}$, $\bm{\theta}_{\text{mixup}}$ and $\bm{\theta}_{\text{new-mixup}}$, and compare it with the $L_2$ max-margin classifier obtained from support vector machine (SVM), represented as $\bm{\theta}_{\op{svm}} = ( \cos \frac{\pi}{4}, \sin  \frac{\pi}{4} )$. Here, we denote the angle of SVM solution by $\phi_{\op{svm}} = \pi/4$. 
Below we first analyze the loss of vanilla training, and then provide analysis on the loss of the mixup scheme (using either original linear labeling or the suggested GenLabel).

\paragraph{Vanilla training}

Consider the vanilla training which learns $\bm{\theta}$ (or the corresponding $\phi$) by only using the given data. In this case, the sum of logistic loss over all samples can be represented as 
\begin{align}%
\ell_{\text{vanilla}} = \sum_{i=1}^{n+2} \log (1 + e^{- y_i  \bm{\theta}^T \vx_i  } ),
\end{align}
where the exponential term for each data is
\begin{align*}
-y_1 \vtheta^T \vx_1 &= - (r \cos \phi, r \sin \phi )^T (1, 0) = -r \cos \phi, \\
-y_2 \vtheta^T \vx_2 &= - (r \cos \phi, r \sin \phi )^T (0, 1) = -r \sin \phi, \\
-y_i \vtheta^T \vx_i &=  (r \cos \phi, r \sin \phi )^T (-1, 0) = - r \cos \phi,  \quad \quad \quad i = 3,4, \cdots, n+2 
\end{align*}
Then, the loss of vanilla training is 
\begin{align}\label{Eqn:loss_vanilla}
\ell_{\text{vanilla}} = (n+1) \log (1 + e^{-r \cos \phi}) + \log (1 + e^{-r \sin \phi}).
\end{align}
The derivative of the loss with respect to $\phi$ is given as
\begin{align*}
R(\phi) \coloneqq \frac{d}{d\phi} \ell_{\text{vanilla}} = 
(n+1) \frac{r \sin \phi \exp\{-r \cos \phi \}}{1 + \exp\{-r \cos \phi\}} + \frac{-r \cos \phi \exp\{-r \sin \phi\}}{1 + \exp\{-r \sin \phi\}}.
\end{align*}
By plugging in $\phi_{\op{svm}} = \pi/4$ in this expression, we have 
\begin{align*}
R(\phi = \phi_{\op{svm}}) = \frac{nr}{\sqrt{2}} \cdot  \frac{ \exp\{-r/\sqrt{2}\} }{1 + \exp\{-r/\sqrt{2}\}} 
= 
\frac{n}{\sqrt{2}} \cdot  \frac{ r}{1 + e^{r/\sqrt{2}}}
\neq 0,
\end{align*}
meaning that vanilla training cannot achieve the max-margin classifier $\vtheta_{\text{svm}}$ for a fixed $r > 0$.
Note that $R(\phi=\phi_{\op{svm}}) \rightarrow 0$ holds when $\lVert \vtheta \rVert = r \rightarrow \infty$, \ie the vanilla gradient descent training achieves the SVM solution. This coincides with the result of ~\citep{soudry2018implicit} which showed that for linearly separable data, the model parameter $\vw(t)$ updated by gradient descent satisfies both $\lim_{t \rightarrow \infty} \lVert \vtheta(t) \rVert = \infty$ and $\lim_{t \rightarrow \infty} \vtheta(t) / \lVert \vtheta(t) \rVert = \vtheta_{\op{svm}}$.

\paragraph{Mixup}
Now we analyze the case of mixup + GenLabel (or new-mixup). Here we briefly recap how the suggested data augmentation works.
Basically, following the vanilla mixup scheme, we randomly sample data points $\vx_i$ and $\vx_j$, and generate augmented data $\vx_{ij} = \lambda \vx_i + (1-\lambda) \vx_j$ where $\lambda \sim \op{Beta}(\alpha,\alpha)$ for some $\alpha > 0$. Then, we label this augmented data as 
$\vy_{ij} = \rho \vy_{i} + (1-\rho) \vy_{j}$ where
$\rho = \lambda$ for vanilla mixup with linear labeling, and $\rho = \frac{1}{1 + \exp\{-(\lambda - 1/2)/\sigma^2\}}$ for new labeling, where $\sigma$ is a small positive number.
Since there are total $n+2$ points in the training set, we have $(n+2)^2$ pairs of $\vx_i, \vx_j \in X$. The sum of loss values of all pairs can be represented as
\begin{align}\label{Eqn:loss_mixup}
    \ell_{\text{mixup}} = 2n \int_{L_1 \cup L_2} \ell(\vy_{ij}, \hat{\vy}_{ij}) 
    + n^2 \ell(\vy_3, \hat{\vy}_3) + 2 \int_{L_3} \ell(\vy_{ij}, \hat{\vy}_{ij}) + \ell(\vy_1, \hat{\vy}_1) + \ell(\vy_2, \hat{\vy}_2)
\end{align}
where 
the line segments $L_1, L_2, L_3$ 
are illustrated in Fig.~\ref{Fig:toy_label_appendix}.
Note that each line segment can be represented as the set of following $(\vx_{ij}, \vy_{ij})$ pairs for $\lambda \in [0,1]$: 
\begin{align*}
L_1 &: \vx_{ij} = (-\lambda, 1-\lambda ), \quad \vy_{ij} = [1-\rho,  \rho] \\
L_2 &: \vx_{ij} = ( 2\lambda -1, 0 ), \quad \quad \vy_{ij} = [\rho, 1-\rho] \\
L_3 &: \vx_{ij} = ( 1-\lambda, \lambda  ), \quad \quad \vy_{ij} = [1, 0] 
\end{align*}

Recall that for a given random data $\vx$, the label estimated by logistic regression model $\vtheta$ is represented as $\bm{\hat{y}} = [\hat{y}^{(0)}, \hat{y}^{(1)}] = [ \frac{1}{1+\text{exp}(-\vtheta^T \vx)}, \frac{1}{1+\text{exp}(+\vtheta^T \vx)}]$. If this sample has true one-hot encoded label $\vy = [y^{(0)}, y^{(1)}]$, then the logistic loss of this model (regarding the specific sample ($\vx$, $\vy$)) is given as
\begin{align}%
\ell (\vy, \hat{\vy})= - y^{(0)} \log  \hat{y}^{(0)} - y^{(1)} \log  \hat{y}^{(1)}
\end{align}

Thus, each loss term in (\ref{Eqn:loss_mixup}) can be represented as
\begin{align*}
\int_{L_1} \ell(\vy_{ij}, \hat{\vy}_{ij}) &= \int_{0}^{1} \{ (1-\rho) \log(1+ e^{ \lambda r \cos \phi - (1-\lambda) r \sin \phi }) + \rho \log(1 + e^{ - \lambda r \cos \phi + (1-\lambda) r \sin \phi }) \}   p(\lambda) \dd \lambda , \\
\int_{L_2} \ell(\vy_{ij}, \hat{\vy}_{ij}) &= \int_{0}^{1} \{ (1-\rho) \log(1 + e^{  (2\lambda-1) r \cos \phi }) + \rho \log(1 + e^{ - (2\lambda-1) r \cos \phi }) \}   p(\lambda) \dd \lambda , \\
\int_{L_3} \ell(\vy_{ij}, \hat{\vy}_{ij}) &= \int_{0}^{1}  \log(1 + e^{ - (1-\lambda) r \cos \phi - \lambda r \sin \phi })  \quad  p(\lambda) \dd \lambda , \\
\ell(\vy_{1}, \hat{\vy}_{1}) &= \ell(\vy_{3}, \hat{\vy}_{3}) =    \log(1 + e^{ - r \cos \phi}), \\
\ell(\vy_{2}, \hat{\vy}_{2}) &=    \log(1 + e^{ - r \sin \phi}),
\end{align*}
where $p(\lambda)$ is the probability density function for sampling $\lambda$.

Based on the expression of the loss $\ell(r, \phi)$ for each scheme given in (\ref{Eqn:loss_vanilla}) and (\ref{Eqn:loss_mixup}), we numerically plotted $\phi^{\star} = \argmin_{\phi} \ell(r, \phi)$ for various $r$ in Fig.~\ref{Fig:toy_label_correction}.
It turns out that the optimal $\phi^{\star}_{\text{new-mixup}}$ of new-mixup approaches to the SVM solution $\phi_{\op{svm}} = \pi/4 $ as $r$ increases. Using the standard definition of margin denoted by $\text{margin}(\vtheta) = \min\limits_{(\vx_i, y_i) \in D} \frac{y_i \vtheta^T \vx_i}{\lVert \vtheta \rVert} = \min \{ \cos \phi, \sin \phi \}$, we have
\begin{align*}
    \text{margin}(\vtheta_{\op{svm}}) = \text{margin}(\vtheta_{\text{new-mixup}}) > \text{margin}(\vtheta_{\text{vanilla}}) > \text{margin}(\vtheta_{\text{mixup}})
\end{align*}
according to Fig.~\ref{Fig:toy_label_correction}.

\subsection{Proof of Theorem~\ref{Theorem:robustness_logreg}}
Following the proof of Theorem 3.1 of~\citep{zhang2021does}, when $\boldsymbol{\theta}\in \Theta$, we have
\[(h'(f_{\boldsymbol{\theta}}(\boldsymbol{x_i}))-y_i) \nabla f_{\boldsymbol{\theta}}(\boldsymbol{x_i})^T \mathbb{E}_{\boldsymbol{r_x}\sim D_X}[\boldsymbol{r_x}-\boldsymbol{x_i}]\geq 0,\]
\[h''(f_{\boldsymbol{\theta}}(\boldsymbol{x_i}))\nabla f_{\boldsymbol{\theta}}(\boldsymbol{x_i})^T \mathbb{E}_{\boldsymbol{r_x}\sim D_X}[(\boldsymbol{r_x}-\boldsymbol{x_i})(\boldsymbol{r_x}-\boldsymbol{x_i})^T]\nabla f_{\boldsymbol{\theta}}(\boldsymbol{x_i})\geq 0.\]
The first inequality in Theorem \ref{Theorem:robustness_logreg} is directly obtained by combining Lemma~\ref{Lemma: mixup loss} and the fact that $A_{\sigma_1,c,\tau,d}^i < 1/3$ and $B_{\sigma_1,c,\tau,d}^i < \frac{1}{6}$ holds, which is proven in Lemma~\ref{Lemma: GenLabel loss for 2 class}.
The second inequality in Theorem \ref{Theorem:robustness_logreg} is obtained by applying Theorem 3.1 of~\citep{zhang2021does} into the Taylor approximation of GenLabel loss $\tilde{L}_n^{\op{gen}}(\theta,S)$ in Lemma~\ref{Lemma: GenLabel loss for 2 class}.

\subsection{Proof of Theorem~\ref{Theorem:robustness_relu}}

Similar to the proof of Theorem~\ref{Theorem:robustness_logreg}, the first inequality is directly from Lemma~\ref{Lemma: GenLabel loss for 2 class}.
The second inequality is obtained by applying Theorem 3.3 of~\citep{zhang2021does} into the Taylor approximation of GenLabel loss $\tilde{L}_n^{\op{gen}}(\theta,S)$ in Lemma~\ref{Lemma: GenLabel loss for 2 class}.

\subsection{Lemmas used for proving Lemma~\ref{Lemma: GenLabel loss for 2 class}}

We here provide lemmas that are used in the proof of Lemma~\ref{Lemma: GenLabel loss for 2 class}, which is given in Section~\ref{sec:proof_lemma: GenLabel loss}. Before stating our first lemma, recall that the covariance matrix of each class-conditional data distribution is a scalar factor of $\boldsymbol{\Sigma}$, which is defined as 
\begin{equation}\label{covariance}
\boldsymbol{\Sigma} = \left(
  \begin{array}{ccccc}
    1 & \tau & \tau & \cdots & \tau \\
    \tau & 1 & \tau & \cdots & \tau \\
    \vdots & \tau & \ddots & \tau & \vdots \\
    \tau & \tau & \cdots & 1 & \tau \\
    \tau &\tau & \cdots & \tau & 1 \\
  \end{array}
\right).
\end{equation}

Below we provide the inverse matrix of $\boldsymbol{\Sigma}$.
\begin{lemma}\label{Lemma: inverse}
When $\tau \notin \{ \frac{-1}{d-1}, \frac{-1}{d-2} \}$ and $-1<\tau<1$, the matrix in (\ref{covariance}) is invertible. The inverse is given by
\begin{equation}\label{Sigma inverse}
\boldsymbol{\Sigma}^{-1} = c_d\left(
  \begin{array}{ccccc}
    1 & -\tau_d & -\tau_d & \cdots & -\tau_d \\
    -\tau_d & 1 & -\tau_d & \cdots & -\tau_d \\
    \vdots & -\tau_d & \ddots & -\tau_d & \vdots \\
    -\tau_d & -\tau_d & \cdots & 1 & -\tau_d \\
    -\tau_d & -\tau_d & \cdots & -\tau_d & 1 \\
  \end{array}
\right),
\end{equation}
where  
\begin{equation}\label{c_d}
 c_d = \frac{1}{1-\tau}\frac{(d-2)\tau+1}{(d-1)\tau+1} 
\end{equation}
and 
\begin{equation}\label{epsilon_d}
\tau_d = \frac{\tau}{(d-2)\tau+1}.
\end{equation}

\end{lemma}

\begin{proof}
We prove the lemma by verifying $\boldsymbol{\Sigma} \times (\ref{Sigma inverse}) = \mI_d$.

Clearly the diagonal element in $\boldsymbol{\Sigma} \times (\ref{Sigma inverse})$ reads
\begin{align*}
  c_d[1-(d-1)\tau \tau_d]  & = \frac{1}{1-\tau}\frac{(d-2)\tau +1}{(d-1)\tau +1}[1-(d-1)\frac{\tau^2}{(d-2)\tau+1}]   \\
    & = \frac{1}{1-\tau}\frac{(d-2)\tau +1}{(d-1)\tau +1} \frac{(d-2)\tau+1-\tau^2 (d-1)}{(d-2)\tau+1} \\ 
    & =  \frac{1}{1-\tau}\frac{(d-2)\tau+1-\tau^2 (d-1)}{(d-1)\tau+1}=\frac{1}{1-\tau}\frac{(1-\tau)(\tau(d-1)+1)}{(d-1)\tau+1}=1.
\end{align*}
The off-diagonal element in $\boldsymbol{\Sigma} \times (\ref{Sigma inverse})$ reads
\begin{align*}
c_d(\tau-\tau_d-(d-2)\tau \tau_d)    &=c_d[\frac{(d-2)\tau^2 +\tau - \tau}{(d-2)\tau+1}-\frac{(d-2)\tau^2}{(d-2)\tau+1}     ] =0.
\end{align*}
Then we conclude the proof.

\end{proof}

\begin{lemma}\label{Lemma: induction on ZZ}
For $\boldsymbol{Z}\sim \mathcal{N}(\mathbf{0},\boldsymbol{\Sigma})$, we have the following formula for $\boldsymbol{Z}^T \boldsymbol{\Sigma}^{-1}\boldsymbol{Z}$:
\begin{align*}
\boldsymbol{Z}^T \boldsymbol{\Sigma}^{-1} \boldsymbol{Z}   &  =c_d \Big[ A_1[Z_1 - B_1(\sum_{i=2}^d Z_i)]^2 + A_2[Z_2-B_2(\sum_{i=3}^d Z_i)]^2 \\
 &     +\cdots + A_{d-1}[Z_{d-1} - B_{d-1}(\sum_{i=d}^d Z_i)]^2 + A_d Z_d^2\Big],
\end{align*}
where $A_n,B_n$ are constants that satisfy the following recurrence relation for $n\leq d$
\begin{equation}\label{recurrence}
\begin{split}
    &  A_n = A_{n-1}-B_{n-1}^2 A_{n-1},\quad  B_{n} = \frac{A_{n-1}B_{n-1}+B_{n-1}^2 A_{n-1}}{A_{n}}, \\
    &A_1 = 1, \quad B_1 = \tau_d,
\end{split}    
\end{equation}
and $\boldsymbol{Z} = [Z_1, \cdots, Z_d]$.
\end{lemma}

\begin{proof}
Using (\ref{Sigma inverse}) we have
\begin{align*}
\boldsymbol{Z}^T \boldsymbol{\Sigma}^{-1}\boldsymbol{Z}     &= c_d[\underbrace{\sum_{i=1}^d Z_i^2 - 2\tau_d\sum_{i\neq j}Z_iZ_j }_{(*)}].
\end{align*}
We focus on (*). We claim the following induction formula:

\textbf{Claim}: for any $n<d$, and $A_n,B_n$ satisfying (\ref{recurrence}), we can decompose $(*)$ into
\begin{align}
(*)    & =  A_1[Z_1-B_1(\sum_{i=2}^d Z_i)]^2 + \cdots + A_n[Z_n - B_n(\sum_{i=n+1}^d Z_i)]^2 \notag \\
&+ A_{n+1}\sum_{i=n+1}^d  Z_i^2 - 2A_{n+1}B_{n+1} \sum_{i,j=n+1,i\neq j}^d Z_iZ_j   \label{Claim}.
\end{align}
The lemma immediately follows by setting $n=d-1$ in the claim. Now we use induction to prove the claim.

\textbf{Base case:} when $n=1$, we complete the square for $Z_1$ and obtain
\begin{align*}
    (*) &   =    [Z_1-\tau_d(Z_2+\cdots+Z_d)]^2- \tau_d^2 (Z_2+\cdots+Z_d)^2 + \sum_{i=2}^d Z_{i}^2 -2\tau_d \sum_{i,j=2,i\neq j}Z_i Z_j\\
    & = [Z_1-\tau_d(Z_2+\cdots+Z_d)]^2 + (1-\tau_d^2)\sum_{i=2}^d Z_{i}^2 -2(\tau_d+\tau_d^2) \sum_{i,j=2,i\neq j}Z_i Z_j.
\end{align*}
We conclude the base case with $A_1,B_1,A_2,B_2$ satisfying (\ref{recurrence}) as:
\[A_1=1,\quad B_1=\tau_d,\quad A_2=1-\tau_d^2=A_1-B_1^2A_1,\]
\[A_2B_2=\tau_d + \tau_d^2 = A_2\frac{A_1B_1+B_1^2A_1}{A_2}=A_1B_1+B_1^2A_1. \]

\textbf{Induction hypothesis:} we assume the claim holds true for $n$. We want to show the claim also holds true for $n+1$. We focus on the second line of the claim:~(\ref{Claim}). We further complete the square and have
\begin{align*}
(\ref{Claim})    & :=  A_{n+1}\sum_{i=n+1}^d  Z_i^2 - 2A_{n+1}B_{n+1} \sum_{i,j=n+1,i\neq j}^d Z_iZ_j \\
&= A_{n+1} [Z_{n+1} -   B_{n+1}\sum_{i=n+2}^d Z_i]^2 - A_{n+1}B_{n+1}^2 (\sum_{i=n+2}^d Z_i)^2\\
& + A_{n+1}\sum_{i=n+2}^d Z_i^2 - 2A_{n+1}B_{n+1}\sum_{i,j=n+2,i\neq j}^d Z_iZ_j \\
& = A_{n+1} [Z_{n+1} -   B_{n+1}\sum_{i=n+2}^d Z_i]^2 \\
& + (A_{n+1}-A_{n+1}B_{n+1}^2) \sum_{i=n+2}^d Z_i^2 -2(A_{n+1}B_{n+1}+A_{n+1}B_{n+1}^2) \sum_{i,j=n+2,i\neq j}Z_iZ_j \\
& = A_{n+1} [Z_{n+1} -   B_{n+1}\sum_{i=n+2}^d Z_i]^2 + A_{n+2} \sum_{i=n+2}^d Z_i^2 -2A_{n+2}B_{n+2} \sum_{i,j=n+2,i\neq j}Z_iZ_j.
\end{align*}
Thus the claim holds true for $n+1$ with $A_{n+2},B_{n+2},A_{n+1},B_{n+1}$ satisfying  (\ref{recurrence}) as
\[A_{n+2} = A_{n+1}-A_{n+1}B_{n+1}^2, \quad B_{n+2} = \frac{A_{n+1}B_{n+1}+A_{n+1}B_{n+1}^2}{A_{n+2}}.\] 
Then we conclude the claim and the lemma.

\end{proof}

\begin{lemma}\label{Lemma: cdf of Y}
Denote $\boldsymbol{Y} = (Y_1,\cdots,Y_j)$, and $Y_j=Z_j - B_j(\sum_{i=j+1}^d Z_i)$ with $B_j,Z_j$ defined in Lemma \ref{Lemma: induction on ZZ} for $1\leq j\leq d$, then $Y_j$ follows a 1-D Gaussian distribution:
\[Y_j \sim \mathcal{N}(0,\frac{1}{c_dA_j}).\]

\end{lemma}

\begin{proof}
Applying Lemma \ref{Lemma: induction on ZZ}, we compute the cumulative density function of $Y_j$ as
\begin{align}
&P(Y_j<x)      = P(Z_j - B_j(\sum_{i=j+1}^d Z_i)<x)=\int_{(Z_{j+1},Z_{j+2},\cdots,Z_{d})\in \mathbb{R}^{d-j}} \notag\\
    & \times \int_{-\infty}^{x+B_j(Z_{j+1}+\cdots+Z_d)} \int_{(Z_{1},Z_2,\cdots,Z_{j-1})\in \mathbb{R}^{j-1}} (2\pi)^{-\frac{d}{2}} \det(\boldsymbol{\Sigma})^{-\frac{1}{2}} \exp\{ -\frac{1}{2}Z^T \boldsymbol{\Sigma}^{-1} Z   \} \dd \boldsymbol{Z}  \notag\\
    & =\int_{(Z_{j+1},Z_{j+2},\cdots,Z_{d})\in \mathbb{R}^{d-j}}   (2\pi)^{-\frac{d-j}{2}}  \exp\{-\frac{c_d}{2}[A_{j+1}Y_{j+1}^2+\cdots+A_dY_d^2]\} \notag\\
    & \times \int_{-\infty}^{x+B_j(Z_{j+1}+\cdots+Z_d)} (2\pi)^{-\frac{1}{2}} \exp\{-\frac{c_d}{2}A_jY_j^2\} \notag\\
    &\times \int_{(Z_{1},Z_2,\cdots,Z_{j-1})\in \mathbb{R}^{j-1}} (2\pi)^{-\frac{j-1}{2}} \det(\boldsymbol{\Sigma})^{-\frac{1}{2}} \exp\{ -\frac{c_d}{2}[A_1Y_1^2+\cdots+A_{j-1}Y_{j-1}^2]   \} \dd \boldsymbol{Z} \label{cdf: Y3},
\end{align}
where we used $Y_d = Z_d$.
Note that $Y_j = Z_j-B_j(\sum_{i=j+1}^m Z_i)$, we apply change of variable
\begin{equation}\label{change of variable}
 (Z_1-B_1(\sum_{i=2}^d Z_i),\cdots,Z_{d})\to (Y_1,\cdots,Y_{d}).   
\end{equation}
The corresponding Jacobian matrix $|\frac{\partial (Y_1,\cdots,Y_d)}{\partial (Z_1,\cdots,Z_d)}|$ is an upper triangular matrix with diagonal element $1$. Thus the Jacobian is $1$, and we conclude
\begin{align}
 (\ref{cdf: Y3}) & = \int_{(Y_{j+1},Y_{j+2},\cdots,Y_{d})\in \mathbb{R}^{d-j}}   (2\pi)^{-\frac{d-j}{2}}  \exp\{-\frac{c_d}{2}[A_{j+1}Y_{j+1}^2+\cdots+A_dY_d^2]\} \notag\\
 & \times  \int_{-\infty}^{x} (2\pi)^{-\frac{1}{2}} \exp\{-\frac{c_d}{2}A_jY_j^2\} \notag\\
 &\times \int_{(Y_{1},Y_2,\cdots,Y_{j-1})\in \mathbb{R}^{j-1}} (2\pi)^{-\frac{j-1}{2}} \det(\boldsymbol{\Sigma})^{-\frac{1}{2}} \exp\{ -\frac{c_d}{2}[A_1Y_1^2+\cdots+A_{j-1}Y_{j-1}^2]   \} \dd \boldsymbol{Y} \notag\\
 & = C \times \frac{1}{(c_dA_j)^{1/2}}\times \frac{1}{2}[1+\text{erf}(\frac{x\sqrt{c_d A_j}}{\sqrt{2}})]  , \label{compute cdf}
\end{align}
where $C$ a constant that corresponds to the integration in the first and third line. Here $C$ does not depend on $x$. Note that $\frac{1}{2}[1+\text{erf}(\frac{x\sqrt{c_d A_j}}{\sqrt{2}})]$ is the cdf of $\mathcal{N}(0,\frac{1}{c_dA_j}),$ let $x\to \infty$, we conclude $ C \times \frac{1}{(c_dA_j)^{1/2}}=1$, thus $Y_j\sim \mathcal{N}(0,\frac{1}{c_dA_j})$. 

\end{proof}

\begin{lemma}\label{Lemma: independent}
$Y_j$ and $Y_k$ are independent for $j\neq k$, where $Y_j$ is defined in Lemma \ref{Lemma: cdf of Y}.
\end{lemma}

\begin{proof}
We prove the lemma by showing the joint cdf of $Y_j,Y_k$ can be written as the product of cdf of $Y_j$ and cdf of $Y_k$. Without loss of generality, we assume $j<k$. We focus on computing the joint cdf $P(Y_j<x,Y_k<y)$. Following the same procedure of (\ref{cdf: Y3}), we apply the change of variable (\ref{change of variable}), then the integration becomes: 
\begin{align}
 P(Y_j<x,Y_k<y)& = \int_{(Y_{k+1},Y_{k+2},\cdots,Y_{d})\in \mathbb{R}^{d-k}}   (2\pi)^{-\frac{d-k}{2}}  \exp\{-\frac{c_d}{2}[A_{k+1}Y_{k+1}^2+\cdots+A_dY_d^2]\} \notag\\
 & \times  \int_{-\infty}^{y} (2\pi)^{-\frac{1}{2}} \exp\{-\frac{c_d}{2}A_kY_k^2\} \notag\\
 &\times \int_{(Y_{j+1},Y_{j+2},\cdots,Y_{k-1})\in \mathbb{R}^{k-j-1}}   (2\pi)^{-\frac{k-j-1}{2}}  \exp\{-\frac{c_d}{2}[A_{j+1}Y_{j+1}^2+\cdots+A_{k-1}Y_{k-1}^2]\} \notag\\
 & \times \int_{-\infty}^{x} (2\pi)^{-\frac{1}{2}} \exp\{-\frac{c_d}{2}A_jY_j^2\} \notag\\
 &\times \int_{(Y_{1},Y_2,\cdots,Y_{j-1})\in \mathbb{R}^{j-1}} (2\pi)^{-\frac{j-1}{2}} \det(\boldsymbol{\Sigma})^{-\frac{1}{2}} \exp\{ -\frac{c_d}{2}[A_1Y_1^2+\cdots+A_{j-1}Y_{j-1}^2]   \} \dd \boldsymbol{Y} \notag\\
 & = C \times \frac{1}{(c_dA_j)^{1/2}}\frac{1}{2}[1+\text{erf}(\frac{x\sqrt{c_d A_j}}{\sqrt{2}})] \times \frac{1}{(c_dA_k)^{1/2}}\frac{1}{2}[1+\text{erf}(\frac{y\sqrt{c_dA_k}}{2})] . \label{compute cdf: x,y}
\end{align}
Similar to (\ref{compute cdf}), $C$ is a constant that corresponds to the first, third and fifth line, and $C$ does not depend on $x,y$. Note that $\frac{1}{2}[1+\text{erf}(\frac{x\sqrt{c_dA_j}}{2})]$ and $\frac{1}{2}[1+\text{erf}(\frac{y\sqrt{c_dA_k}}{2})]$ are the cdf of $\mathcal{N}(0,\frac{1}{c_dA_j})$ and $\mathcal{N}(0,\frac{1}{c_dA_k})$. Let $x,y\to\infty$, we conclude that the constant terms combine to be $C\times \frac{1}{c_d \sqrt{A_jA_k}}=1$. Thus the joint cdf is
\[P(Y_j<x,Y_k<y) = \frac{1}{2}[1+\text{erf}(\frac{x\sqrt{c_dA_j}}{2})]\times \frac{1}{2} [1+\text{erf}(\frac{y\sqrt{c_dA_k}}{2})].\]
This equals to $P(Y_j<x)\times P(Y_k<y)$ by directly applying Lemma \ref{Lemma: cdf of Y}. Then we conclude the lemma.

\end{proof}

\begin{lemma}\label{Lemma: e_Sigma_Z} Suppose $\boldsymbol{Z}\sim \mathcal{N}(\mathbf{0},\boldsymbol{\Sigma})$ with $\boldsymbol{\Sigma}$ given by~(\ref{covariance}), then
\begin{equation}\label{e_Sigma_Z}
\boldsymbol{e}_1^T \boldsymbol{\Sigma}^{-1}\boldsymbol{Z} \sim \mathcal{N}(0,c_d),
\end{equation}
where $c_d$ is defined in~(\ref{epsilon_d}). Here (\ref{e_Sigma_Z}) corresponds to a 1-D Gaussian distribution.

\end{lemma}

\begin{proof}
By Lemma \ref{Lemma: inverse} we have
\begin{align*}
  \boldsymbol{e}_1^T \boldsymbol{\Sigma}^{-1} & = c_d (1,-\tau_d,-\tau_d,\cdots,-\tau_d)^T .
\end{align*}
Thus
\[\boldsymbol{e}_1^T \boldsymbol{\Sigma}^{-1}\boldsymbol{Z} = c_d [Z_1-\tau_d Z_2 -\tau_d Z_3 -\cdots- \tau_d Z_d]. \]
Then the lemma follows by directly applying Lemma \ref{Lemma: cdf of Y} with $A_1=1,B_1=\tau_d$ in (\ref{recurrence}).

\end{proof}

\begin{lemma}\label{Lemma: Z sigma Z}
Suppose $\boldsymbol{Z}\sim \mathcal{N}(\mathbf{0},\boldsymbol{\Sigma})$ with $\boldsymbol{\Sigma}$ given by (\ref{covariance}), then

\begin{equation}\label{Z sigma Z}
\boldsymbol{Z}^T \boldsymbol{\Sigma}^{-1} \boldsymbol{Z}  \text{ has the same distribution as } \chi^2(d),    
\end{equation}
where $\chi^2(d)$ is the Chi-square distribution with freedom $d$.

\end{lemma}

\begin{proof}
Applying Lemma \ref{Lemma: induction on ZZ}, Lemma \ref{Lemma: cdf of Y} and Lemma \ref{Lemma: independent} we have
\begin{align*}
  \boldsymbol{Z}^T \boldsymbol{\Sigma}^{-1} \boldsymbol{Z}    &  = \sum_{i=1}^d c_d A_i Y_i^2,
\end{align*}
where $A_i$ is defined in (\ref{recurrence}), $Y_i$ is defined in Lemma \ref{Lemma: cdf of Y}. Here $Y_i,Y_j$ are independent for $i\neq j$ and $c_dA_iY_i^2= (\sqrt{c_dA_i}Y_i)^2$. Then we apply Lemma \ref{Lemma: cdf of Y} to get $(\sqrt{c_dA_i}Y_i) \sim \mathcal{N}(0,1)$ is a standard normal distribution. Then by the definition of the Chi-square distribution we conclude the lemma.

\end{proof}

\subsection{Proof of Lemma~\ref{Lemma: GenLabel loss for 2 class}}\label{sec:proof_lemma: GenLabel loss}

\begin{proof}
Denote the mixed point by $\tilde{\vx}_{ij}(\lambda) = \lambda \boldsymbol{x_i}+(1-\lambda)\vx_j$. In order to estimate the second order Taylor expansion of $L_n^{\op{gen}}(\boldsymbol{\theta},S)$, we first compute the GenLabel $y_{ij}^{\text{gen}}$. Next we use expression of $y_{ij}^{\text{gen}}$ to estimate $L_n^{\op{gen}}(\boldsymbol{\theta},S)$. Then we derive the second order Taylor expansion and the correspond coefficients $A_{\sigma_1,c,\tau,d}^i,B_{\sigma_1,c,\tau,d}^i$. Last we consider the asymptotic limit $\sigma_1\to\infty$.

\textbf{Step 1: compute $y_{ij}^{\op{gen}}$}.

Recall that when $y_i = y_j$, we set the label of mixed point as $y_{ij}^{\op{mix}} = y_i$. 
For such case, we have $y^{\op{mix}}_{ij} = \lambda_1 y_i + (1-\lambda_1)y_i = \lambda_1 y_i + (1-\lambda_1)y_j$ for any $\lambda_1 \in \mathbb{R}$. 
When $y_i\neq y_j,$ we use the suggested GenLabel $y^{\op{gen}}_{ij}$ in~(\ref{eqn:GenLabel}).
Without loss of generality, we assume $\boldsymbol{x_i} \sim \mathcal{N}(-\boldsymbol{e}_1,\frac{\boldsymbol{\Sigma}}{\sigma_1^2})$ and $\vx_j \sim \mathcal{N}(\boldsymbol{e}_1, \frac{\boldsymbol{\Sigma}}{\sigma_2^2})$.
Thus the correspond labels are $y_i=0, y_j=1$. We compute the mixed point $\tilde{x}_{ij}$ as follows:
\begin{equation}\label{xij}
\begin{split}
 \tilde{\vx}_{ij}(\lambda)  & =\lambda \boldsymbol{x_i} + (1-\lambda) \vx_j = \lambda ( -\boldsymbol{e}_1+\boldsymbol{Z}_i) + (1-\lambda)(\boldsymbol{e}_1 + \boldsymbol{Z}_j) = (1-2\lambda)\boldsymbol{e}_1 + \boldsymbol{Z}_{ij} \\
 \boldsymbol{Z}_{ij}   &= \lambda \boldsymbol{Z}_i + (1-\lambda) \boldsymbol{Z}_j \sim \mathcal{N}(\mathbf{0},\frac{\lambda^2 \boldsymbol{\Sigma}}{\sigma_1^2}+\frac{(1-\lambda)^2 \boldsymbol{\Sigma}}{\sigma_2^2}).
\end{split}
\end{equation}
where $\boldsymbol{Z}_i = \boldsymbol{x_i}+\boldsymbol{e}_1$ and $\boldsymbol{Z}_j = \vx_j - \ve_1$. Now we compute the GenLabel $y_{ij}^{\op{gen}}$ and express it as a convex combination of $y_i$ and $y_j$.
To compute the $y^{\op{gen}}_{ij}$, we denote the density function of $\mathcal{N}(-\boldsymbol{e}_1,\frac{\boldsymbol{\Sigma}}{\sigma_1^2})$ as
\begin{equation}\label{pdf: N(0)}
p(\vx) =  (2\pi)^{-\frac{d}{2}} \det(\boldsymbol{\Sigma})^{-\frac{1}{2}} \sigma_1^{d} e^{-\frac{\sigma_1^2}{2} (\vx+\boldsymbol{e}_1)^T \boldsymbol{\Sigma}^{-1} (\vx+\boldsymbol{e}_1)},
\end{equation}
and we denote the density function of $\mathcal{N}(\boldsymbol{e}_1,\frac{\boldsymbol{\Sigma}}{\sigma_2^2})$ as
\begin{equation}\label{pdf: N(e1)}
q(\vx) =  (2\pi)^{-\frac{d}{2}} \det(\boldsymbol{\Sigma})^{-\frac{1}{2}} \sigma_2^{d} e^{-\frac{\sigma_2^2}{2} (\vx-\boldsymbol{e}_1)^T \boldsymbol{\Sigma}^{-1} (\vx-\boldsymbol{e}_1)}.
\end{equation}
Then the GenLabel $y_{ij}^{\op{gen}}$ in (\ref{eqn:GenLabel}) is given by the ratio:
\begin{align*}
 y_{ij}^{\op{gen}}  & =    \frac{q(\tilde{\vx}_{ij}(\lambda))}{p(\tilde{\vx}_{ij}(\lambda))+q(\tilde{\vx}_{ij}(\lambda))}=\frac{1}{1+\frac{p(\tilde{\vx}_{ij}(\lambda))}{q(\tilde{\vx}_{ij}(\lambda))}}   \\
   & =\frac{1}{1+\frac{\sigma_1^d}{\sigma_2^d}\exp\{-\frac{\sigma_1^2}{2}[(\tilde{\vx}_{ij}(\lambda)+\boldsymbol{e}_1)^T \boldsymbol{\Sigma}^{-1}(\tilde{\vx}_{ij}(\lambda)+\boldsymbol{e}_1)- \frac{
  \sigma_2^2}{\sigma_1^2}(\tilde{\vx}_{ij}(\lambda)-\boldsymbol{e}_1)^T \boldsymbol{\Sigma}^{-1}(\tilde{\vx}_{ij}(\lambda)-\boldsymbol{e}_1)]\}} .
\end{align*}
We use $\tilde{\vx}_{ij}(\lambda)   = (1-2\lambda)\boldsymbol{e}_1 + \boldsymbol{Z}_{ij}$ in (\ref{xij}) to express the exponential term in the denominator as
\begin{align*}
   & \exp\{-\frac{\sigma_1^2}{2}[(2-2\lambda)^2 \boldsymbol{e}_1^T \boldsymbol{\Sigma}^{-1} \boldsymbol{e}_1 + 4(1-\lambda) \boldsymbol{e}_1^T \boldsymbol{\Sigma}^{-1} \boldsymbol{Z}_{ij} + \boldsymbol{Z}_{ij}^T \boldsymbol{\Sigma}^{-1} \boldsymbol{Z}_{ij}]\}\\
   & \times \exp\{\frac{\sigma_1^2}{2}[ 4\lambda^2 \boldsymbol{e}_1^T \boldsymbol{\Sigma}^{-1} \boldsymbol{e}_1 - 4\lambda \boldsymbol{e}_1^T \boldsymbol{\Sigma}^{-1}\boldsymbol{Z}_{ij} + \boldsymbol{Z}_{ij}^T \boldsymbol{\Sigma}^{-1} \boldsymbol{Z}_{ij}             ]\} \\
   &\times \exp\{\frac{\sigma_2^2-\sigma_1^2}{2}[ 4\lambda^2 \boldsymbol{e}_1^T \boldsymbol{\Sigma}^{-1} \boldsymbol{e}_1 - 4\lambda \boldsymbol{e}_1^T \boldsymbol{\Sigma}^{-1}\boldsymbol{Z}_{ij} + \boldsymbol{Z}_{ij}^T \boldsymbol{\Sigma}^{-1} \boldsymbol{Z}_{ij}             ]\}         \\
   & = \exp\{-\sigma_1^2[(2-4\lambda)\boldsymbol{e}_1^T \boldsymbol{\Sigma}^{-1}\boldsymbol{e}_1 +  2\boldsymbol{e}_1^T \boldsymbol{\Sigma}^{-1}\boldsymbol{Z}_{ij}    ]\} \exp\{(\sigma_2^2-\sigma_1^2)[2\lambda^2 \boldsymbol{e}_1^T \boldsymbol{\Sigma}^{-1} \boldsymbol{e}_1-2\lambda \boldsymbol{e}_1^T \boldsymbol{\Sigma}^{-1} \boldsymbol{Z}_{ij} + \frac{\boldsymbol{Z}_{ij}^T\boldsymbol{\Sigma}^{-1}\boldsymbol{Z}_{ij}}{2}  ]\}.
\end{align*}
Now we apply previous lemmas to estimate all terms in the exponent.

For $\boldsymbol{e}_1^T \boldsymbol{\Sigma}^{-1}\boldsymbol{e}_1$, we apply Lemma \ref{Lemma: inverse} to $\boldsymbol{\Sigma}^{-1}$ and conclude
\[(2-4\lambda)\boldsymbol{e}_1^T \boldsymbol{\Sigma}^{-1}\boldsymbol{e}_1 = (2-4\lambda)c_d,\]
where $c_d$ is defined in~(\ref{epsilon_d}). 

For the other two terms, we define $Z_{ij}':=\boldsymbol{e}_1^T \boldsymbol{\Sigma}^{-1} \boldsymbol{Z}_{ij}$ and $\bar{z}_{ij}:= \boldsymbol{Z}_{ij}^T\boldsymbol{\Sigma}^{-1} \boldsymbol{Z}_{ij}$. From (\ref{xij}), $\boldsymbol{Z}_{ij} = \sqrt{\frac{\lambda^2}{\sigma_1^2}+\frac{(1-\lambda)^2}{\sigma_2^2}}\boldsymbol{Z} \sim \mathcal{N}(\mathbf{0},[\frac{\lambda^2}{\sigma_1^2}+\frac{(1-\lambda)^2}{\sigma_2^2}]\boldsymbol{\Sigma})$, with $\boldsymbol{Z}\sim \mathcal{N}(\mathbf{0},\boldsymbol{\Sigma})$. Then we apply Lemma \ref{Lemma: e_Sigma_Z} to $Z'_{ij}$ and have
\begin{equation}\label{Z'}
Z_{ij}'=\boldsymbol{e}_1^T \boldsymbol{\Sigma}^{-1}\boldsymbol{Z}_{ij}\sim \mathcal{N}(0,[\frac{\lambda^2}{\sigma_1^2}+\frac{(1-\lambda)^2}{\sigma_2^2}]c_d).
\end{equation}
For $\bar{z}_{ij}$, we apply Lemma \ref{Lemma: Z sigma Z} and have
\begin{equation}\label{bar Z}
\bar{z}_{ij} = \boldsymbol{Z}_{ij}^T \boldsymbol{\Sigma}^{-1}\boldsymbol{Z}_{ij} = [\frac{\lambda^2}{\sigma_1^2}+\frac{(1-\lambda)^2}{\sigma_2^2}]\boldsymbol{Z}^T \boldsymbol{\Sigma}^{-1}\boldsymbol{Z} \sim [\frac{\lambda^2}{\sigma_1^2}+\frac{(1-\lambda)^2}{\sigma_2^2}] \chi^2(d)    
\end{equation}
where $\chi^2(d)$ is the Chi-square distribution with freedom $d$. Thus we conclude that the GenLabel reads
\begin{equation}\label{y_gen}
y_{ij}^{\op{gen}} = \frac{1}{1+\frac{\sigma_1^d}{\sigma_2^d}\exp\{-\sigma_1^2[(2-4\lambda)c_d +2 Z_{ij}' ]\}\exp\{(\sigma^2_2-\sigma^2_1)[2\lambda^2 c_d-2\lambda Z_{ij}'+\frac{\bar{z}_{ij}}{2}]\}}.
\end{equation}
In other words, $y_{ij}^{\op{gen}}$ can be written as a convex combination of $y_i$ and $y_j$ as follows:
\begin{equation}\label{lambda1}
\begin{split}
 y_{ij}^{\op{gen}}   & = \lambda_1 y_i + (1-\lambda_1)y_j = 1-\lambda_1 = (\ref{y_gen}), \\
\lambda_1    & = 1-(\ref{y_gen}) = \frac{1}{1+\frac{\sigma_2^d}{\sigma_1^d}\exp\{\sigma_1^2[(2-4\lambda)c_d +2 Z_{ij}' ]\}\exp\{(\sigma^2_1-\sigma^2_2)[2\lambda^2 c_d-2\lambda Z_{ij}'+\frac{\bar{z}_{ij}}{2}]\}}.
\end{split}
\end{equation}

\textbf{Step 2: estimate $L_n^{\op{gen}}(\boldsymbol{\theta},S)$.}

Now we plug the expression of $y_{ij}^{\op{gen}}$ (\ref{lambda1}) into the GenLabel loss, we have
\begin{align}
 L_n^{\op{gen}}(\boldsymbol{\theta},S)  & = \frac{1}{n^2}\mathbb{E}_{\lambda \sim \text{Unif}([0,1])}   \sum_{i,j=1}^n   [h(f_{\boldsymbol{\theta}}(\tilde{\vx}_{ij}(\lambda)))-(\lambda_1 y_i+(1-\lambda_1)y_j)]f_{\boldsymbol{\theta}}(\tilde{\vx}_{ij}(\lambda)) \notag \\
   & =\frac{1}{n^2}\mathbb{E}_{\lambda \sim \text{Unif}([0,1])}  \sum_{i,j=1}^n \Bigg\{\mathbb{E}_{B\sim \text{Bern}(\lambda_1)}\big[ B[h(f_{\boldsymbol{\theta}}(\tilde{\vx}_{ij}))-y_i f_{\boldsymbol{\theta}}(\tilde{\vx}_{ij})]         \notag \\
   & +(1-B)[h(f_{\boldsymbol{\theta}}(\tilde{\vx}_{ij}))-y_j f_{\boldsymbol{\theta}}(\tilde{\vx}_{ij})]\big] \Bigg\}. \label{d finite}
\end{align}

For $\lambda \sim \text{Unif}([0,1])$, $B\vert \lambda  \sim \text{Bern}(\lambda_1)$, we can exchange them in order and have
\[B\sim \text{Bern($a_{ij}$)}, \quad \lambda | B \sim \left\{
                                             \begin{array}{ll}
                                               \mathcal{F}^1_{ij}, & \hbox{$B=1$;} \\
                                               \mathcal{F}^2_{ij}, & \hbox{$B=0$.}
                                             \end{array}
                                           \right.
\]
\[a_{ij}= \int_0^1 \lambda_1 \dd \lambda = \int_0^1 \frac{1}{1+\frac{\sigma_2^d}{\sigma_1^d}\exp\{\sigma_1^2[(2-4\lambda)c_d +2 Z_{ij}' ]\}\exp\{(\sigma^2_1-\sigma^2_2)[2\lambda^2 c_d-2\lambda Z_{ij}'+\frac{\bar{z}_{ij}}{2}]\}} \dd \lambda .        \]
$\mathcal{F}^1_{ij}$ has density function
\[ \mathcal{F}^1_{ij} \sim \frac{\lambda_1}{a_{ij}} = \frac{1}{a_{ij}} \frac{1}{1+\frac{\sigma_2^d}{\sigma_1^d}\exp\{\sigma_1^2[(2-4\lambda)c_d +2 Z_{ij}' ]\}\exp\{(\sigma^2_1-\sigma^2_2)[2\lambda^2 c_d-2\lambda Z_{ij}'+\frac{\bar{z}_{ij}}{2}]\}}.\]
$\mathcal{F}^2_{ij}$ has density function
\[\mathcal{F}^2_{ij} \sim \frac{1-\lambda_1}{1-a_{ij}} = \frac{1}{1-a_{ij}} \frac{1}{1+\frac{\sigma_1^d}{\sigma_2^d}\exp\{-\sigma_1^2[(2-4\lambda)c_d +2 Z_{ij}' ]\}\exp\{(\sigma^2_2-\sigma^2_1)[2\lambda^2 c_d-2\lambda Z_{ij}'+\frac{\bar{z}_{ij}}{2}]\}}.  \]

After changing the order of $\lambda$ and $B$ in (\ref{d finite}), we get 
\begin{align}
  (\ref{d finite}) & = \frac{1}{n^2}\bigg\{ \sum_{i,j=1}^n a_{ij}\mathbb{E}_{\lambda \sim \mathcal{F}^1_{ij}}[h(f_{\boldsymbol{\theta}}(\tilde{\vx}_{ij}(\lambda)))-y_i f_{\boldsymbol{\theta}}(\tilde{\vx}_{ij}(\lambda))] \label{d finite i}\\
   & +\sum_{i,j=1}^n (1-a_{ij})\mathbb{E}_{\lambda \sim \mathcal{F}^2_{ij}}[h(f_{\boldsymbol{\theta}}(\tilde{\vx}_{ij}(\lambda)))-y_j f_{\boldsymbol{\theta}}(\tilde{\vx}_{ij}(\lambda))]\bigg \} \label{d finite j}.
\end{align}

Since $\tilde{\vx}_{ij}(\lambda) = \tilde{\vx}_{ji}(1-\lambda)$, we can rewrite (\ref{d finite j}) as
\begin{equation}\label{d finite j: variant}
(\ref{d finite j})= \frac{1}{n^2} \sum_{i,j=1}^n  (1-a_{ij})\mathbb{E}_{\lambda   \sim \mathcal{F}^3_{ij}}  [h(f_{\boldsymbol{\theta}}(\tilde{\vx}_{ij}(\lambda)))-y_i f_{\boldsymbol{\theta}}(\tilde{\vx}_{ij}(\lambda))].     
\end{equation}
Here $\mathcal{F}^3_{ij}$ has density function $\mathcal{F}^{2}_{ij}(1-\lambda)$:
\[\mathcal{F}^3_{ij} \sim \frac{1}{1-a_{ij}} \frac{1}{1+\frac{\sigma_1^d}{\sigma_2^d}\exp\{-\sigma_1^2[(4\lambda-2)c_d +2 Z_{ij}' ]\}\exp\{(\sigma^2_2-\sigma^2_1)[2(1-\lambda)^2c_d-2(1-\lambda) Z_{ij}'+\frac{\bar{z}_{ij}}{2}]\}}.\]
From (\ref{d finite i}) and (\ref{d finite j: variant}), we denote $\mathcal{F}_{ij}$ as a mixture distribution:
\begin{equation}\label{Fij}
 \mathcal{F}_{ij} = a_{ij} \mathcal{F}_{ij}^1 + (1-a_{ij})\mathcal{F}_{ij}^3.
\end{equation}
Then (\ref{d finite}) reads
\begin{align}
 (\ref{d finite}) &  =\frac{1}{n^2}\sum_{i,j=1}^n \mathbb{E}_{\lambda \sim \mathcal{F}_{ij}}[h(f_{\boldsymbol{\theta}}(\tilde{\vx}_{ij}(\lambda)))-y_i f_{\boldsymbol{\theta}}(\tilde{\vx}_{ij}(\lambda))]\notag\\
  & =\frac{1}{n}\sum_{i=1}^n \mathbb{E}_{\lambda \sim \mathcal{F}_{i}} \mathbb{E}_{\boldsymbol{r_x} \sim D_X} \ell_{\check{\vx}_i,y_i}(\boldsymbol{\theta}). \label{step 2 conclusion.}
\end{align}
Here we defined $\mathcal{F}_i$ as a mixture distribution:
\begin{equation}\label{Fi}
 \mathcal{F}_{i} = \frac{1}{n}\sum_{j=1}^n \mathcal{F}_{ij}.
\end{equation}
$D_X$ is the empirical distribution induced by training samples and $\check{\vx}_i=\lambda \boldsymbol{x_i}+(1-\lambda)\boldsymbol{r_x}$.

\textbf{Step 3: derive the second order Taylor expansion}.

Given the expression of $L_n^{\op{gen}}(\boldsymbol{\theta},S)$ in (\ref{step 2 conclusion.}), we follow the proof of Lemma \ref{Lemma: mixup loss} and conclude that the second order Taylor expansion is given by Lemma \ref{Lemma: GenLabel loss for 2 class}, with the coefficients $A_{\sigma_1,c,\tau,d}^i,B_{\sigma_1,c,\tau,d}^i$ given by
\begin{equation}\label{A,B}
A_{\sigma_1,c,\tau,d}^i = \mathbb{E}_{\lambda \sim \mathcal{F}_i}[1-\lambda],\quad  B_{\sigma_1,c,\tau,d}^i = \mathbb{E}_{\lambda \sim \mathcal{F}_i}[(1-\lambda)^2].    
\end{equation}
Here $\mathcal{F}_i$ has density function
\begin{align*}
\frac{1}{n}\sum_{j=1}^n &\{ \frac{1}{1+\frac{\sigma_2^d}{\sigma_1^d}\exp\{\sigma_1^2[(2-4\lambda)c_d + 2Z_{ij}' ]\}\exp\{(\sigma^2_1-\sigma^2_2)[2\lambda^2c_d-2\lambda Z_{ij}'+\frac{\bar{z}_{ij}}{2}]\}} \\
& + \frac{1}{1+\frac{\sigma_1^d}{\sigma_2^d}\exp\{-\sigma_1^2[(4\lambda-2)c_d + 2Z_{ij}' ]\}\exp\{(\sigma^2_2-\sigma^2_1)[2(1-\lambda)^2c_d-2(1-\lambda) Z_{ij}'+\frac{\bar{z}_{ij}}{2}]\}} \},    
\end{align*}
where $Z'_{ij}$, $\bar{z}_{ij}$ and $c_d$ are defined in (\ref{Z'}), (\ref{bar Z}) and (\ref{c_d}) respectively.

It remains to prove that when $\sigma_1\to \infty$, these coefficients satisfy the properties mentioned in Lemma \ref{Lemma: GenLabel loss for 2 class}.

\textbf{Step 4: asymptotic analysis for $\sigma_1\to \infty$ }

Now we prove $\lim_{\sigma_1\to \infty}A_{\sigma_1,c,\tau,d}^i = \frac{c^2+1}{2(c+1)^2}$ and $\lim_{\sigma_1\to \infty}B_{\sigma_1,c,\tau,d}^i = \frac{c^2-c+1}{3(1+c)^2}.$
Recall that GenLabel $y_{ij}^{\op{gen}}$ is 
given in (\ref{lambda1}). When $\sigma_1\to \infty$, we have $Z_{ij}',\bar{z}_{ij} \to 0$, then $\lambda_1$ in (\ref{lambda1}) becomes
\begin{align*}
\lambda_1    & = \frac{1}{1+c^d\exp\{\sigma_1^2[(2-4\lambda)]c_d+2(\sigma_1^2-c^2 \sigma_1^2 )\lambda^2 c_d\} }  \\
    &  = \frac{1}{1+c^d \exp\{2\sigma_1^2 c_d[(1-c^2)\lambda^2-2\lambda+1]\}}\\
    & =\left\{
    \begin{array}{ll}
      \frac{1}{1+c^d \exp\{2\sigma_1^2 c_d (1-c^2) (\lambda-\frac{1}{1-c})(\lambda-\frac{1}{1+c})\}}, & \hbox{$c\neq 1$;} \\
      \frac{1}{1+c^d \exp\{2\sigma_1^2 c_d(1-2\lambda)\}}, & \hbox{$c=1$.}
    \end{array}
  \right.
\end{align*}

We have three cases regarding $c$.

If $c>1$, then $\frac{1}{1-c}<0$, $1-c^2<0$, which implies
\[(1-c^2)(\lambda-\frac{1}{1-c})(\lambda-\frac{1}{1+c}) \left\{
    \begin{array}{ll}
      >0, & \hbox{$\frac{1}{1-c}<0\leq \lambda < \frac{1}{1+c}$;} \\
      <0, & \hbox{$\frac{1}{1+c}< \lambda \leq 1$.}
    \end{array}
  \right.\]
If $0 < c <1$, then $\frac{1}{1-c}>1$, $1-c^2>0$, which implies  
\[(1-c^2)(\lambda-\frac{1}{1-c})(\lambda-\frac{1}{1+c}) \left\{
    \begin{array}{ll}
      >0, & \hbox{$0\leq \lambda < \frac{1}{1+c}$;} \\
      <0, & \hbox{$\frac{1}{1+c}< \lambda \leq 1<\frac{1}{1-c}$.}
    \end{array}
  \right.\]
If $c=1$, we have $1-2\lambda>0$ for $0\leq \lambda<\frac{1}{2}$ and $1-2\lambda<0$ for $\frac{1}{2}<\lambda \leq 1$. 

When $\sigma_1 \to \infty$, we combine all three cases above and conclude 
\begin{equation*}
\lambda_1 = \left\{
              \begin{array}{ll}
                0, & \hbox{$0\leq \lambda < \frac{1}{1+c}$;} \\
                1, & \hbox{$\frac{1}{1+c}<\lambda \leq 1$.}
              \end{array}
            \right. \quad y_{ij}^{\op{gen}} =  \left\{
              \begin{array}{ll}
                y_j, & \hbox{$0\leq \lambda < \frac{1}{1+c}$;} \\
                y_i, & \hbox{$\frac{1}{1+c}<\lambda \leq 1$.}
              \end{array}
            \right.
\end{equation*}
With the GenLabel given by the above equation, we compute the GenLabel loss as
\begin{align}
 L_n^{\op{gen}}(\boldsymbol{\theta},S)& =\frac{1}{n^2}\mathbb{E}_{\lambda \sim \text{Unif}([0,1])}\sum_{i,j=1}^n[h(f_{\boldsymbol{\theta}}(\tilde{\vx}_{ij}(\lambda)))-y_{ij}^{\text{gen}}f_{\boldsymbol{\theta}}(\tilde{\vx}_{ij}(\lambda))] \notag\\
   & =\frac{1}{(c+1)n^2}\mathbb{E}_{\lambda \sim \text{Unif}([0,1/(1+c)])} \sum_{i,j=1}^{n} \bigg\{h(f_{\boldsymbol{\theta}}(\tilde{\vx}_{ij}(\lambda)))-y_j f_{\boldsymbol{\theta}}(\tilde{\vx}_{ij}(\lambda)) \bigg\} \label{sum in yj}\\
& + \frac{c}{(c+1)n^2} \mathbb{E}_{\lambda \sim \text{Unif}([1/(1+c),1])} \sum_{i,j=1}^{n} \bigg\{h(f_{\boldsymbol{\theta}}(\tilde{\vx}_{ij}(\lambda)))-y_i f_{\boldsymbol{\theta}}(\tilde{\vx}_{ij}(\lambda)) \bigg\}.\notag
\end{align}
Since $1-\text{Unif}([0,1/(1+c)])$ and $\text{Unif}([1-1/(1+c),1])$ are of the same distribution and $\tilde{\vx}_{ij}(1-\lambda)=\tilde{\vx}_{ji}(\lambda)$, we have
\begin{align*}
  (\ref{sum in yj}) & = \frac{1}{(c+1)n^2} \mathbb{E}_{\lambda \sim \text{Unif}([c/(c+1),1])} \sum_{i,j=1}^{n} \bigg\{h(f_{\boldsymbol{\theta}}(\tilde{\vx}_{ij}(\lambda)))-y_i f_{\boldsymbol{\theta}}(\tilde{\vx}_{ij}(\lambda)) \bigg\}.
\end{align*}
Using the above equation, the GenLabel loss reads
\begin{align*}
 L_n^{\op{gen}}(\boldsymbol{\theta},S) &=\frac{1}{n^2} \mathbb{E}_{\lambda \sim  \frac{c}{c+1}\text{Unif}([1/(c+1),1]) + \frac{1}{c+1}\text{Unif}([c/(c+1),1])} \sum_{i,j=1}^{n} \bigg\{h(f_{\boldsymbol{\theta}}(\tilde{\vx}_{ij}(\lambda)))-y_i f_{\boldsymbol{\theta}}(\tilde{\vx}_{ij}(\lambda)) \bigg\}  \\
   & =\frac{1}{n}\sum_{i=1}^n \mathbb{E}_{\lambda \sim \frac{c}{c+1}\text{Unif}([1/(c+1),1]) + \frac{1}{c+1}\text{Unif}([c/(c+1),1])} \mathbb{E}_{\boldsymbol{r_x} \sim D_X} \ell_{\check{\vx}_i,y_i}(\boldsymbol{\theta}).
\end{align*}
Following the proof of Lemma \ref{Lemma: mixup loss}, we conclude that when $\sigma_1 \to \infty$, the coefficients $A_{\sigma_1,c,\tau,d}^i,B^i_{\sigma_1,c,\tau,d}$ are given by
\begin{align*}
 \lim_{\sigma_1 \to \infty} A_{\sigma_1,c,\tau,d}^i =  &\mathbb{E}_{\lambda \sim \frac{c}{c+1}\text{Unif}([1/(c+1),1]) + \frac{1}{c+1}\text{Unif}([c/(c+1),1])}[1-\lambda] \\
    &= \frac{c}{c+1}\int_{\frac{1}{c+1}}^1 \frac{c+1}{c} (1-\lambda)   \dd \lambda + \frac{1}{c+1}\int_{\frac{c}{c+1}}^1 \frac{c+1}{1} (1-\lambda)  \dd \lambda \\
   & = \int_{\frac{1}{c+1}}^1 (1-\lambda)  \dd \lambda + \int_{\frac{c}{c+1}}^1  (1-\lambda)  \dd \lambda = \frac{c^2}{2(c+1)^2}+ \frac{1}{2(c+1)^2} = \frac{c^2+1}{2(c+1)^2}.
\end{align*}

\begin{align*}
  \lim_{\sigma_1 \to \infty} B_{\sigma_1,c,\tau,d}^i =   &\mathbb{E}_{\lambda \sim \frac{c}{c+1}\text{Unif}([1/(c+1),1]) + \frac{1}{c+1}\text{Unif}([c/(c+1),1])}[(1-\lambda)^2] \\
    &= \int_{\frac{1}{c+1}}^1 (1-\lambda)^2  \dd \lambda + \int_{\frac{c}{c+1}}^1  (1-\lambda)^2  \dd \lambda =  \frac{c^3}{3(c+1)^3} + \frac{1}{3(c+1)^3} = \frac{c^2-c+1}{3(c+1)^2}.
\end{align*}

From direct computation, we conclude that when $2-\sqrt{3}<c<2+\sqrt{3}$,
\[ \frac{c^2+1}{2(c+1)^2} < \frac{1}{3}, \quad \frac{c^2-c+1}{3(c+1)^2} < \frac{1}{6}  \Longleftrightarrow  c^2-4c+1<0  .    \]
We conclude the lemma.
\end{proof}

\section{Mathematical results in~\citep{zhang2021does}}

\begin{lemma}[Lemma 3 of~\citep{zhang2021does}]
\label{Lemma: mixup loss}
The second order Taylor approximation of the mixup loss $L_n^{\op{mix}}(\boldsymbol{\theta},S)$ is given by
\begin{equation*}%
\tilde{L}_n^{\op{mix}}(\boldsymbol{\theta},S) = L_n^{\op{std}}(\boldsymbol{\theta},S)+ {R}_1^{\op{mix}}(\boldsymbol{\theta},S) + {R}_2^{\op{mix}}(\boldsymbol{\theta},S) + {R}_3^{\op{mix}}(\boldsymbol{\theta},S),
\end{equation*}
where
\begin{align*}
   & {R}_1^{\op{mix}}(\boldsymbol{\theta},S)=\frac{1}{n} \sum_{i=1}^n \frac{1}{3} (h'(f_{\boldsymbol{\theta}}(\boldsymbol{x_i}))-y_i) \nabla f_{\boldsymbol{\theta}}(\boldsymbol{x_i})^T \mathbb{E}_{\boldsymbol{r_x}\sim D_X}[\boldsymbol{r_x}-\boldsymbol{x_i}], \\
   & {R}_2^{\op{mix}}(\boldsymbol{\theta},S)=\frac{1}{2n}\sum_{i=1}^n \frac{1}{6} h''(f_{\boldsymbol{\theta}}(\boldsymbol{x_i}))\nabla f_{\boldsymbol{\theta}}(\boldsymbol{x_i})^T \mathbb{E}_{\boldsymbol{r_x}\sim D_X}[(\boldsymbol{r_x}-\boldsymbol{x_i})(\boldsymbol{r_x}-\boldsymbol{x_i})^T]\nabla f_{\boldsymbol{\theta}}(\boldsymbol{x_i}), \\
   & {R}_3^{\op{mix}}(\boldsymbol{\theta},S)=\frac{1}{2n}\sum_{i=1}^n \frac{1}{6} (h'(f_{\boldsymbol{\theta}}(\boldsymbol{x_i}))-y_i) \mathbb{E}_{\boldsymbol{r_x}\sim D_X}[(\boldsymbol{r_x}-\boldsymbol{x_i})^T\nabla^2f_{\boldsymbol{\theta}}(\boldsymbol{x_i})(\boldsymbol{r_x}-\boldsymbol{x_i})].
\end{align*}
\end{lemma}

\begin{lemma}[Lemma 3.2 of~\citep{zhang2021does}]\label{Lemma:adv_loss_approx}
Consider the logistic regression model having $f_{\boldsymbol{\theta}}(\vx)=\boldsymbol{\theta}^T \vx$. 
The second order Taylor approximation of $L_{n}^{\op{adv}} (\boldsymbol{\theta}, S)$ is $\frac{1}{n} \sum_{i=1}^n \tilde{\ell}_{\op{adv}}(\eps \sqrt{d}, (\boldsymbol{x_i}, y_i))$, where for any $\eta > 0, \vx \in \mathbb{R}^d$ and $y \in \{0, 1\}$,
\begin{equation*}
\tilde{\ell}_{\op{adv}}(\eta,(\vx, y))=\ell(\boldsymbol{\theta},(\vx, y))+\eta\left|g\left(\boldsymbol{\theta}^T \vx\right)-y\right| \cdot\|\boldsymbol{\theta}\|_{2}+\frac{\eta^{2}}{2} \cdot g\left(\boldsymbol{\theta}^T \vx\right)\left(1-g\left(\boldsymbol{\theta}^T \vx\right)\right) \cdot\|\boldsymbol{\theta}\|_{2}^{2}
\end{equation*}
and $g(x) = e^x / (1+e^x)$ is the logistic function.
\end{lemma}

\section{Detailed experiments setup}\label{sec:exp_setup}

Here we provide a detailed description on our experimental settings.

\subsection{Synthetic datasets}

\paragraph{Datasets} The 2D cube dataset with 2 classes (class 0 and 1) is defined as follows. Consider two adjacent squares centered at $\bm{\mu}_{0} = (-1, 0)$ and $\bm{\mu}_{1} = (1, 0)$, respectively, where the length of each side of each square is 2. We define the support of class $i$ as the area of each square. In other words, the support of class 0 is $X_0 = \{ \vx \in \mathbb{R}^2 : \lVert \bm{\mu}_0 -  \vx  \rVert_{\infty} \leq 1 \} $ where $\lVert \cdot \rVert_{\infty}$ is the $L_{\infty}$ norm operator. Similarly, the support of class 1 is $X_1 = \{ \vx \in \mathbb{R}^2 : \lVert \bm{\mu}_1 -  \vx  \rVert_{\infty} \leq 1 \} $.
The data point $\vx$ for class $i \in \{0,1\}$ is uniform-randomly sampled from the square $X_i$.

The 3D cube dataset with 8 classes is defined as below. Consider 8 adjacent cubes, each of which is located at each octant, where the center of each cube is $\bm{\mu} = (\mu^{(1)}, \mu^{(2)}, \mu^{(3)})$ for $\mu^{(1)}, \mu^{(2)}, \mu^{(3)} \in \{ -1, 1 \}$ and the length of each side of each cube is 2.
We define the support of class $i$ as the volume of each cube. For example, the class 0 corresponds to the cube centered at $\bm{\mu}_0 = (-1, -1, -1)$, and the support of class 0 is
$X_0 = \{ \vx \in \mathbb{R}^3 : \lVert \bm{\mu}_0 -  \vx  \rVert_{\infty} \leq 1 \} $. Similarly, we define the support of class $i \in \{0, 1, \cdots, 7\}$.
The data point $\vx$ for class $i$ is uniform-randomly sampled from the cube $X_i$.

The 9-class Gaussian dataset used in Fig.~\ref{Fig:intrusion_fix_margin} is defined as follows. 
We generate 9 Gaussian clusters having the covariance matrix of $\bm{\Sigma} = \frac{1}{10} \mI_2$ and centered at $\bm{\mu} = (\mu^{(1)}, \mu^{(2)})$ for $\mu^{(1)}, \mu^{(2)} \in \{ -10, 0, 10 \}$. For example, cluster 0 (or class 0) is centered at  $\bm{\mu}_0 = (-10, -10)$ and cluster 8 (or class 8) is centered at $\bm{\mu}_0 = (10, 10)$. 

The Circle and Moon datasets used in Table~\ref{Table:mixup_fail}
are from scikit-learn~\citep{scikit-learn} combined with Laplacian noise, where the exponential decay $\lambda$ of Laplacian noise is set to 0.1 for Moon and 0.02 for Circle.

The Two-circle dataset used in Table~\ref{Table:mixup_fail} is generated as follows.
We first generate a Circle dataset from scikit-learn~\citep{scikit-learn} combined with Laplacian noise, where the exponential decay $\lambda$ of Laplacian noise is set to 0.01. Then, we generate another (second) Circle dataset under the same setting (but having different realization), shift it to the right, and flip the label of the second Circle dataset. In this way, we get two adjacent Circle datasets with flipped label.

\paragraph{Training setting} For synthetic datasets, the hyperparameters used in our experiments are summarized in Table~\ref{Table:syn_cln_settings}. For both 2D and 3D cube datasets, we randomly generate 20 data samples from uniform distribution for each class as training data, and evaluate the decision boundary by another 10000 randomly generated data samples for each class. For 9-class Gaussian dataset, each cluster has 5000 randomly generated samples as the training data. For Moon and Circle datasets, we randomly generate 1000 data samples for both training and testing. For Two-circle dataset, we randomly generate 1000 data samples for each Circle dataset for both training and testing.
For 2D and 3D cube datasets, we use a 3-layer fully connected network, which has 64 neurons in the first hidden layer and 128 neurons in the second hidden layer. For Moon, Circle and Two-circle datasets, we use a 4-layer fully connected network, which has 64 neurons in the first hidden layer and 128 neurons in the remaining hidden layers. For all the datasets, we use the SGD optimizer and the multi-step learning rate decay. We measure the clean validation accuracy at each epoch and choose the best model having the highest clean accuracy.

\paragraph{Algorithms} For mixup~\citep{zhang2017mixup}, we followed the code from the official github repository: \url{https://github.com/facebookresearch/mixup-cifar10}. 
For our GenLabel scheme on %
9-class Gaussian datasets, we use the ground-truth mean and identity covariance to estimate the Gaussian mixture (GM) models at the input layer.

\subsection{Real datasets}

\paragraph{Datasets} We use OpenML datasets from~\openml, MNIST, CIFAR-10 and CIFAR-100 datasets from PyTorch~\citep{pytorch}, and Tiny-Imagenet-200 dataset from \url{http://cs231n.stanford.edu/tiny-imagenet-200.zip}.

For experiments on OpenML datasets, we first accessed all datasets from Python OpenML API~\citep{feurer-arxiv19a}.
Afterwards, we filtered out the datasets having more than 20 features, datasets with more than 5000 data samples. We tested our GenLabel on the remaining datasets.

\paragraph{Training setting} 

The hyperparameters used in our experiments are summarized in Table~\ref{Table:tabular_cln_settings},~\ref{Table:tabular_rob_settings}, \ref{Table:real_cln_settings} and~\ref{Table:real_rob_settings}. 
When we train mixup+GenLabel on OpenML datasets, we used a 6-fold cross-validation
for choosing the best loss ratio $\gamma \in \{0.0, 0.2, 0.4, 0.6, 0.8, 1.0 \}$. 
For the clean validation runs, we measured the clean validation accuracy at each epoch and choose the best model having the highest clean accuracy. For the robust validation runs, we measured the robust validation accuracy at every 5 epochs and choose the best model having the highest robust accuracy.
For OpenML datasets, we tested training methods on both the logistic regression model and the neural network with 2 hidden layers. For the latter, we followed the same architecture used in mixup~\citep{zhang2017mixup} which has 128 neurons in each hidden layer. For MNIST and CIFAR-10 datasets, we used LeNet-5 and ResNet-18, respectively. For both CIFAR-100 and Tiny-Imagenet-200 datasets, we used PreActResNet-18.
We tested on NVIDIA Tesla V100 GPUs in Amazon Web Service (AWS) and local NVIDIA RTX2080 GPU machines.

\paragraph{Algorithms} For mixup~\citep{zhang2017mixup} and manifold-mixup~\citep{pmlr-v97-verma19a}, we followed the code from the official github repository: \url{https://github.com/facebookresearch/mixup-cifar10} and \url{https://github.com/vikasverma1077/manifold_mixup}. Note that the mixup github repository contains license: see \url{https://github.com/facebookresearch/mixup-cifar10/blob/master/LICENSE}.
For AdaMixUp~\citep{guo2019mixup}, we cloned the source code in \url{https://github.com/SITE5039/AdaMixUp} for MNIST and CIFAR-10 implemented in TensorFlow~\citep{tensorflow2015-whitepaper}, and made slight modifications to make their experimental settings and models consistent with ours. For our GenLabel schemes, we estimated and updated the Gaussian mixture (GM) models at the penultimate layer. 

\begin{table}[t]
	\centering
	\vspace{-2mm}
	\caption{Hyperparameters and models used for clean validation in synthetic dataset experiments.}
	\scriptsize
	\label{Table:syn_cln_settings}
	\begingroup
    \setlength{\tabcolsep}{3pt} %
    \renewcommand{\arraystretch}{1} %
    \begin{tabular}{c|l|c|c|c}
        \midrule  
        \multirow{2}{*}{General settings} & 
        \multicolumn{1}{c|}{Optimizer} & Momentum & Weight decay & Batch size
        \\
        \cmidrule{2-5}  
        & \multicolumn{1}{c|}{SGD} & 0.9 & 0.0001 & 128
        \\
        \midrule
         \end{tabular}
         \begin{tabular}{c|l|c|c|c|c}
        \midrule  
        Datasets & \multicolumn{1}{c|}{\multirow{1}{*}{Methods}} & Model & Training epochs & Learning rate & Loss ratio
        \\
        \midrule  
        \multirow{3}{*}{2D cube} & 
        \textbf{Vanilla} 
        & 3-layer FC net
        & 40
        & 0.1 
        & - 
        \\
        & \textbf{Mixup}
        & 3-layer FC net 
        & 40
        & 0.1
        & - 
        \\
        & \grayl{\textbf{Mixup+GenLabel}}     
        & \grayc{3-layer FC net}       
        & \grayc{40}
        & \grayc{0.1}         
        & \grayc{1}
        \\
        \midrule  
        \multirow{3}{*}{3D cube} & 
        \textbf{Vanilla} 
        & 3-layer FC net 
        & 40
        & 0.1 
        & - 
        \\
        & \textbf{Mixup}
        & 3-layer FC net 
        & 40
        & 0.1
        & - 
        \\
        & \grayl{\textbf{Mixup+GenLabel}}     
        & \grayc{3-layer FC net}       
        & \grayc{40}
        & \grayc{0.1}         
        & \grayc{0.8}
        \\
        \midrule  
        \multirow{3}{*}{Moon} & 
        \textbf{Vanilla} 
        & 4-layer FC net 
        & 100
        & 0.1 
        & - 
        \\
        & \textbf{Mixup}
        & 4-layer FC net 
        & 100
        & 0.1
        & - 
        \\
        & \grayl{\textbf{Mixup+GenLabel}}     
        & \grayc{4-layer FC net}       
        & \grayc{100}
        & \grayc{0.1}         
        & \grayc{1}
        \\
        \midrule  
        \multirow{3}{*}{Circle} & 
        \textbf{Vanilla} 
        & 4-layer FC net 
        & 100
        & 0.1 
        & - 
        \\
        & \textbf{Mixup}
        & 4-layer FC net 
        & 100
        & 0.1
        & - 
        \\
        & \grayl{\textbf{Mixup+GenLabel}}     
        & \grayc{4-layer FC net}       
        & \grayc{100}
        & \grayc{0.1}         
        & \grayc{0.8}
        \\
        \midrule  
        \multirow{3}{*}{Two-circle} & 
        \textbf{Vanilla} 
        & 4-layer FC net 
        & 100
        & 0.1 
        & - 
        \\
        & \textbf{Mixup}
        & 4-layer FC net 
        & 100
        & 0.1
        & - 
        \\
        & \grayl{\textbf{Mixup+GenLabel}}     
        & \grayc{4-layer FC net}       
        & \grayc{100}
        & \grayc{0.1}         
        & \grayc{1}
        \\
        \bottomrule
        
    \end{tabular}
    \endgroup
	\vspace{-2mm}
\end{table}

\begin{table}[t]
	\centering
	\vspace{-2mm}
	\caption{Hyperparameters and models used for clean validation in OpenML datasets experiments.}
	\scriptsize
	\label{Table:tabular_cln_settings}
	\begingroup
    \setlength{\tabcolsep}{3pt} %
    \renewcommand{\arraystretch}{1} %
    \begin{tabular}{c|l|c|c|c}
        \midrule  
        \multirow{2}{*}{General settings} & 
        \multicolumn{1}{c|}{Training epochs} & Optimizer & Weight decay & Batch size
        \\
        \cmidrule{2-5}  
        & \multicolumn{1}{c|}{100}& Adam &  0.0001 & 128
        \\
        \midrule
         \end{tabular}
         \begin{tabular}{c|l|c|c|c}
        \midrule   
        Datasets & \multicolumn{1}{c|}{\multirow{1}{*}{Methods}} & Model & Learning rate & Loss ratio
        \\
        \midrule  
        \multirow{3}{*}{OpenML} & 
        \textbf{Vanilla} 
        & Logistic Regression 
        & Chosen by cross-validation (among 0.1, 0.01, 0.001, and 0.0001)
        & - 
        \\
        & \textbf{Mixup}
        & Logistic Regression 
        & Chosen by cross-validation (among 0.1, 0.01, 0.001, and 0.0001)
        & - 
        \\
        & \grayl{\textbf{Mixup+GenLabel}}     
        & \grayc{Logistic Regression}       
        & \grayc{Chosen by cross-validation (among 0.1, 0.01, 0.001, and 0.0001)}         
        & \grayc{Chosen by cross-validation}
        \\
        \bottomrule
        
    \end{tabular}
    \endgroup
	\vspace{0mm}
\end{table}

\begin{table}[t]
	\centering
	\vspace{-2mm}
	\caption{Hyperparameters and models used for robust validation in OpenML dataset experiments.}
	\scriptsize
	\label{Table:tabular_rob_settings}
	\begingroup
    \setlength{\tabcolsep}{3pt} %
    \renewcommand{\arraystretch}{1} %
    \begin{tabular}{c|l|c|c|c|c|c}
        \midrule  
        \multirow{2}{*}{General settings} & 
        \multicolumn{1}{c|}{Training epochs} & Optimizer & Momentum & Weight decay & Batch size & FGSM attack radius
        \\
        \cmidrule{2-7}  
        & \multicolumn{1}{c|}{50}& SGD & 0.9 & 0.0001 & 128 & 0.2
        \\
        \midrule
         \end{tabular}
         \begin{tabular}{c|l|c|c|c|c}
        \midrule   
        Datasets & \multicolumn{1}{c|}{\multirow{1}{*}{Methods}} & Model & Learning rate & Loss ratio
        \\
        \midrule  
        \multirow{3}{*}{OpenML} & 
        \textbf{Vanilla} 
        & Logistic Regression 
        & 0.02
        & - 
        \\
        & \textbf{Mixup}
        & Logistic Regression 
        & 0.02
        & - 
        \\
        & \grayl{\textbf{Mixup+GenLabel}}     
        & \grayc{Logistic Regression}       
        & \grayc{0.02}         
        & \grayc{Chosen by cross-validation}
        \\
        \bottomrule
        
    \end{tabular}
    \endgroup
	\vspace{0mm}
\end{table}

\begin{table}[t]
	\centering
	\vspace{0mm}
	\caption{Hyperparameters and models used for clean validation in image dataset experiments.}
	\scriptsize
	\label{Table:real_cln_settings}
	\begingroup
    \setlength{\tabcolsep}{3pt} %
    \renewcommand{\arraystretch}{1} %
    \begin{tabular}{c|l|c|c|c|c|c}
        \midrule  
        \multirow{2}{*}{General settings} & 
        \multicolumn{1}{c|}{Training epochs} & Learning rate scheduler & Optimizer & Momentum & Weight decay & Batch size
        \\
        \cmidrule{2-7}  
        & \multicolumn{1}{c|}{200} & multi-step decay & SGD & 0.9 & 0.0001 & 128
        \\
        \midrule
         \end{tabular}
         \begin{tabular}{c|l|c|c|c|c|c}
        \midrule  
        Datasets & \multicolumn{1}{c|}{\multirow{1}{*}{Methods}} & Model & Learning rate & Attack radius & Loss ratio & Memory ratio
        \\
        \midrule  
        \multirow{6}{*}{MNIST} & 
        \textbf{Vanilla} 
        & LeNet-5 
        & 0.1 
        & 0.05 
        & - 
        & -
         \\
        & \textbf{AdaMixup}
        & LeNet-5 
        & 0.1 
        & 0.05
        & - 
        & -
        \\
        & \textbf{Mixup}
        & LeNet-5 
        & 0.1
        & 0.05
        & - 
        & -
        \\
        & \grayl{\textbf{Mixup+GenLabel}}     
        & \grayc{LeNet-5}       
        & \grayc{0.1}                   
        & \grayc{0.05}          
        & \grayc{0.15}
        & \grayc{0.95}
        \\
        & \textbf{Manifold mixup}
        & LeNet-5 
        & 0.1
        & 0.05 
        & - 
        & -
        \\
        & \grayl{\textbf{Manifold mixup+GenLabel}}    
        & \grayc{LeNet-5}  
        & \grayc{0.1}              
        & \grayc{0.05}          
        & \grayc{0.15}
        & \grayc{0.99}
        \\  
        \midrule
        \multirow{6}{*}{CIFAR-10} & 
        \textbf{Vanilla} 
        & ResNet-18 
        & 0.1
        & 2/255
        & - 
        & -
         \\
        & \textbf{AdaMixup} 
        & ResNet-18 
        & 0.1
        & 2/255
        & - 
        & -
        \\
        & \textbf{Mixup} 
        & ResNet-18 
        & 0.1
        & 2/255
        & - 
        & -
        \\
        & \grayl{\textbf{Mixup+GenLabel}}  
        & \grayc{ResNet-18}           
        & \grayc{0.1}                  
        & \grayc{2/255}          
        & \grayc{0.1}
        & \grayc{0.95}
        \\
        & \textbf{Manifold mixup} 
        & ResNet-18 
        & 0.1
        & 2/255
        & - 
        & -
        \\
        & \grayl{\textbf{Manifold mixup+GenLabel}}    
        & \grayc{ResNet-18}           
        & \grayc{0.1}         
        & \grayc{2/255}          
        & \grayc{0.1}
        & \grayc{0.95}
        \\  
        \midrule
        \multirow{5}{*}{CIFAR-100} & 
        \textbf{Vanilla} 
        & PreAct ResNet-18 
        & 0.1
        & 1/255
        & - 
        & -
        \\
        & \textbf{Mixup} 
        & PreAct ResNet-18 
        & 0.1
        & 1/255 
        & - 
        & -
        \\
        & \grayl{\textbf{Mixup+GenLabel}}
        & \grayc{PreAct ResNet-18}             
        & \grayc{0.1}                  
        & \grayc{1/255}          
        & \grayc{0.1}
        & \grayc{0.97}
        \\
        & \textbf{Manifold mixup} 
        & PreAct ResNet-18 
        & 0.1
        & 1/255 
        & - 
        & -
        \\
        & \grayl{\textbf{Manifold mixup+GenLabel}}  
        & \grayc{PreAct ResNet-18}               
        & \grayc{0.1}          
        & \grayc{1/255}          
        & \grayc{0.1}
        & \grayc{0.97}
        \\  
        \midrule
        \multirow{5}{*}{\shortstack{Tiny \\ ImageNet}} & 
        \textbf{Vanilla} 
        & PreAct ResNet-18 
        & 0.1
        & 1/255 
        & - 
        & -
        \\
        & \textbf{Mixup} 
        & PreAct ResNet-18 
        & 0.1
        & 1/255 
        & - 
        & -
        \\
        & \grayl{\textbf{Mixup+GenLabel}}  
        & \grayc{PreAct ResNet-18}                
        & \grayc{0.1}         
        & \grayc{1/255}          
        & \grayc{0.05}
        & \grayc{0.995}
        \\
        & \textbf{Manifold mixup} 
        & PreAct ResNet-18 
        & 0.1
        & 1/255
        & - 
        & -
        \\
        & \grayl{\textbf{Manifold mixup+GenLabel}}  
        & \grayc{PreAct ResNet-18}                
        & \grayc{0.1}         
        & \grayc{1/255}          
        & \grayc{0.05}
        & \grayc{0.995}
        \\  
        \bottomrule
        
    \end{tabular}
    \endgroup
	\vspace{0mm}
\end{table}

\begin{table}[t]
	\centering
	\vspace{0mm}
	\caption{Hyperparameters and models used for AutoAttack validation in image dataset experiments.}
	\scriptsize
	\label{Table:real_rob_settings}
	\begingroup
    \setlength{\tabcolsep}{3pt} %
    \renewcommand{\arraystretch}{1} %
    \begin{tabular}{c|l|c|c|c|c|c}
        \midrule  
        \multirow{2}{*}{General settings} & 
        \multicolumn{1}{c|}{Training epochs} & Learning rate scheduler & Optimizer & Momentum & Weight decay & Batch size
        \\
        \cmidrule{2-7}  
        & \multicolumn{1}{c|}{50} & multi-step decay & SGD & 0.9 & 0.0001 & 128
        \\
        \midrule
         \end{tabular}
         \begin{tabular}{c|l|c|c|c|c|c}
        \midrule   
        Datasets & \multicolumn{1}{c|}{\multirow{1}{*}{Methods}} & Model & Learning rate & Attack radius & Loss ratio & Memory ratio
        \\
        \midrule  
        \multirow{6}{*}{MNIST} & 
        \textbf{Vanilla} 
        & LeNet-5 
        & 0.001 
        & 0.1 
        & - 
        & -
         \\
        & \textbf{AdaMixup}
        & LeNet-5 
        & 0.001 
        & 0.1 
        & - 
        & -
        \\
        & \textbf{Mixup}
        & LeNet-5 
        & 0.001 
        & 0.1 
        & - 
        & -
        \\
        & \grayl{\textbf{Mixup+GenLabel}}     
        & \grayc{LeNet-5}       
        & \grayc{0.001}                   
        & \grayc{0.1}          
        & \grayc{0.15}
        & \grayc{0.97}
        \\
        & \textbf{Manifold mixup}
        & LeNet-5 
        & 0.001 
        & 0.1 
        & - 
        & -
        \\
        & \grayl{\textbf{Manifold mixup+GenLabel}}    
        & \grayc{LeNet-5}  
        & \grayc{0.001}              
        & \grayc{0.1}          
        & \grayc{0.15}
        & \grayc{0.97}
        \\  
        \midrule
        \multirow{6}{*}{CIFAR-10} & 
        \textbf{Vanilla} 
        & ResNet-18 
        & 0.001 
        & 2/255
        & - 
        & -
         \\
        & \textbf{AdaMixup} 
        & ResNet-18 
        & 0.001 
        & 2/255
        & - 
        & -
        \\
        & \textbf{Mixup} 
        & ResNet-18 
        & 0.001 
        & 2/255
        & - 
        & -
        \\
        & \grayl{\textbf{Mixup+GenLabel}}  
        & \grayc{ResNet-18}           
        & \grayc{0.001}                  
        & \grayc{2/255}          
        & \grayc{0.15}
        & \grayc{0.9}
        \\
        & \textbf{Manifold mixup} 
        & ResNet-18 
        & 0.001 
        & 2/255
        & - 
        & -
        \\
        & \grayl{\textbf{Manifold mixup+GenLabel}}    
        & \grayc{ResNet-18}           
        & \grayc{0.001}         
        & \grayc{2/255}          
        & \grayc{0.15}
        & \grayc{0.9}
        \\  
        \midrule
        \multirow{5}{*}{CIFAR-100} & 
        \textbf{Vanilla} 
        & PreAct ResNet-18 
        & 0.001 
        & 1/255
        & - 
        & -
        \\
        & \textbf{Mixup} 
        & PreAct ResNet-18 
        & 0.001 
        & 1/255 
        & - 
        & -
        \\
        & \grayl{\textbf{Mixup+GenLabel}}
        & \grayc{PreAct ResNet-18}             
        & \grayc{0.001}                  
        & \grayc{1/255}          
        & \grayc{0.15}
        & \grayc{0.97}
        \\
        & \textbf{Manifold mixup} 
        & PreAct ResNet-18 
        & 0.001 
        & 1/255 
        & - 
        & -
        \\
        & \grayl{\textbf{Manifold mixup+GenLabel}}  
        & \grayc{PreAct ResNet-18}               
        & \grayc{0.001}          
        & \grayc{1/255}          
        & \grayc{0.15}
        & \grayc{0.97}
        \\  
        \midrule
        \multirow{5}{*}{\shortstack{Tiny \\ ImageNet}} & 
        \textbf{Vanilla} 
        & PreAct ResNet-18 
        & 0.002 
        & 1/255 
        & - 
        & -
        \\
        & \textbf{Mixup} 
        & PreAct ResNet-18 
        & 0.002 
        & 1/255 
        & - 
        & -
        \\
        & \grayl{\textbf{Mixup+GenLabel}}  
        & \grayc{PreAct ResNet-18}                
        & \grayc{0.002}         
        & \grayc{1/255}          
        & \grayc{0.15}
        & \grayc{0.995}
        \\
        & \textbf{Manifold mixup} 
        & PreAct ResNet-18 
        & 0.002 
        & 1/255
        & - 
        & -
        \\
        & \grayl{\textbf{Manifold mixup+GenLabel}}  
        & \grayc{PreAct ResNet-18}                
        & \grayc{0.002}         
        & \grayc{1/255}          
        & \grayc{0.15}
        & \grayc{0.995}
        \\  
        \bottomrule
        
    \end{tabular}
    \endgroup
	\vspace{-2mm}
\end{table}

\section{Generative model-based mixup algorithm (GenMix)}\label{sec:GenMix}

In Section~\ref{sec:disc_GenMix} of the main manuscript, we suggested a new way of mixing data points using generative models. Here, we formally define the algorithm for such ``generative model-based mixup'', which is dubbed as \textit{GenMix}.
Our algorithm first trains a class-conditional generative model.
One can use any generative models off-the-shelf, \eg Gaussian mixture models, GANs. Based on the learned class-conditional distribution $p_c(\vx)$'s, our algorithm augments the training dataset with data points $\vx^{\text{mix}}$ that satisfy $p_{c_1}(\vx^{\text{mix}}) : p_{c_2}(\vx^{\text{mix}}) = (1-\lambda) : \lambda$ for arbitrary pre-defined $\lambda \in [0, 1]$.
It then trains a model via a standard (non-adversarial) training algorithm with the augmented dataset.
The key idea behind GenMix is that such augmented data points can act as an implicit regularizer, promoting larger margins for the classification boundary of the trained model, which in turn guarantees robustness with good generalization.

The rest of this section is organized as follows.
We first provide a formal description of the GenMix framework. 
Then, we propose two specific instances of our framework, namely, GenMix+GM and GenMix+GAN, which use Gaussian mixture (GM) and GANs for generative modeling, respectively.

\subsection{General framework}

Let $D_c = \{(\vx^{(m)}_c, \ve_c)\}_{m=1}^{n_c}$ be the training data for class $y \in [k]$, where $\vx_c^{(m)}$ is the feature vector for $m$-th data point, $\ve_c$ is the one-hot encoded label vector for any data points in class $c$, and $n_c$ is the number of data points with class $y$. The training data is denoted by $D = \cup_{c \in [k]} D_c$.
In the first stage, it trains class-conditional generative model using the given training data $X_c = \{\vx^{(m)}_c \}_{m=1}^{n_c}$, thereby learning the underlying data distribution $p_c(\vx)$. 

In the second stage, we randomly sample mixing coefficient $\lambda \in [0, 1]$. 
For each class pair $i, j \in [k]$, we generate augmented points $X_{\text{mixup}} = \{ \vx_1,  \cdots, \vx_{n_{\text{aug}}} \}$, each of which satisfies 
$p_{i}(\vx^{\text{mix}}) : p_{j}(\vx^{\text{mix}}) = (1-\lambda) : \lambda$.
In other words, the goal is to find virtual data $\vx$'s which satisfy 
\vspace{-2mm}
\begin{equation}\label{Eqn:mid_condition}
\left \lvert  \frac{ p_{j}(\vx)}{p_{i}(\vx) + p_{j}(\vx)}  - \lambda \right \rvert \leq \epsilon
\end{equation}
for a pre-defined small margin $\epsilon > 0$.
Depending on the generative model used in the algorithm, we use different methods to find these mixup points $X_{\text{mixup}}$. The detailed description of these methods are given in the following subsections.
In both schemes, we check whether the generated mixup points $X_{\text{mixup}}$ incur manifold intrusion~\citep{guo2019mixup}, and discard the mixup points having such issues. 
To be specific, for the case of mixing class $i$ and $j$, we decide that the manifold intrusion does not occur for a mixup point $\vx \in X_{\text{mixup}}$ if classes $i$ and $j$ are the two most probable classes of $\vx$, \ie $\min \{ p_i(\vx), p_j(\vx) \} \geq p_{\ell}(\vx)$ holds for all other classes $\ell \in [k]\backslash \{i, j\}$.
For augmented data $\vx$ without such manifold intrusion issue, we soft-label it as $\vy = \frac{p_i}{p_i+p_j} \ve_i  + \frac{p_j}{p_i+p_j} \ve_j $ where $p_i  = p_i(\vx)$ is the probability that $\vx$ is sampled from class $i$.
We denote 
the set of data-label pair as $D_{\text{mixup}} = \{ (\vx,\vy) \}$ for $\vx \in X_{\text{mixup}}$.

\begin{algorithm}[t!]
	\small
	\textbf{Input} Training data $D$, 
	Number of augmented data $n_{\textrm{aug}}$, likelihood-ratio margin $\eps > 0$, mixing coefficient $\lambda \in [0,1]$
	\\
	\textbf{Output} Trained model $f(\cdot)$, Augmented data $D_{\text{mixup}}$ \\
	\vspace{-4mm}
	\begin{algorithmic}%
        \STATE $p_{c} \leftarrow$ Data distribution of class $c$ learned by generative model%
		\STATE $ D_\text{mixup} \leftarrow \{\}$
        \FOR{classes $i \in [k]$ and $j\in [k]\backslash\{i\}$} 
        \STATE $n \leftarrow 0$
        \WHILE{$n < n_{\text{aug}}$}
        \STATE Find point $\vx$ satisfying $\left \lvert  \frac{ p_{j}(\vx)}{p_{i}(\vx) + p_{j}(\vx)}  - \lambda \right \rvert \leq \epsilon$
        \STATE $p_{\ell} \leftarrow p_{\ell}(\vx)$ for $\ell \in [k]$
        \IF{$\min\{p_{i}, p_{j}\} \geq p_{\ell} \quad \forall \ell \in [k]\backslash \{i, j\}$ }
            \STATE $D_\text{mixup} \leftarrow D_\text{mixup} \cup \{ (\vx, \frac{  p_{i}   }{p_{i}+p_{j}} \ve_i+\frac{  p_{j}   }{p_{i}+p_{j}} \ve_j  ) \}$
            \STATE $n \leftarrow n+1$
        \ENDIF
        \ENDWHILE
        \ENDFOR
		\STATE $f \leftarrow$ model training with $D \cup D_\text{mixup}$
	\end{algorithmic}
	\caption{GenMix}
	\label{Algo:GenMix}
\end{algorithm}

Given $n_{\text{aug}}$ data points obtained in the second stage, the algorithm finally trains the classification model $f : \mathbb{R}^n \rightarrow [0,1]^k$ that predicts the label $\vy = [y_1, \cdots, y_k]$ of the input data. 
Here, the cross-entropy loss is used while optimizing the model.
In our GenMix scheme, the model is trained by using not only the given training data $D = \cup_{c \in [k]} \{D_c \}$, but also the augmented dataset $D_{\text{mixup}}$.
The pseudocode of the GenMix algorithm is given in Algorithm~\ref{Algo:GenMix}. 

In summary, the proposed scheme is a novel data augmentation technique that first learns the data distributions for each class using class-conditional generative models, and then augments the train data with soft-labeled data points $X_{\text{mixup}}$, each of which has the likelihood ratio of $\lambda \in [0,1]$ with respect to a target class pair.

\subsection{GenMix+GM}
We first suggest GenMix+GM, a data augmentation scheme which uses the Gaussian mixture (GM) model for generative modeling. 
Here, we provide a formal description on how GenMix+GM finds the augmented points $\vx$ satisfying the likelihood ratio condition (\ref{Eqn:mid_condition}). 
Given training samples, GenMix+GM algorithm first estimates the parameters of Gaussian distribution for each class. To be specific, it computes the sample mean and the sample covariance of class $c$, represented as
$\widehat{\bm{\mu}_{c}} = \frac{1}{n_c} \sum_{m=1}^{n_c} \vx_c^{(m)}$ and 
$\widehat{\bm{\Sigma}_{c}} = \frac{1}{n_c} \sum_{m=1}^{n_c} (\vx_c^{(m)} - \widehat{\bm{\mu}_{c}})  (\vx_c^{(m)} - \widehat{\bm{\mu}_{c}})^T$, respectively.
Then, the (estimated) probability of point $\vx$ sampled from class $c$ is $p_c(\vx) = \frac{1}{\sqrt{(2\pi)^k \text{det}(\widehat{\bm{\Sigma}_c})}} e^{-(\vx - \widehat{\bm{\mu}_c})^T \widehat{\bm{\Sigma}_c}^{-1} (\vx - \widehat{\bm{\mu}_c})/2 }$.
Now, the question is how to find the virtual data points $\vx$ satisfying (\ref{Eqn:mid_condition}). This can be solved by applying quadratic discriminant analysis (QDA)~\citep{ghojogh2019linear}, which gives us the closed-form solution for $\vx$ satisfying $ \lvert \log  \frac{ p_{j}(\vx)}{p_{i}(\vx) + p_{j}(\vx)}  \rvert \simeq  \lambda
$, for given target classes $i, j$.

\begin{figure}[t]
    \vspace{-2mm}
	\centering
	\includegraphics[width = 0.5 \linewidth]
	{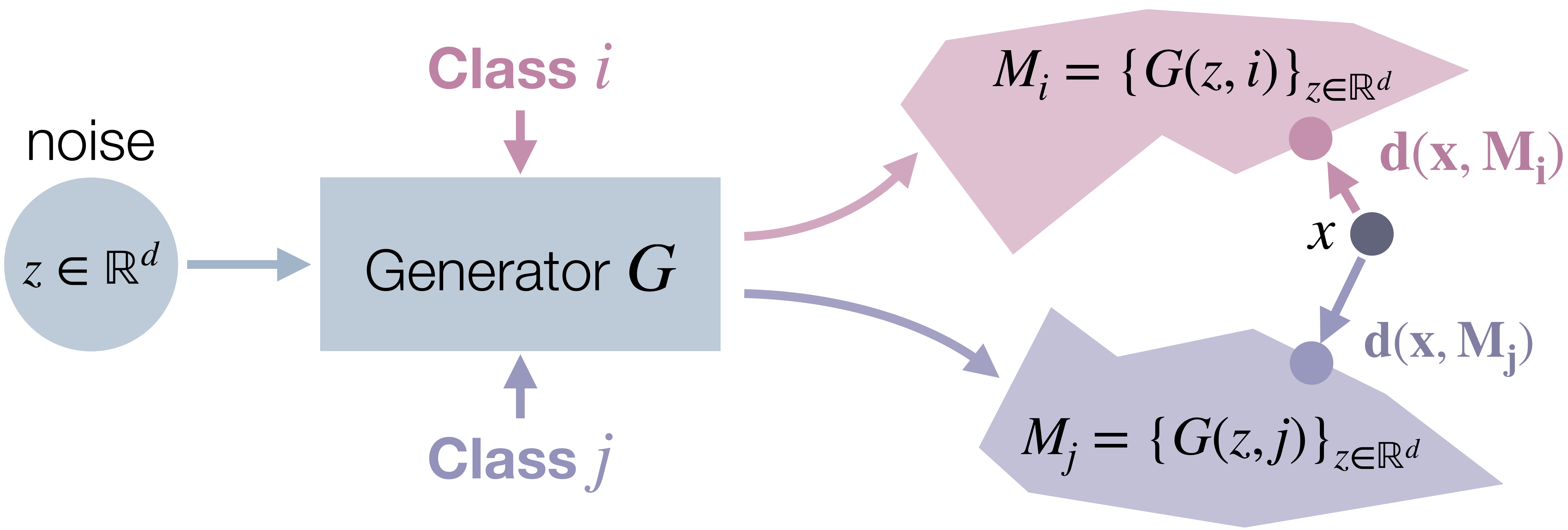}
	\vspace{-2mm}
	\caption{
	Finding the augmented point $\vx \in X_{\text{mixup}}$ satisfying $d(\vx, \mathcal{M}_j) - d(\vx, \mathcal{M}_i) \simeq \log(\frac{1}{\lambda} - 1)$ in GenMix+GAN, for arbitrary classes $i \neq j$ and a pre-defined mixing coefficient $\lambda \in [0,1]$. Given the manifold $\mathcal{M}_c = \{G(\bm{z} , c) \}_{\bm{z} \in \mathbb{R}^d} $ for class $c \in [k]$ estimated by class-conditional GAN, the distance $d(\vx, \mathcal{M}_c)$ is measured by inverting the generator of GAN~\citep{creswell2018inverting}. }
	\label{Fig:navigation}
	\vspace{-3mm}
\end{figure}

\subsection{GenMix+GAN}\label{sec:GenMix+GAN}
The Gaussian mixture (GM) model is a simple generative model that works well when the data distribution is similar to Gaussian, but it cannot learn other distributions. 
In such cases, GANs are useful for learning the underlying distribution. Thus, here we suggest GenMix+GAN which uses GANs for generative modeling.
As discussed in Section~\ref{Sec:generative_model_general}, we can replace $p_c(\vx)$ by $\exp( - d(\vx, \mathcal{M}_c) )$ in Algorithm~\ref{Algo:GenMix} and apply GenMix scheme. 
Note that the condition in (\ref{Eqn:mid_condition}) reduces to
$  d(\vx, \mathcal{M}_j) - d(\vx, \mathcal{M}_i)  \simeq \log(\frac{1}{\lambda} - 1)$. Thus, the goal is to solve $\min_{\vx} (d(\vx, \mathcal{M}_j) - d(\vx, \mathcal{M}_i) - \log(\frac{1}{\lambda} - 1) )^2$.

We use an iterative method to find points $\vx$ that satisfy this condition.
One key observation that helps us to design an efficient optimization algorithm is that if $\bm{m}^\star = \arg\min_{\bm{m}\in\mathcal{M}} d(\vx, \bm{m})$, then $d(\vx+\bm{\delta}, \mathcal{M}) \approx d(\vx+\bm{\delta}, \bm{m}^\star)$ if $\bm{\delta}$ is small.
That is, once we have a projection of $\vx$ onto a manifold $\mathcal{M}$, say $\bm{m}^\star$, the distance between $\vx+\bm{\delta}$ and the same manifold can be safely approximated by the distance between $\vx+\bm{\delta}$ and $\bm{m}^\star$, without recomputing the projection.

To formally prove this, from triangle inequality, 
{\small
	\begin{align*}
	&d(\vx+\bm{\delta}, \mathcal{M}) = \min_{\bm{m} \in \mathcal{M}} d(\vx+\bm{\delta}, \bm{m}) \\ &\leq \hspace{-1mm} \min_{\bm{m} \in \mathcal{M}}[ d(\vx+\bm{\delta}, \vx) \hspace{-0.5mm}+\hspace{-0.5mm} d(\vx, \bm{m})]
	= \hspace{-1mm} \min_{\bm{m} \in \mathcal{M}} \hspace{-1mm} d(\vx, \bm{m}) \hspace{-0.5mm} +\hspace{-0.5mm}  d(\vx, \vx+\bm{\delta}) \\
	&= d(\vx, \bm{m}^\star)+ d(\vx, \vx+\bm{\delta})
	\leq d(\vx+\bm{\delta},\bm{m}^\star)+2 d(\vx, \vx+\bm{\delta})
	\end{align*}}holds. 
Similarly, we have $d(\vx+\bm{\delta}, \mathcal{M}) \geq d(\vx+\bm{\delta},\bm{m}^\star) - 2d(\vx, \vx+\bm{\delta})$.
This implies that when $d(\vx+\bm{\delta}, \bm{m}^\star) \gg d(\vx,\vx+\bm{\delta})$, we have $d(\vx+\bm{\delta}, \mathcal{M}) \approx d(\vx+\bm{\delta}, \bm{m}^\star)$.

Using this approximation, we propose the following sequential optimization algorithm, as illustrated in Fig.~\ref{Fig:navigation}.
Starting from a random initial point $\vx \in \mathbb{R}^n$, we first compute its projection on $k$ class-conditional manifolds, finding $\bm{m}_c^\star = \arg\min_{\bm{m} \in \mathcal{M}_c} d(\vx, \bm{m})$ for each $c \in [k]$. Each of these projections can be approximately computed by solving a respective optimization problem $\min_{\bm{z} \in \mathbb{R}^d} d(\vx, G(\bm{z}, c))$.
Now, we select two target classes $i,j$ which are closest to the initial point, \ie $d(\vx, \bm{m}_i^{\star}) \leq d(\vx, \bm{m}_j^{\star}) \leq d(\vx, \bm{m}_l^{\star})$ for all $l \in [k]\setminus \{i,j\}$, and consider the following optimization problem:
\vspace{-2mm}
{\small
\begin{align*} \min\limits_{\bm{\delta}} \ \lvert &d(\vx+\bm{\delta},  \mathcal{M}_j) - d(\vx+\bm{\delta}, \mathcal{M}_i) - \log(\frac{1}{\lambda} - 1)\rvert^2\\
\text{such that}& ~d(\vx+\bm{\delta}, \bm{m}_c^\star) \gg d(\vx, \vx+\bm{\delta}), \ \ c \in \{i, j\}
\end{align*}
}

\vspace{-6mm}
That is, we find the best direction $\bm{\delta}$ that minimizes the objective function, within a small set around $\vx$.
By the aforementioned approximation, the target function can be rewritten as
$\lvert d(\vx+\bm{\delta}, \bm{m}_j^\star) - d(\vx+\bm{\delta}, \bm{m}_i^\star) - \log(\frac{1}{\lambda} - 1) \rvert^2$.
Since $\bm{m}_i^\star$ and $\bm{m}_j^\star$ are given, we can compute the gradient of this objective function with respect to $\bm{\delta}$ and run a gradient descent algorithm. 
The solution to this sub-optimization problem is now defined as $\vx$, and we repeat the whole procedure until $\lvert \frac{1}{1 + \exp(d(\vx, \bm{m}_j^{\star}) - d(\vx, \bm{m}_i^{\star})) } - \lambda  \rvert   \leq  \epsilon$, and obtain the augmented data point $\vx$. We label this augmented data as $\vy = \frac{\exp(-d_i)}{\exp(-d_i)+\exp(-d_j)} \ve_i+ \frac{\exp(-d_j)}{\exp(-d_i)+\exp(-d_j)} \ve_j$ where $d_i  = d(\vx, \bm{m}_i^{\star} )$.

\subsection{GenMix in the hidden feature space}
\label{sec:GenMix_hidden}
As illustrated in Fig.~\ref{Fig:robust_dataset_illust}, the suggested GenMix can be also defined in the hidden feature space. Below we describe the details of using GenMix in the hidden space.

Let $f_{\text{robust}}$ be the robust feature extractor suggested in~\citep{engstrom2019adversarial}. 
Note that this feature extractor is approximately {\it invertible}, \ie the input data $\vx$ can be well estimated by the representation $\bm{z} = f_{\text{robust}}(\vx)$ in the feature space. We first apply GenMix in the feature space to find the middle features $\bm{z}_{\text{mid}}$ satisfying $p_i (\bm{z}_{\text{mid}}) \simeq p_j (\bm{z}_{\text{mid}})$ for target classes $i, j$. Then, using the invertibility of $f_{\text{robust}}$, we compute $\vx_{\text{mid}} = f^{-1}_{\text{robust}}(\bm{z}_{\text{mid}}) = \arg\min\limits_{\vx} \lVert \bm{z}_{\text{mid}} - f_{\text{robust}}(\vx) \rVert $. Afterwards, we define the augmented dataset as $D_{\text{robust}} = D \cup \{ (\vx_{\text{mid}}, \vy_{\text{mid}} ) \}$,
where $\vy_{\text{mid}} =  \alpha \ve_i + (1-\alpha) \ve_j$ for $\alpha = \frac{p_i ( \bm{z}_{\text{mid}} )}{p_i ( \bm{z}_{\text{mid}} ) + p_j ( \bm{z}_{\text{mid}} )}$.

\begin{figure}[t]
	\centering
	\small
	\includegraphics[width=.7\linewidth ]{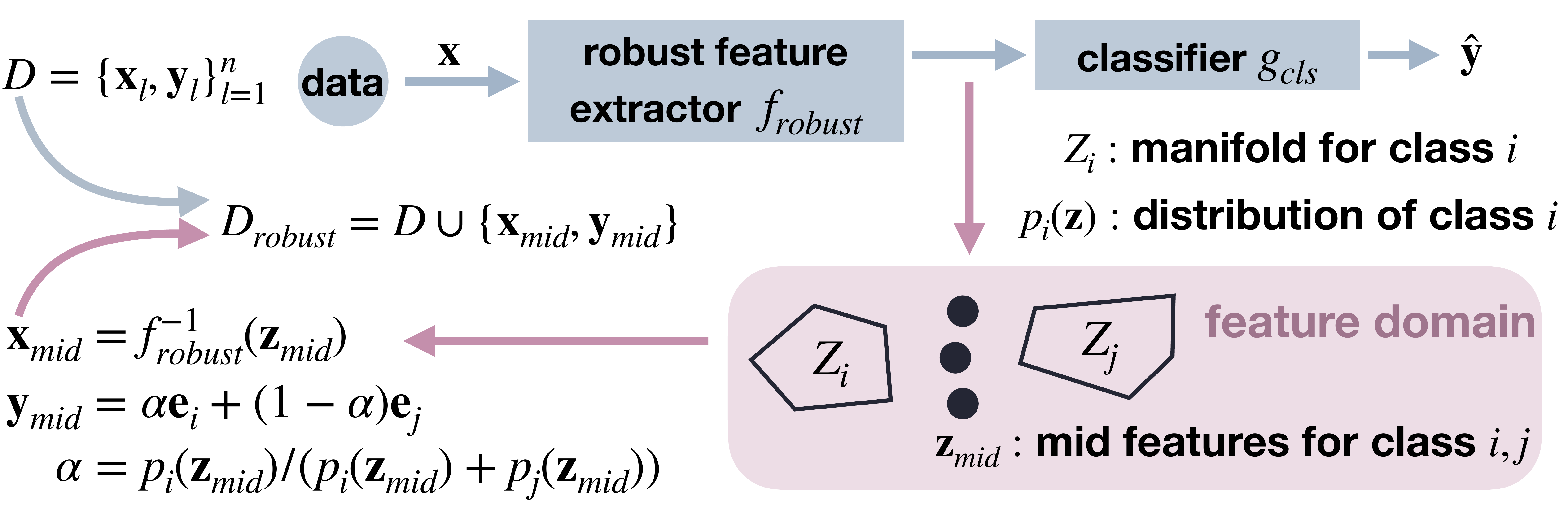}
    \caption{How to generate augmented dataset $D_{\text{robust}}$ by applying GenMix in the hidden feature space. 
    For target classes $i,j$, we first apply GenMix+GM in the feature space to find the mid hidden feature $\bm{z}_{\text{mid}}$, 
    and then invert it back to get mid input feature $\vx_{\text{mid}} = f_{\text{robust}}^{-1} (\bm{z}_{\text{mid}})$.
    Finally, we add mid input features in the original dataset $D$ to obtain $D_{\text{robust}}$.
    Here, we make use of the invertibility of $f_{\text{robust}}$ suggested in~\citep{engstrom2019adversarial}.
    }
	\label{Fig:robust_dataset_illust}
	\vspace{-3mm}
\end{figure}

\subsection{Experimental results on GenMix}\label{sec:exp_GenMix}

We evaluate the generalization and robustness performances of GenMix+GAN, GenMix+GM and existing algorithms. 
We tested on synthetic datasets (Circle, Moon in scikit-learn~\citep{scikit-learn} and V, Ket, Y datasets designed by us) and a real dataset (MNIST with digits 7 and 9). The V, Ket, Y-datasets are illustrated in Fig.~\ref{Fig:syn_train}. 
We compare our schemes with mixup~\citep{zhang2017mixup} and manifold-mixup~\citep{pmlr-v97-verma19a}.

\begin{figure}[t]
	\vspace{-5mm}
	\centering
	\small
    \subfloat[][\centering{Training data}]{\includegraphics[height=45mm ]{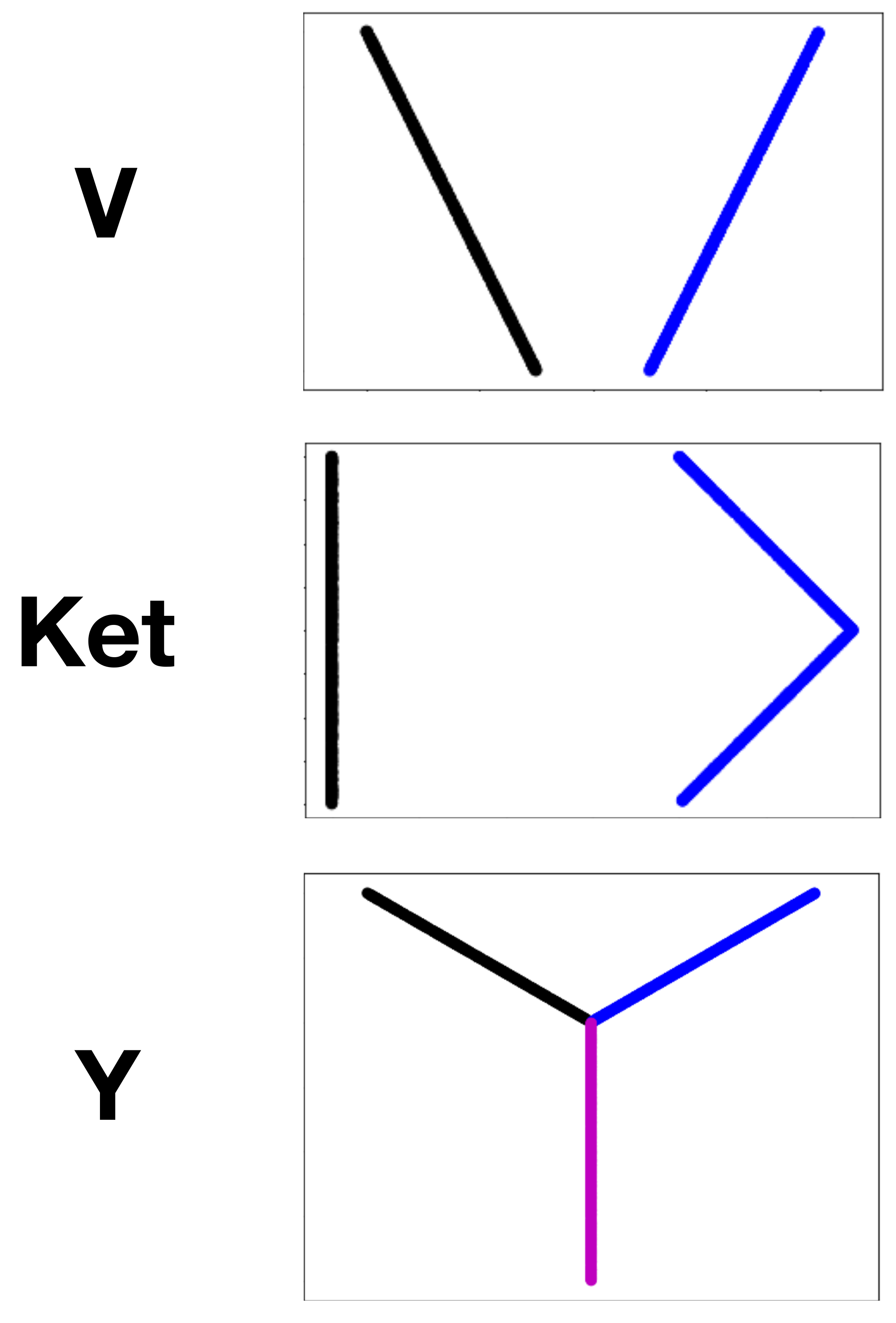}\label{Fig:syn_train}}
	\
	\subfloat[][\centering{Augmented \newline data}]{\includegraphics[height=45mm ]{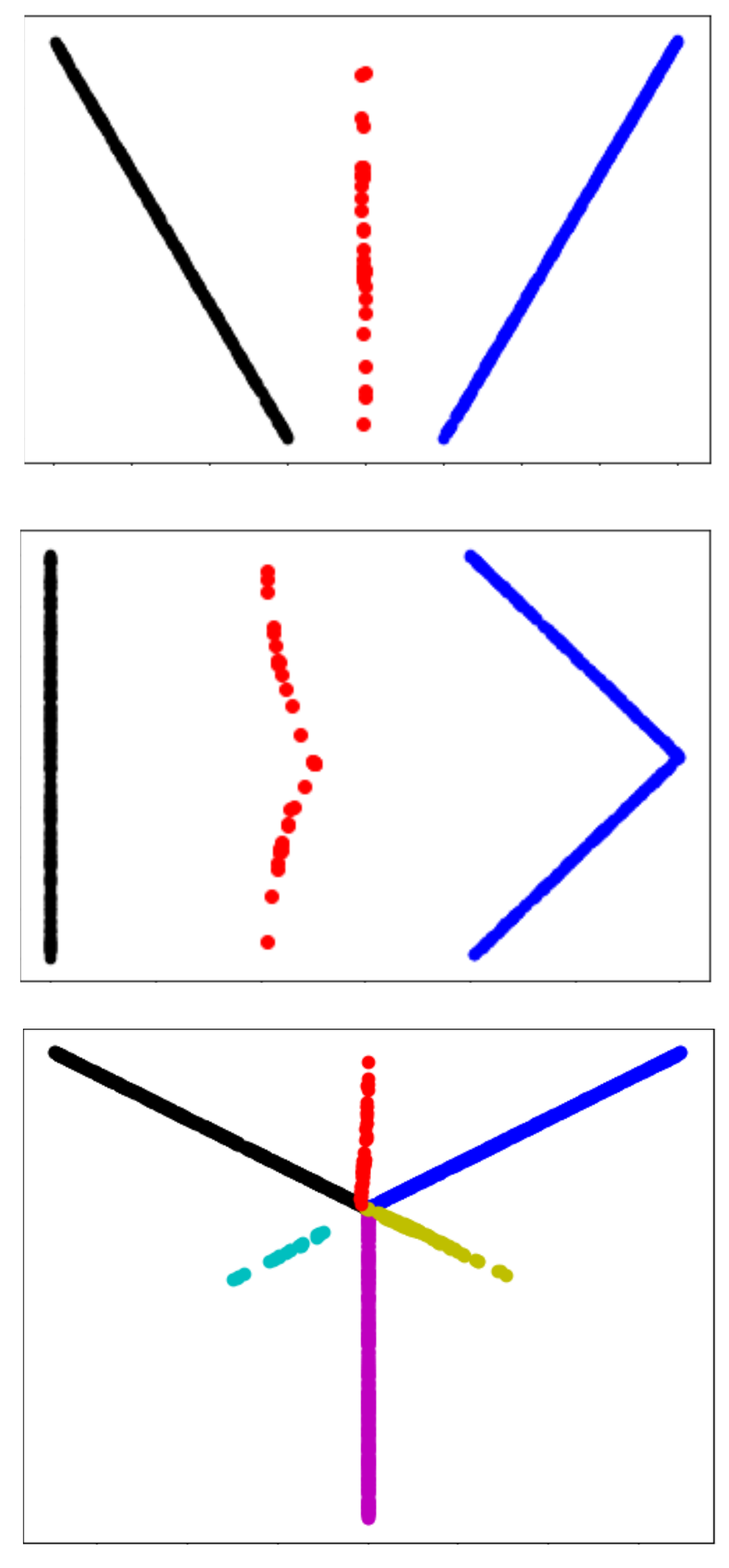}\label{Fig:syn_aug}}
	\
	\subfloat[][\centering{Decision \newline Boundary }]{\includegraphics[height=45mm]{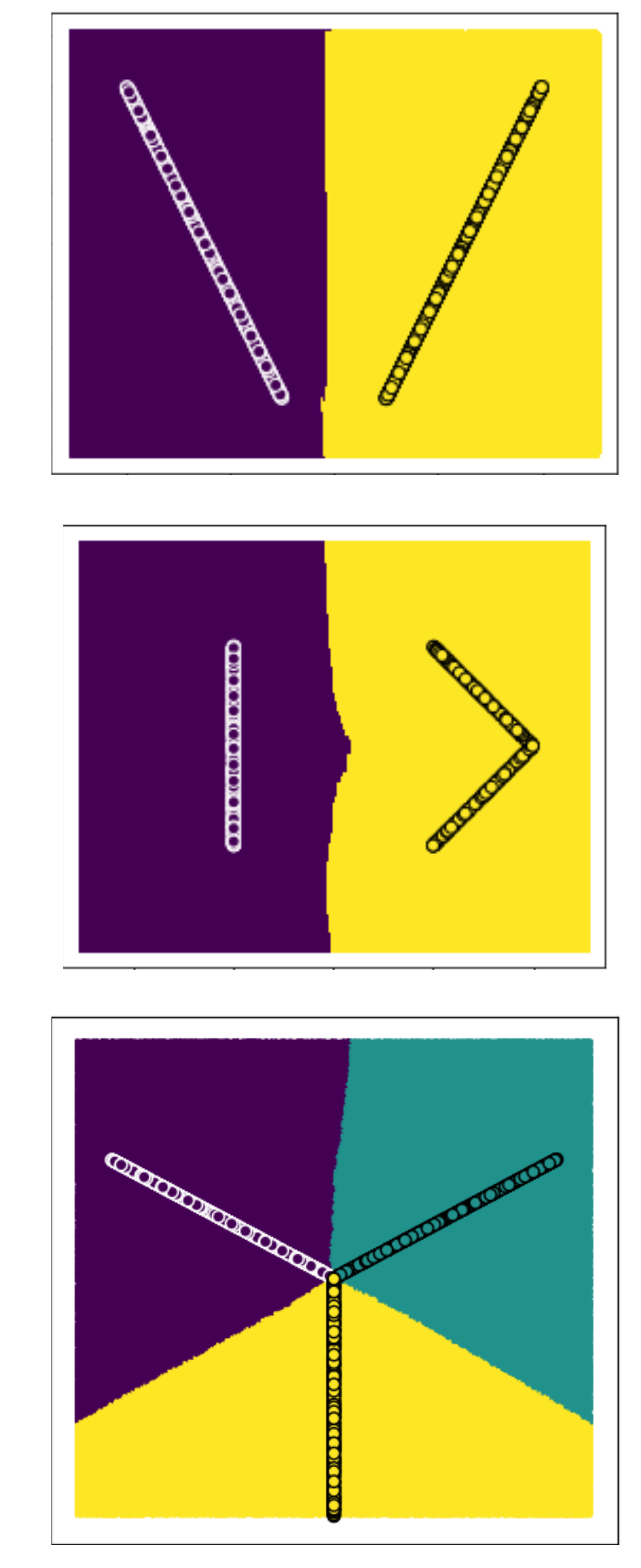}\label{Fig:syn_DB}}
	\caption{Result of GenMix+GAN on V, Ket and Y datasets. (a): Training data (black: $X_{1}$, blue: $X_{2}$, magenta: $X_3$), (b) Augmented data $X_{\text{mixup}}$ including middle points (red, yellow, cyan),
	(c): Decision boundaries (DBs) of GenMix+GAN. 
	The region with same color represents the set of points classified as the same class.
	}
	\vspace{-2mm}
	\label{Fig:Syn_middle_points}
\end{figure}

\subsubsection{GenMix enjoys large margins}

Fig.~\ref{Fig:Syn_middle_points} shows the result of GenMix+GAN for three synthetic datasets. Here, we set the mixing coefficient as $\lambda=0.5$, so that GenMix generates mixup data that are equiprobable to target classes. One can confirm that the equiprobable points help the trained model to enjoy large margins in all datasets.

In Fig.~\ref{Fig:GenMix_toy}, we visualize the suggested mixup points and the model trained by the suggested data augmentation on various synthetic datasets, and compare them with those found by vanilla mixup. Here, we set the mixing coefficient $\lambda=0.5$, meaning that the suggested mixup points are equally probable to be sampled by two target classes.

First, we show the result for 2D Gaussian dataset with 4 classes, where each data in class $c$ is sampled from a Gaussian distribution $\mathcal{N}(\bm{\mu}_c, \Sigma_c)$. 
Trivially, Gaussian mixture (GM) model fits well with this data, so we use GM to estimate $p_c(\vx)$ in this dataset.
The middle points $\vx^{\text{mix}}$ generated by the suggested mixup are illustrated in (a). 
Note that the mid points lie on the equiprobable regime for each class pair. Here, the suggested mixup learns to \emph{not} mix class-1 data (red) and class-2 data (blue), since mixing these classes incur manifold intrusion.
In (b) and (c), we show the decision boundary found by suggested mixup and vanilla mixup. 
One can see that the suggested mixup, which makes use of the underlying distribution to generate proper middle points, achieves large margins for all classes.
On the other hand, the standard mixup interpolates samples without considering the overall data distribution, resulting in smaller margins around the class-$0$ data.

Second, we show the result for circle and moon datasets defined in~\citep{scikit-learn}.
Since the Gaussian mixture model is not suitable for these datasets, we use GANs to estimate the underlying distribution $p_c(\vx)$. As described in the discussion section for applying GenLabel to ``implicit density'', we inverted GAN and used the projected distance as a proxy to the negative log likelihood. 
In (a) of circle and moon datasets, the mixed points satisfying $p_0(\vx^{\text{mix}}) = p_1(\vx^{\text{mix}})$ are colored as red, which are indeed at the middle of two manifolds of black and blue. Using these mixed points, the decision boundary has a larger margin compared with vanilla mixup, as shown in (b) and (c).

Note that in Fig.~\ref{Fig:GenMix_toy} we used $L_2$ norm for generating middle points in Moon dataset, but we can also generate middle points for $L_1$ or $L_{\infty}$ norms.
Fig.~\ref{Fig:compare_L1_L2_Linf_moon} illustrates the mixup points generated for Moon dataset, when $L_1, L_2$, and $L_{\infty}$ distance metrics are used. Here, we set the mixing coefficient as $\lambda=0.5$, i.e., the goal is to find equidistant points to target manifolds.
From the figures, we can conclude that GenMix+GAN successfully finds the points that are equidistant to both manifolds, for various $L_p$ distance settings.

\begin{figure}[t]
\vspace{-1mm}
\small
\centering
\includegraphics[width=.99\linewidth  ]{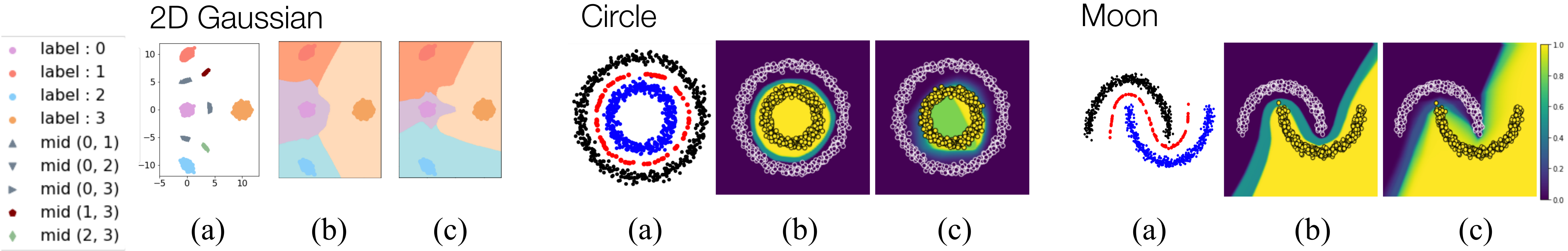}
\vspace{-2mm}
\caption{Comparison between generative-model based mixup (suggested mixup) and vanilla mixup for 2D datasets. For 2D Gaussian data, class 0,1,2,3 are colored as violet, red, blue, and yellow, respectively. For circle and moon datasets, class 0 and 1 are colored as black and blue, respectively, and the middle point obtained in the suggested mixup is colored as red. 
For each dataset, we show three results: (a) the mid points generated by the suggested mixup, and the decision boundaries of (b) suggested mixup and (c) vanilla mixup.
In (a), one can confirm that the mid points of the suggested mixup lie on the equiprobable regime for the target class pair. As shown in (b) and (c), the suggested mixup enjoys larger margins for all classes than vanilla mixup.
}
\label{Fig:GenMix_toy}
\end{figure}

	\begin{figure}[h]
		\vspace{-5mm}
		\centering
		\small
		\subfloat[][$L_1$-norm]{\includegraphics[width=.2\linewidth ]{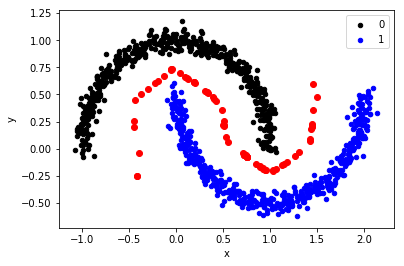}\label{Fig:Moon_L1}}
		\quad
		\subfloat[][$L_2$-norm]{\includegraphics[width=.2\linewidth ]{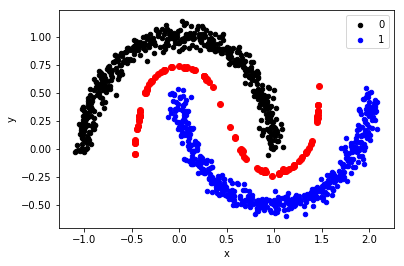}\label{Fig:Moon_L2}}
		\subfloat[][$L_{\infty}$-norm]{\includegraphics[width=.2\linewidth ]{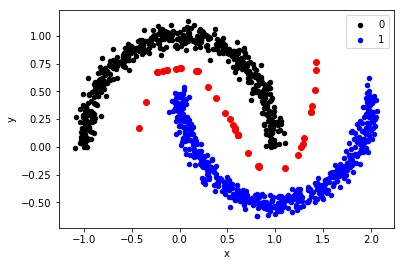}\label{Fig:Moon_Linf}}
		\caption{Illustration of data points generated by GenMix+GAN in various $L_p$ norm setup. Black and blue points correspond to each class. The red points represent the mixup points %
		generated by Algorithm~\ref{Algo:GenMix}. 
		}
		\label{Fig:compare_L1_L2_Linf_moon}
		\vspace{-2mm}
	\end{figure}

\begin{table}[t]
\small %
\vspace{-2mm}
\caption{Classification errors (\%). GenMix+GAN has a better generalization performance than other schemes.}
\label{Table:generalization}
\centering
\scriptsize
\setlength\tabcolsep{3pt}
\begin{tabular}{l|c|c|c}
\toprule  
\textbf{Schemes / Datasets}    &  Circle (2D)  & Circle (3D)  & MNIST 7/9 ($n_{\text{train}}$=500)
\\
\midrule
\textbf{Vanilla Training}  & 8.60 $\pm$ 4.84  & 1.40 $\pm$ 0.54   & 2.72  $\pm$ 0.20
\\
\textbf{Mixup}%
&
7.98 $\pm$ 2.94 & 5.22 $\pm$ 1.99 
& 	2.32  $\pm$ 0.40		
\\
\textbf{Manifold-mixup}%
& 7.34 $\pm$ 1.43  & 0.94 $\pm$ 0.75  
& 3.88  $\pm$ 0.53 
\\
\textbf{GenMix+GAN}   & \textbf{4.90} $\pm$ 0.12  & \textbf{0.22}  $\pm$ 0.06   
&  \textbf{2.13}  $\pm$ 0.12	 
\\
\bottomrule
\end{tabular}
\end{table}

\subsubsection{GenMix helps generalization}

Here we compare GenMix with mixup and manifold-mixup in terms of generalization performance. 
Table~\ref{Table:generalization} compares the performance for circle and MNIST datasets. For MNIST, we used binary classification of digits $7$ and $9$ using only $n_{\text{train}} = 500$ samples at each class, to show the scenarios with large gap between GenMix and existing schemes.
One can confirm that GenMix+GAN strictly outperforms the other data augmentation schemes in terms of generalization performances.
This shows that depending on how we generate middle points (\ie how we mix data), generalization performance varies significantly. 
One can confirm that GenMix outperforms conventional ways of mixing data, by making use of the underlying data distribution learned by generative models.

\subsubsection{GenMix in the hidden feature space}

Recall that in Section~\ref{sec:GenMix_hidden}, we have suggested GenMix in the hidden feature space. 
Fig.~\ref{Fig:robust_dataset_cifar} shows the result of GenMix+GM applied for the hidden feature space, tested on CIFAR-10 dataset. Note that each generated image contains the features of both classes $c_1, c_2$ written in the caption, showing that the mid features obtained by the suggested mixup indeed lies in between the target class manifolds.

\subsection{Reducing the computational complexity of GenMix+GAN}

Here we discuss methods for reducing the complexity of GenMix+GAN, which used inverting the generator of GAN. 
We can reduce the complexity of inverting the generator of GAN, by using alternative GAN architectures that simultaneously learn the inverse mapping during training, \eg bidirectional GAN~\citep{donahue2016adversarial} and ALIGAN~\citep{dumoulin2016adversarially}.
One can also consider using flow-based generative models, \eg~\citep{kingma2018glow}.

\begin{figure}[t]
	\vspace{-2mm}
	\centering
	\small
	\includegraphics[width=.65\linewidth ]{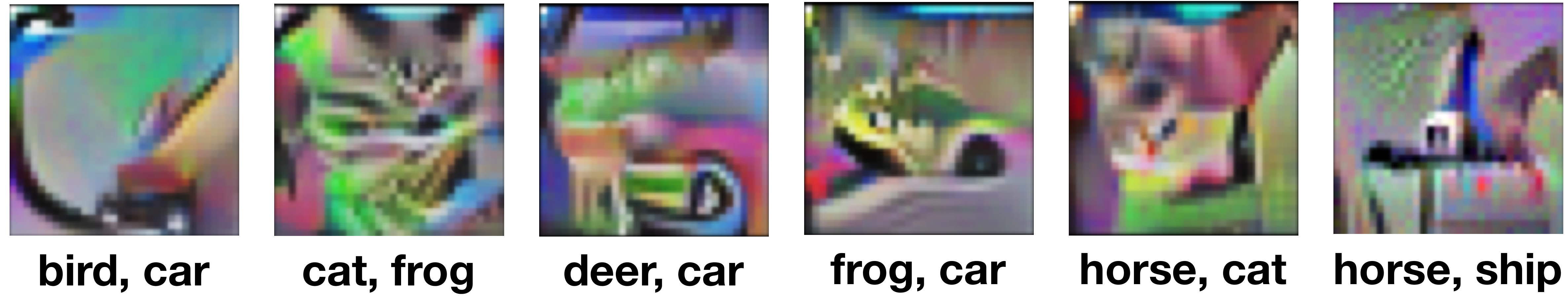}
	\vspace{-2mm}
	\caption{Generative model-based mixup in the hidden feature space for CIFAR-10.
	Each image $\vx^{\text{mix}}$ contains features of both classes $c_1, c_2$ in the caption, supporting that it lies in between the manifold of target classes.   
	}
	\label{Fig:robust_dataset_cifar}
	\vspace{-3mm}
\end{figure}

\end{document}